\documentclass[anon,11pt]{article} 

\usepackage[letterpaper, left=1in, right=1in, top=1in, bottom=1in]{geometry}

\usepackage[dvipsnames]{xcolor}
\usepackage[colorlinks=true, linkcolor=blue, citecolor=blue,linktoc=all]{hyperref}

\usepackage{natbib}
\bibpunct{(}{)}{;}{a}{,}{,}

\usepackage{amsthm}

\theoremstyle{definition}  

\newtheorem{lemma}{Lemma}

\newtheorem{corollary}{Corollary}
\newtheorem{proposition}{Proposition}
\newtheorem{assumption}{Assumption}

\theoremstyle{plain}
\newtheorem{remark}{Remark}
\newtheorem{example}{Example}
\newtheorem{theorem}{Theorem}
\newtheorem{definition}{Definition}

\usepackage{algorithm}
\usepackage{algpseudocode}

\usepackage{setspace}

\usepackage[suppress]{color-edits}
\addauthor{df}{red}
\addauthor{sr}{red}

\usepackage[utf8]{inputenc}
\usepackage[T1]{fontenc}

\usepackage{csquotes}
\usepackage{multicol}
\usepackage{dylan2}
\newcommand{\lsz}{\ls_{01}}
\newcommand{\lsh}{\ls_{\mathrm{hinge}}}
\newcommand{\lsl}{\ls_{\mathrm{lin}}}
\newcommand{\lsa}{\ls_{\mathrm{abs}}}

\newcommand{\lsb}{\boldsymbol{\ls}}

\newcommand{\burk}{\mathbf{U}}
\newcommand{\burkp}{\mathbf{U}_{p}}
\newcommand{\burkbp}{\burk_{p}^{\Bspace}}

\newcommand{\Dconst}{\mathbf{D}(\F, n)}
\newcommand{\Dconstt}{\mathbf{D}}

\newcommand{\Cconstt}{\mathbf{C}}

\newcommand{\Tr}{\mathbf{Tr}}

\newcommand{\wt}[1]{\widetilde{#1}}

\newcommand{\lp}{\ls'}

\newcommand{\UMD}{\textsc{UMD}}
\newcommand{\UMDp}{\textsc{UMD}_{p}}

\newcommand{\xr}[1][n]{x_{1:#1}}
\newcommand{\lpr}[1][n]{\ls'_{1:#1}}

\newcommand{\er}[1][n]{\eps_{1:#1}}
\newcommand{\yr}[1][n]{y_{1:#1}}

\newcommand{\B}{\mathcal{B}}

\renewcommand{\H}{\mathcal{H}}
\newcommand{\M}{\mathcal{M}}

\newcommand{\Reg}{\mathrm{\mathbf{Reg}_{n}}}
\newcommand{\Rel}{\mathrm{\mathbf{Rel}}}

\newcommand{\Rad}{\mathrm{\mathbf{Rad}}}
\newcommand{\RadH}{\widehat{\Rad}_{\F}}

\newcommand{\zigzag}{\textsc{ZigZag}}

\newcommand{\Bspace}{\mathfrak{B}}

\newcommand{\tens}{\otimes{}}
\newcommand{\trn}{\intercal}

\newcommand{\yh}{\hat{y}}

\newcommand{\x}{\mathbf{x}}

\newcommand{\rank}{\mathrm{rank}}

\renewcommand{\trn}{\dagger}

\title{ZigZag: A new approach to adaptive online learning}

\author{
Dylan J. Foster \thanks{Department of Computer Science, Cornell University}
\and
Alexander Rakhlin \thanks{Department of Statistics, University of Pennsylvania}
\and
Karthik Sridharan \thanks{Department of Computer Science, Cornell University}
}
\date{}

\begin{document}

\maketitle

\begin{abstract}
	We develop a novel family of algorithms for the online learning setting with regret against any data sequence bounded by the \emph{empirical Rademacher complexity} of that sequence. To develop a general theory of when this type of adaptive regret bound is achievable we establish a connection to the theory of \emph{decoupling inequalities} for martingales in Banach spaces. When the hypothesis class is a set of linear functions bounded in some norm, such a regret bound is achievable if and only if the norm satisfies certain decoupling inequalities for martingales. Donald Burkholder's celebrated \emph{geometric characterization} of decoupling inequalities \citep{burkholder1984boundary} states that such an inequality holds if and only if there exists a special function called a \emph{Burkholder function}  satisfying certain restricted concavity properties. Our online learning algorithms are efficient in terms of queries to this function. 

We realize our general theory by giving novel efficient algorithms for classes including $\ls_p$ norms, Schatten $p$-norms, group norms, 
and reproducing kernel Hilbert spaces. The empirical Rademacher complexity regret bound implies --- when used in the i.i.d. setting --- a \emph{data-dependent} complexity bound for excess risk after online-to-batch conversion. 
To showcase the power of the empirical Rademacher complexity regret bound, we derive improved rates for a supervised learning generalization of the \emph{online learning with low rank experts} task and for the \emph{online matrix prediction} task.

In addition to obtaining tight data-dependent regret bounds, our algorithms enjoy improved efficiency over previous techniques based on Rademacher complexity, automatically work in the infinite horizon setting, and are scale-free.
To obtain such adaptive methods, we introduce novel machinery, and the resulting algorithms are not based on the standard tools of online convex optimization.
\end{abstract}

\newpage
\tableofcontents
\newpage

\section{Introduction}

In the \emph{online supervised learning} task, a learner receives data $(x_1, y_1),\ldots,(x_n, y_n)$ in a stream. At time $t$ they receive an instance $x_t$ and must predict $y_t$ given the instance and the previous observations $(x_1, y_1,)\ldots,(x_{t-1}, y_{t-1})$. The learner's prediction, denoted $\yh_t$, is evaluated against $y_t$ according to a loss function $\ls(\yh_t, y_t)$; for classification this is typically a convex surrogate for the zero-one loss $\lsz(\yh, y)=\ind\crl*{\yh\neq{}y}$ such as the hinge loss $\lsh(\yh, y)=\max\crl*{0, 1-\yh\cdot{}y}$. The learner's overall performance is measured in terms of their \emph{regret} against a benchmark function class $\F$:
\begin{equation}
\label{eq:regret}
\sum_{t=1}^{n}\ls(\yh_t, y_t) - \inf_{f\in\F}\sum_{t=1}^{n}\ls(f(x_t), y_t).
\end{equation}

In the \emph{statistical setting}, each pair $(x_t,y_t)$ is drawn i.i.d. from some joint distribution $\D$. In this case, a bound on \pref{eq:regret} is appealing because it immediately translates to an excess loss bound for the batch statistical learning  setting after online-to-batch conversion. At the other extreme is the \emph{fully adversarial} setting, where no generating assumptions on the data are made. We would like to develop methods that enjoy optimal guarantees in both worlds.

Our goal is to come up with prediction strategies that adapt to the ``difficulty'' of the sequence. In the statistical setting, optimal excess risk behavior has long been understood through empirical process theory and, in particular, Rademacher averages \citep{BarMed03}. Empirical Rademacher averages were shown to be an attractive data-dependent measure of complexity that can be used for model selection and for estimating the excess risk of empirical minimizers. The question considered in this paper is whether there exist prediction strategies such that empirical Rademacher averages control the  per-sequence regret \eqref{eq:regret}. As we show below, the empirical Rademacher average is the best sequence-based measure of complexity one can hope for.

Let us formally define the \emph{empirical Rademacher complexity} of the class $\F$:
\begin{equation}
\label{eq:emp_Rademacher}
\RadH(\xr[n]) = \En_{\eps}\sup_{f\in\F}\sum_{t=1}^{n}\eps_{t}f(x_t),
\end{equation}
where the Rademacher sequence $\eps\in\pmo^{n}$ is drawn uniformly at random and $x_{1:n}=\{x_1,\ldots,x_n\}$.

The questions studied in this paper are:
\textbf{
\begin{itemize}
\item When does there exist a strategy $(\yh_t)$ such that
\begin{equation}
\label{eq:regret_rad}
\sum_{t=1}^{n}\ls(\yh_t, y_t) - \inf_{f\in\F}\sum_{t=1}^{n}\ls(f(x_t), y_t) \leq{} \Dconst\cdot{}\RadH(\xr[n])
\end{equation}
for \emph{every sequence} $\xr[n],\yr[n]$? 
\item What is the best constant $\Dconst$? 
\item When can the strategy $(\yh_t)$ be efficiently computed?
\end{itemize}
}

We provide a characterization of when the bound \pref{eq:regret_rad} is possible, and, furthermore, develop efficient algorithms based on a new set of techniques. The algorithms are parametrized by a certain special function that has been studied in probability theory and harmonic analysis for the last three decades. Interestingly, the function is neither convex nor concave (see \pref{fig:burkholder}), yet it satisfies a property called ``zig-zag concavity''. The main message of this paper is that this special function can be used for algorithmic purposes and to answer the above questions.

We start our analysis by showing that $\RadH$ is an ``optimal'' data-dependent regret bound in the following sense:
\begin{lemma}[Sequence Optimality]
\label{lem:optimal_b}
Let $\ls$ be the absolute, hinge, or linear loss and let $\F$ be any class of functions with value bounded by $1$.
Let $\B(\xr[n])$ be a data-dependent regret bound for which there exists a strategy $(\yh_t)$ guaranteeing
\begin{equation}
\sum_{t=1}^{n}\ls(\yh_t, y_t) - \inf_{f\in\F}\sum_{t=1}^{n}\ls(f(x_t), y_t) \leq{} \B(\xr[n]).
\end{equation}
Then
\[
\RadH(\xr[n]) \leq{} \B(x_{1:n})\quad\forall{}\xr[n].
\]
The same result holds for the zero-one loss if we restrict to $\F$ and $(\yh_t)$ with range $\pmo$.
\end{lemma}

\pref{lem:optimal_b} reveals that no data-dependent regret bound can improve upon $\RadH$ beyond the factor $\Dconst$. As we will soon show, the question of identifying $\Dconst$ is an extremely rich one. When one restricts to linear function classes, this question is deeply tied to theory of Banach space geometry and, in particular, to martingales in Banach spaces. 

In Sections~\ref{sec:relaxations}-\ref{sec:algorithms} we assume that $\F$ is a class of linear functions indexed by a unit ball; \pref{sec:martingales2} will concern the general case. For the linear case, we assume that $x_{t}$'s lie in the unit ball of a separable Banach space $(\Bspace, \nrm{\cdot})$ and  $$\F=\crl*{x\mapsto{}\tri*{w,x}\mid{}w\in\Bspace^{\star}, \nrm{w}_{\star}\leq{}1},$$ 
with $\nrm{\cdot}_{\star}$ being the dual norm and $\Bspace^{\star}$ the dual space. We then observe that $$\RadH(\xr[n])=\En_{\eps}\sup_{\nrm{w}_{\star}\leq{}1}\sum_{t=1}^{n}\eps_{t}\tri*{w,x_t} =  \En_{\eps}\nrm*{\sum_{t=1}^{n}\eps_{t}x_t}.$$ Consider the Euclidean case where $\F$ is a unit $\ell_2$ ball. It is known that gradient descent with an adaptive step size yields a regret bound of order $\sqrt{\sum_{t=1}^n \nrm*{x_t}^2}$ for any sequence. Khintchine's inequality then gives a further upper bound of order  $\En_{\eps}\nrm*{\sum_{t=1}^{n}\eps_{t}x_t}$. Hence, adaptive gradient descent answers the questions posed earlier for the specific case of linear functions indexed by Euclidean ball. This is one of the very few cases known to us where the bound of $\RadH$ is available.\footnote{The other example is $\RadH$ for the $\ls_{\infty}$ ball, attained by diagonal AdaGrad \citep{duchi2011adaptive}.}

\section{Background}
\label{sec:background}

Let $(\Bspace, \nrm{\cdot})$ be a separable Banach space and $(\Bspace^{\star}, \nrm{\cdot}_{\star})$ denote its dual. This paper focuses on the problem of online supervised learning described in \pref{proto:linear_prediction}.  
Input instances belong to some subset $\X\subseteq{}\Bspace$ and predictions $\yh_t$ are real valued. Outcomes $y_t$'s are selected from some abstract label space $\mathcal{Y}$. 
Throughout this paper we assume that the loss $\ell(\hat{y},y)$ is convex and $1$-Lipschitz in its first argument. We also assume that there exists some bounded domain $[-B,B]$ such that for all $y \in \mathcal{Y}$, $\exists \hat{y} \in [-B,B]$ such that the derivative with respect to the first argument $\ell'(\hat{y},y) = 0$ (that is, minimum is achievable in the compact set). Call such a loss function \emph{well-behaved}. We remark that this bound $B$ never explicitly appears in our results, and its only purpose is to enable application of the Minimax Theorem, which requires compactness.

\begin{algorithm}\floatname{algorithm}{Protocol}\caption{Online Supervised Learning}
\begin{spacing}{.5}
\label{proto:linear_prediction}
\begin{itemize}
\item For $t=1,\ldots,n$:
\begin{itemize}
\item Nature provides $x_{t}\in\X$.
\item Learner selects randomized strategy $q_t\in\Delta(\R)$
\item Nature provides outcome $y_t\in\Y$.
\item Learner draws $\yh_t\sim{}q_t$ and incurs loss $\ls(\yh_t, y_t)$.
\end{itemize}
\end{itemize}	
\end{spacing}
\end{algorithm}

\paragraph{Definitions}
For $p\in(1,\infty)$, let $p'=p/(p-1)$ denote its conjugate, and $p^{\star}=\max\crl*{p, p'}$. An $\X$-valued tree $\x$ is a sequence of mappings $(\x_t)_{t=1}^{n}$ with $\x_{t}:\pmo^{t-1}\to{}\X$. When $\epsilon_1,\ldots,\epsilon_n$ are independent Rademacher random variables, the tree $\x$ is simply a predictable process with respect to the dyadic filtration. Recall that a sequence of random variables $(Z_t)_{t=1}^{n}$ is a \emph{martingale} if for each $t$, $\En\brk*{Z_t\mid{}Z_1,\ldots,Z_{t-1}}=Z_{t-1}$, and is called a \emph{martingale difference sequence} if $\En\brk*{Z_t\mid{}Z_1,\ldots,Z_{t-1}}=0$. For a given martingale $(Z_t)$, we let $(dZ_t)$ denote its corresponding martingale difference sequence, i.e. $dZ_{t}=Z_{t}-Z_{t-1}$. For a matrix $X\in\R^{d\times{}d}$, let $X_{i,\cdot}$ denote the $i$th row and $X_{\cdot{}j}$ denote the $j$th column. We define its $(p,q)$ group norm as $\nrm*{X}_{p,q} = \prn{\sum_{i\in\brk{d}}\nrm{X_{i,\cdot}}_{q}^{p}}^{1/p} = \nrm{(\nrm*{X_{i,\cdot}}_{q})_{i\in\brk{d}}}_{p}$. The Schatten $p$-norm is defined as $\nrm{X}_{S_p}=\Tr\prn{(XX^{\trn})^{\frac{p}{2}}}^{\frac{1}{p}}$. We let $\nrm{X}_{\sigma}$ denote the spectral norm (Schatten $S_{\infty}$) and $\nrm{X}_{\Sigma}$ denote the nuclear norm (Schatten $S_{1}$). For a set $\mc{A}\subseteq{}\R^{d}$, assumed to be symmetric, the atomic norm with respect to $\mc{A}$ is given by $\nrm{x}_{\mc{A}}=\min\crl*{\alpha\mid{} x\in \alpha\cdot{}\textrm{conv}(\mc{A})}$.

\section{Deriving algorithms: Adaptive relaxations and zig-zag concavity}
\label{sec:relaxations}
Let us propose a simple schema for designing algorithms to achieve \pref{eq:regret_rad}. It will turn out that considering this scheme naturally leads to us to decoupling inequalities for Banach space-valued martingales via a deep result of \cite{burkholder1984boundary}.

We start by observing that by convexity of the loss function, 
\begin{align}
	\ls(\yh_t, y_t) - \ls(\tri*{w,x_t}, y_t) \leq \ls'(\yh_t,y_t)\cdot (\yh_t-\tri*{w,x_t}) 
\end{align} 
and hence, denoting the derivative by $\ls'_t = \ls'(\yh_t,y_t)$,
\begin{align}
	\label{eq:linearize_and_dualize}
	\sum_{t=1}^n \ls(\yh_t, y_t) - \inf_{\nrm{w}_{\star}\leq{}1} \sum_{t=1}^n \ls(\tri*{w,x_t}, y_t) \leq \sum_{t=1}^n \yh_t\cdot \ls'_t + \nrm*{\sum_{t=1}^n \ls'_t x_t}.
\end{align} 
Rather than aiming for the adaptive bound of empirical Rademacher averages in \pref{eq:regret_rad}, we shall aim for $\RadH(\xr[n], \lpr[n])=\En_{\eps}\nrm*{\sum_{t=1}^{n}\eps_t{}\ls'_{t}x_t}$, a quantity that is always tighter than $\RadH(\xr[n])=\En_{\eps}\nrm*{\sum_{t=1}^{n}\eps_t{}x_t}$ because $\ls$ is 1-Lipschitz. 

\cite{FosRakSri15} proposed a general framework called \emph{adaptive relaxations} for deriving algorithms to achieve data-dependent regret bounds. Adaptive relaxations are a compact tool for reasoning about minimax strategies on a round-by-round basis.

\begin{definition}
\label{def:relaxation}
An admissible relaxation $\Rel:\bigcup_{t=0}^{n}\X^{t}\times{}[-1,1]^{t}\to{}\R$ satisfies the initial condition
\begin{equation}
\label{eq:rel_initial}
\Rel(\xr[n], \lpr[n]) \geq{} \nrm*{\sum_{t=1}^{n}\lp_tx_t} - \Dconstt\cdot{}\En_{\eps}\nrm*{\sum_{t=1}^{n}\eps_t{}\ls'_{t}x_t},
\end{equation}
and the recursive condition
\begin{equation}
\label{eq:rel_admissible}
\Rel(\xr[t-1], \lpr[t-1]) \geq{} \sup_{x_t\in\X}\inf_{\yh_t}\sup_{\lp_t\in [-1,1]}\brk*{\yh_t\cdot{}\lp_t + \Rel(\xr[t], \lpr[t])}.\footnote{In original game, $\ls'_{t}=\ls'(\yh_t, y_t)$. We have moved to an upper bound by allowing the adversary to choose $\ls'_t$ arbitrarily.}
\end{equation}
\end{definition}

\begin{proposition}
\label{prop:relaxation}
Suppose $\Rel$ is an admissible relaxation. If at each time $t$ the learner plays the strategy
\begin{equation}
\yh_{t}=\argmin_{\yh}\sup_{\lp_t\in [-1,1]}\brk*{\yh\cdot{}\lp_t + \Rel(\xr[t], \lpr[t])},
\end{equation}
regret is bounded as
\[
\sum_{t=1}^{n}\ls(\yh_t, y_t) - \inf_{f\in\F}\sum_{t=1}^{n}\ls(f(x_t), y_t) \leq{} \Dconstt\cdot{}\En_{\eps}\nrm*{\sum_{t=1}^{n}\eps_t{}\ls'_tx_t} + \Rel(\emptyset).
\]
\end{proposition}
The takeaway from \pref{prop:relaxation} is that if we can design an adaptive relaxation for which the end value $\Rel(\emptyset)$ is not too large, we will have succeeded in achieving the upper bound of empirical Rademacher complexity. But how should we find such a relaxation? Let us try the simplest possible choice:
\[
\Rel(\xr[t], \lpr[t]) =\nrm*{\sum_{s=1}^{t}\lp_sx_s} - \Dconstt\cdot{}\En_{\eps}\nrm*{\sum_{s=1}^{t}\eps_s{}\ls'_{s}x_s}.
\]
This relaxation clearly satisfies the initial condition, but it is not so clear how to demonstrate the recursive condition. The challenge in analyzing this relaxation is that the function $z\mapsto{}\nrm{A+z} - \Dconstt\nrm{B+\eps{}z}$ is neither convex nor concave. Virtually all potential functions used in online learning are convex and the absence of such a property makes it difficult to bound the relaxation's growth under possible outcomes for the gradient $\lp_{t}$. Let us propose a surrogate potential with more tractable analytical properties:
\begin{proposition}
\label{prop:burkholder}
Suppose there exists a function $\burk:\Bspace\times{}\Bspace\to\R$ satisfying
\begin{enumerate}
\item\label{def:burkholder:1} $\burk(x, x') \geq{} \nrm*{x} - \Dconstt\nrm*{x'}$.
\item\label{def:burkholder:2} $\burk$ is \textbf{\emph{zig-zag concave}}: $z\mapsto{}\burk(x+z, x'+\eps{}z)$ is concave for all $x,x'\in\Bspace$ and $\eps\in\pmo$.
\item\label{def:burkholder:3} $\burk(0,0)\leq{}0$.
\end{enumerate}
Then the adaptive relaxation
\begin{equation}
\Rel(\xr[t], \lpr[t]) = \En_{\er[t]}\burk\prn*{\sum_{s=1}^{t}\lp_sx_s, \sum_{s=1}^{t}\eps_s{}\ls'_{s}x_s}
\end{equation}
is admissible.
\end{proposition}
Property \ref{def:burkholder:1} of $\burk$ clearly implies that the relaxation satisfies the initial condition, and Property \ref{def:burkholder:3} ensures that the end value is at most $0$. The zig-zag concavity property \eqref{def:burkholder:2} is most critical, as it implies that the simple gradient-based strategy 
\begin{equation}
\label{eq:strategy_warmup}
\yh_t=-\left.\frac{d}{d\alpha}\En_{\er[t]}\burk\prn*{\sum_{s=1}^{t-1}\lp_sx_s + \alpha x_t, \sum_{s=1}^{t-1}\eps_s{}\ls'_{s}x_s + \eps_t{}\alpha x_t}\right|_{\alpha=0}
\end{equation}
achieves admissibility. We remark that this strategy is horizon-independent whenever $\burk$ does not depend on $n$ (which we will show is usually the case). Furthermore, one may avoid re-drawing the random signs, and, hence, the computation time is simply the evaluation of the derivative of $\burk$. 

The full description of the ZigZag algorithm is given in \pref{sec:algorithms}, but before that let us spend some time deriving such $\burk$ functions---called the Burkholder functions---and connecting their existence to other properties of the Banach space.\footnote{We omit proofs of \pref{prop:relaxation} and \pref{prop:burkholder} for space, but the proof of \pref{thm:alg_supervised_tight}, the main algorithm, uses the same techniques is self-contained.}

\section{Zig-Zag functions, regret, and UMD spaces}
\label{sec:martingales1}

What have we gained by reducing our problem to finding a $\burk$ function? 
We will now show that $\burk$ exists \emph{if and only if} $(\Bspace, \nrm{\cdot})$ is an \emph{Unconditional Martingale Difference} (UMD) space. Informally, in a UMD space lengths of martingales are comparable to those of random walks with independent increments (see \pref{def:umd}). We call $\burk$ a \emph{Burkholder function} in reference to Donald Burkholder's central result characterizing UMD spaces in terms of the existence of these functions \citep{burkholder1984boundary}. 

In \pref{prop:burkholder} we assumed that the Burkholder function $\burk$ satisfies $\burk(x, x') \geq{} \nrm*{x} - \Dconstt\nrm*{x'}$. We will soon see that it is often easier to find an efficiently computable zig-zag concave function $\burkp$ that, as before, satisfies $\burkp(0,0) \le0$, but the first requirement in Proposition~\ref{prop:burkholder} is replaced with
$$
\burkp(x, x') \geq{} \nrm*{x}^p - \Dconstt_{p}^p \nrm*{x'}^p
$$
for some $p > 1$ (i.e. $p \ne 1$). However, the simple observation that for any number $a > 0$, $a = \frac{1}{p}\inf_{\eta>0}\{ \eta a^p  + (p-1) \eta^{-1/(p-1)}\}$ will allow us to algorithmically use a $\burk_p$ function for any $p$ to obtain the desired regret bound $\RadH$ (this is described in detail in \pref{sec:algorithms}). This motivates our complete Burkholder function definition:
\begin{definition}
\label{def:burkholder}
A function $\burkbp:\Bspace\times{}\Bspace\to\R$ is \emph{Burkholder} for $(\nrm{\cdot}, p, \Dconstt_p)$ if 
\begin{enumerate}
\item\label{def:burkholder:1} $\burkbp(x, x') \geq{} \nrm*{x}^p - \Dconstt_p^p\nrm*{x'}^p$.
\item\label{def:burkholder:2} $\burkbp$ is \textbf{zig-zag concave}: $z\mapsto{}\burkbp(x+z, x'+\eps{}z)$ is concave for all $x,x'\in\Bspace$ and $\eps\in\pmo$.
\item\label{def:burkholder:3} $\burkbp(0,0)\leq{}0$.\footnote{This condition is without loss of generality.}
\end{enumerate}
\end{definition}
For concreteness, here is a simple example for the scalar case: 
The function 
\[
\burk_{2}^{\R}(x, x') = \abs{x}^{2}-\abs{x'}^{2}
\]
is Burkholder for $(\abs{\cdot}, 2, 1)$. The reader can easily verify that this function is zig-zag concave by observing that $\burk_{2}^{\R}(x+z, x'+z)$ is in fact linear in $z$.
Perhaps the most famous $\burk$ function is Burkholder's construction for general powers in the scalar case: For $p\in{}(1,\infty)$ the function
\[
\burk^{\R}_{p}(x, x')=\alpha_{p}\prn*{\abs{x}-\beta_{p}\abs{x'}}\prn*{\abs{x} + \abs{x'}}^{p-1},
\]
is a $(\abs{\cdot}, p, \beta_p)$ Burkholder function upper bounding $\abs{x}^{p}-\beta_{p}^{p}\abs{x'}^{p}$ for appropriate $\alpha_p,\beta_p$.

\subsection{When does a zig-zag concave $\burk$ function exist?}
It turns out that the most common Banach spaces used in machine learning settings --- such as $\ell_p$ spaces, group norms, Schatten-$p$ classes, and operator norms --- all happen to be UMD spaces, and that each UMD space comes with its own $\burk$ function. This leaves us with the exciting prospect of using their corresponding $\burk$ functions to develop new adaptive online learning algorithms with improved data-dependent regret bounds. Without further ado, let us define a UMD Banach space: 
\begin{definition}
\label{def:umd}
A Banach space $(\Bspace,\|\cdot\|)$ is called UMD${}_p$ for some $1 < p < \infty$, if there is a constant $\Cconstt_p$ such that for any finite $\Bspace$-valued martingale difference sequence $(X_t)_{t=1}^{n}$ in $L_p(\Bspace)$ and any fixed choice of signs $(\epsilon_t)_{t=1}^{n}$ (where each $\epsilon_t \in \{\pm1\}$),
\begin{equation}
\label{eq:umdp}
\En \left\|\sum_{t=1}^n \epsilon_t X_t \right\|^p  \le \Cconstt_p^p \En \left\|\sum_{t=1}^n X_t \right\|^p.
\end{equation}
The space $(\Bspace, \nrm{\cdot})$ is  called $\UMD_1$ if there is a constant $\Cconstt_1$ such that
\begin{equation}
\label{eq:umd1}
\En \sup_{\tau\leq{}n}\left\|\sum_{t=1}^{\tau} \epsilon_t X_t \right\|  \le \Cconstt_1 \En \sup_{\tau\leq{}n}\left\|\sum_{t=1}^{\tau} X_t \right\|.
\end{equation}
\end{definition}
\cite{burkholder1984boundary} proved the following geometric characterization of UMD spaces in terms of existence of appropriate zig-zag concave $\burk$ functions.\footnote{\cite{burkholder1984boundary} does not work with $\burk$ functions directly but rather an equivalent property called $\zeta$-convexity. The $\burk$ function presentation first appeared in \cite{burkholder1986martingales}. See \cite{veraar2015analysis} or \cite{osekowski2012sharp} for a modern exposition.}
\begin{theorem}[\cite{veraar2015analysis}, Theorem 4.5.6]
\label{thm:umd_burkholder}
For a Banach space  $(\Bspace,\|\cdot\|)$, the following are equivalent:
\begin{enumerate}
 \item $\Bspace$ is UMD${}_p$ with constant $\Cconstt_p$.
 \item There exists Burkholder function $\burk^{\Bspace}_p: \Bspace \times \Bspace \mapsto \mathbb{R}$ for $(\nrm{\cdot}, p, \Cconstt_p)$.

\end{enumerate}
\end{theorem}

\pref{thm:umd_burkholder} is strengthened considerably by the following fact:
\begin{theorem}
\label{thm:umd_equiv}
Let $p\in(1,\infty)$. If $\UMD_{p}$ holds with constant $\Cconstt_p$, then 
\begin{itemize}
\item For all $q\in(1,\infty)$, $\UMD_{q}$, holds with constant $\Cconstt_{q}\leq{} 100\prn*{\frac{q}{p} + \frac{q'}{p'}}\Cconstt_{p}$.
\item $\UMD_{1}$ holds with $\Cconstt_{1}=O(\Cconstt_p)$.
\end{itemize}
Furthermore, if $\UMD_{1}$ holds with constant $\Cconstt_{1}$, then for all $p\in(1,\infty)$ there is some constant $\Cconstt'_p$ for which $\UMD_p$ holds.
\end{theorem}

With these properties of $\UMD$ spaces established, we proceed to state our main theorem on achieving the $\RadH$ regret bound in these spaces.

\begin{theorem}
\label{thm:umd_upper}
Let $(\Bspace, \nrm{\cdot})$ satisfy $\UMDp$ with constant $\Cconstt_p$ for any $p\in[1,\infty)$. Then there exists some randomized strategy achieving the regret bound: 
\begin{align}
\En\brk*{\sum_{t=1}^{n}\ls(\yh_t, y_t) - \inf_{f\in\F}\sum_{t=1}^{n}\ls(f(x_t), y_t)} &\leq{} O\prn*{\Cconstt_{p}\En\En_{\eps}\sup_{\tau\leq{}n}\nrm*{\sum_{t=1}^{\tau}\eps_t\ls'(\yh_t, y_t)x_t}}\label{eq:ub_max}\\
&\leq{} O\prn*{\Cconstt_p\En\prn*{\En_{\eps}\nrm*{\sum_{t=1}^{n}\eps_t\ls'(\yh_t, y_t)x_t} + \max_{t\in\brk{n}}\nrm{x_t}\log(n)}}\label{eq:ub_logn}\\
&\leq{}O\prn*{\Cconstt_p\En\prn*{\En_{\eps}\nrm*{\sum_{t=1}^{n}\eps_tx_t} + \max_{t\in\brk{n}}\nrm{x_t}\log(n)}}.
\end{align}
\end{theorem}
This shows that a bound on $\Cconstt_{p}$ for any $p$ gives $\Dconst\leq{}\Cconstt_{p}$ in \pref{eq:regret_rad}, up to an extra additive $\log{}n$ factor\footnote{All of the $\log{}n$ factors incurred in this paper arise when passing from bounds of the form $\En{}\sup_{\tau\leq{}n}F_{\tau}$ to those of the form $\En{}F_{n}$ for some random process $(F_t)$. This is notable technical issue with most martingale inequalities involving the $L_{1}(\Bspace)$ norm, including for instance Doob's maximal inequality. }.

An interesting feature of this theorem is that there are multiple ways through which it can be proven. In the appendix it is proven purely \emph{non-constructively} by plugging the $\UMD$ inequality \pref{eq:umd1} into the minimax analysis framework developed in \cite{FosRakSri15}. In \pref{sec:algorithms} it is proven \emph{constructively} by using the existence of the $\burk$ function to exhibit a particular strategy for the learner.

Let us remark that the bound in \pref{eq:ub_max} has the desirable property of being \emph{scale-free}, in that it can be achieved without an a-priori upper bound on the data norms $\max_{t\in\brk{n}}\nrm{x_t}$.

With \pref{thm:umd_upper} in mind, we finally state bounds on $\Cconstt_{p}$ for classes of interest. 
\begin{theorem}
The following $\UMD$ constants hold:\\
\label{thm:umd_constants}
\hspace{-0.2in}\begin{minipage}{.49\textwidth} %
\begin{itemize}
\item $(\R, \abs{\cdot})$: $\Cconstt_{p}= p^{\star}-1\;\forall{}p\in(1,\infty)$. 
\item $(\R^{d}, \nrm{\cdot}_{p})$, $p\in(1,\infty)$: $\Cconstt_{p}= p^{\star}-1$. 
\item $(\R^{d}, \nrm{\cdot}_{1}/\nrm{\cdot}_{\infty})$:  $\Cconstt_{2}=O(\log{}d)$.
\item $(\R^{d}, \nrm{\cdot}_{\mc{A}}/\nrm{\cdot}_{\mc{A}^{\star}})$: $\Cconstt_{2}=O(\log\abs{\mc{A}})$.
\end{itemize}
\end{minipage} %
\hspace{-0.2in}\begin{minipage}{.52\textwidth} %
\begin{itemize}
\item $(\R^{d\times{}d}, \nrm{\cdot}_{S_{p}})$, $p\in(1,\infty)$: $\Cconstt_{p}=O((p^{\star})^{2})$.
\item $(\R^{d\times{}d}, \nrm{\cdot}_{\sigma}/\nrm{\cdot}_{\Sigma})$: $\Cconstt_{2}=O(\log^{2}d)$.
\item $(\R^{d\times{}d}, \nrm{\cdot}_{p,q})$,  $p,q\in(1,\infty)$: $\Cconstt_{p}=O(p^{\star}{}q^{\star})$.
\item $(\H, \nrm{\cdot}_{\mc{H}})$ for Hilbert space $\mc{H}$: $\Cconstt_{2}=1$.
\end{itemize}
\end{minipage}

\end{theorem}

\subsection{Efficient Burkholder functions}
Burkholder's geometric characterization, \pref{thm:umd_burkholder}, implies existence of a Burkholder function $\burk_{p}^{\Bspace}$ whenever a space $(\Bspace, \nrm{\cdot})$ has $\UMD$ constant $\Cconstt_{p}$. Unfortunately, the generic $\burk$ function construction (see \cite{veraar2015analysis}, Theorem 4.5.6) is not \emph{efficiently computable}; it is expressed in terms of a supremum over all martingale difference sequences. However, the construction of concrete $\burk$ functions has been an active area of research in the three decades since Burkholder's original construction. This is because one can exhibit a $\burk$ function to certify that a space is $\UMD$ for a specific constant $\Cconstt_{p}$, and discovering \emph{sharp} $\UMD$ constants is of general interest to the analysis community \citep{osekowski2012sharp}.

Let us begin by stating Burkholder's optimal $\burk$ function construction for the scalar setting. This function was originally obtained by solving a particular partial differential equation. This function is graphed in \pref{fig:burkholder}.
\begin{example}[$\abs{\cdot}^{p}$, \cite{veraar2015analysis}, Theorem 4.5.7]
\label{ex:burkholder_scalar}
For any $p\in(1,\infty)$, the function 
\begin{equation}
\burk^{\R}_{p}(x,y) \defeq \alpha_{p}\prn*{\abs{x}-\beta_{p}\abs{y}}\prn*{\abs{x} + \abs{y}}^{p-1}
\end{equation}
is Burkholder for $(\abs{\cdot}, p, \beta_{p})$ , where $\alpha_{p}= p\prn*{1-\frac{1}{p^{\star}}}^{p-1}$, $\beta_{p}=p^{\star}-1$. 
$\beta_{p}$ is the sharpest constant possible.
\end{example}
Observe that all of the Burkholder function properties (\pref{def:burkholder}) are preserved under addition. This leads us to a construction for $\ls_p$ norms in the vector setting, which inherits the optimal constants from Burkholder's scalar construction.
\begin{example}[$\ls_{p}$ norm]
\label{ex:burkholder_pp}
\begin{equation}
\burk^{\ls_p}_{p}(x,y) \defeq \sum_{i\in\brk{d}}\burk^{\R}_{p}(x_{i}, y_{i})
\end{equation}
is a Burkholder function for $(\nrm{\cdot}_{p}^{p}, p, \beta_p)$, with $\beta_p$ as in \pref{ex:burkholder_scalar}. $\burk^{\ls_p}_p$ can be computed in time $O(d)$.
\end{example}
\begin{example}[Weighted $\ls_{2}$ norm]
Let $\nrm{x}_{A}=\sqrt{\tri*{x, Ax}}$ for some PSD matrix $A$. Then
\[
\burk^{\ls_{2,A}}_{2}(x,y)\defeq{}U^{\ls_2}_{2}(A^{1/2}x, A^{1/2}y)
\]
is a Burkholder function for $(\ls_{2,A},2,  1)$. $\burk^{\ls_{2,A}}_{2}$ can be computed in time $O(d^2)$.
\end{example}

\begin{figure}[h]
\begin{center}
\includegraphics[scale=0.5]{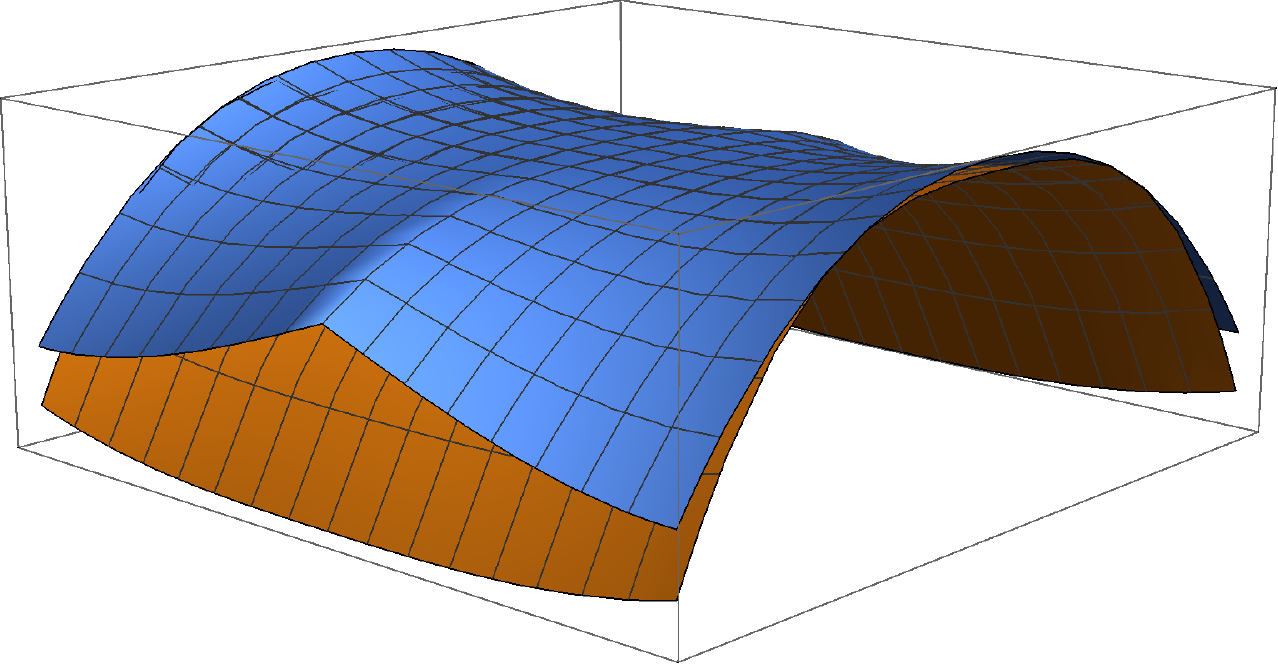}
\end{center}
\caption{$\burk^{\R}_{p}(x,x')$ (blue) and $\abs{x}^{p}-\beta_p^p\abs{x'}^{p}$ (orange) for $p=3$.}
\label{fig:burkholder}
\end{figure}
Another useful construction extends Burkholder's scalar function to general Hilbert spaces. This is useful as it applies even to infinite dimensional spaces such as RKHS.
\begin{example}[General Hilbert Space, \cite{veraar2015analysis}, Theorem 4.5.14]
\label{ex:burkholder_hilbert}
Let $\mc{H}$ be some Hilbert space whose norm will be denoted $\nrm{\cdot}_{\mc{H}}$.
\begin{equation}
\label{eq:burkholder_hilbert}
\burk^{\H}_{p}(x,y) \defeq \alpha_{p}\prn*{\nrm{x}_{\H}-\beta_{p}\nrm{y}_{\H}}\prn*{\nrm{x}_{\H} + \nrm{y}_{\H}}^{p-1}
\end{equation}
is a Burkholder function for $(\nrm{\cdot}_{\mc{H}}, p, \beta_{p})$ for each $p\in(1,\infty)$, where $\alpha_p$ and $\beta_p$, and are as in Example \ref{ex:burkholder_scalar}. This function works for all Hilbert spaces, even those of infinite dimension. For $p=2$ this function and its derivatives can be implemented efficiently using the Representer Theorem.
\end{example}
We can lift the former construction to a construction for group norms in the same fashion as in our construction for $\ls_p$ norms.
\begin{example}[$(p,2)$ Group Norm]
\label{ex:group}
In this example we consider group norms over matrices in $\R^{d\times{}d}$. The function,
\[
\burk^{(p,2)}_{p}(x,y)\defeq{}\sum_{i\in\brk{d}}\burk^{\ls_2}_{p}(x,y),
\]
where $\burk^{\ls_2,p}$ is the general Hilbert space Burkholder function \pref{eq:burkholder_hilbert},
is a Burkholder function for $(\nrm{\cdot}_{(p,2)}, p, \beta_p)$. $\burk^{(p,2)}_{p}$ can be computed in time $O(d^2)$.
\end{example}
Group norms are used in multi-task learning. Furthermore, \pref{ex:group} works not just for $\R^{d\times{}d}$, but more generally for $\R^{d}\times{}\H$ for any Hilbert space $\mc{H}$. This makes it well-suited to multiple kernel learning tasks.

As we will show in the sequel, there are a number of algorithmic tricks we can use to achieve $\RadH$-type bounds even when we do not exactly have a $\burk$ function for a class of interest.

\section{Algorithms and applications}
\label{sec:algorithms}

Recall that our goal is to design algorithms whose regret is bounded by $\RadH(\xr[n], \lpr[n])=\En_{\eps}\nrm*{\sum_{t=1}^{n}\eps_t{}\ls'_tx_t}$. Our first algorithm, \zigzag{} (\pref{alg:zigzag_supervised}), efficiently achieves a regret bound of this form whenever we have an efficient Burkholder function $\burk_{p}^{\Bspace}$ --- even if $p\neq{}1$. This notably yields an efficient algorithm for $\ls_p$ spaces by using the  Burkholder function $\burk_{p}^{\ls_p}$ from \pref{ex:burkholder_pp}.

\begin{algorithm}
\caption{\textsc{ZigZag}}\label{alg:zigzag_supervised}
\begin{algorithmic}[1]
\Procedure{ZigZag}{$\burk_{p}, p, \eta$}\Comment{$\burk_{p}$ is Burkholder for $(\nrm{\cdot}, p, \beta)$. $\eta>0$ is the learning rate.}\\
At time $t$:
\begin{enumerate}
\item Let $G_{t}(\alpha) = \En_{\sigma_{t}\in\pmo}\frac{\eta}{p}\burk_{p}\prn*{\sum_{s=1}^{t-1}\ls'_sx_s + \alpha x_t, \sum_{t=1}^{t-1}\eps_s\ls'_sx_s + \sigma_t \alpha x_t}$.
\item Predict $\yh_{t}=-G'_{t}(0)$. 
\Comment{More generally, use the supergradient.}
\item Draw independent Rademacher $\eps_{t}\in\pmo$.
\end{enumerate}
\EndProcedure
\end{algorithmic}
\end{algorithm}

\begin{theorem}
\label{thm:alg_supervised_tight}
	Denote the prediction of \pref{alg:zigzag_supervised} as $\yh_{t}^{\eps_{1:t-1}}$ to make the dependence on the sequence $(\eps_{t})_{t\leq{}n}$ explicit. \pref{alg:zigzag_supervised} enjoys the regret bound,
\begin{equation}
\label{eq:zigzag_regret2}
\En_{\eps}\brk*{\sum_{t=1}^{n}\ls(\yh^{\eps_{1:t-1}}_{t}, y_t) - \inf_{f\in\F}\sum_{t=1}^{n}\ls(f(x_t), y_t)-\frac{1}{p}\prn*{\eta\beta^{p}\nrm*{\sum_{t=1}^{n}\eps_{t}\ls'_{t}x_t}^{p} + \frac{1}{p'-1}\eta^{-(p'-1)}}}\leq{}0.
\end{equation}
\end{theorem}
A few remarks are in order. A naive application of the relaxation technique would yield a bound
\begin{equation}
\label{eq:zigzag_regret1}
\En_{\eps}\brk*{\sum_{t=1}^{n}\ls(\yh_{t}, y_t) - \inf_{f\in\F}\sum_{t=1}^{n}\ls(f(x_t), y_t)} \leq{} \frac{1}{p}\prn*{\eta\beta^{p}\En_{\eps}\nrm*{\sum_{t=1}^{n}\eps_{t}\ls'_{t}x_t}^{p} + \frac{1}{p'-1}\eta^{-(p'-1)}},
\end{equation}
which falls short of the goal of achieving $\RadH$ for the following reason. Observe that for any $p>1$,
\begin{equation}
\label{eq:var}
x^{1/p} = \frac{1}{p}\inf_{\eta>0}\prn*{\eta{}x + \frac{1}{p'-1}\eta^{1-p'}}\defeq{}\inf_{\eta>0}\Psi_{\eta,p}(x).
\end{equation}
Recall that $\eta>0$ is a parameter of \pref{alg:zigzag_supervised}. \pref{eq:var} combined with \pref{eq:zigzag_regret1} suggest that if we chose the optimal $\eta$ in hindsight, the regret of \textsc{ZigZag} would be bounded by $\sqrt[p]{\En_{\eps}\nrm*{\sum_{t=1}^{n}\eps_{t}\ls'_{t}x_t}^{p}}$. However, this bound is always worse than $\RadH$ via Jensen's inequality, and is indeed sub-optimal for $\ls_p$ norms. Luckily, \pref{eq:zigzag_regret2} reveals that for \zigzag{}, the Rademacher sequence $(\eps_t)_{t\leq{}n}$ used by the algorithm and the Rademacher sequence appearing in the regret bound are one and the same, which allows us to adapt $\eta$ to $\nrm*{\sum_{t=1}^{n}\eps_{t}\ls'_{t}x_t}$ for a particular playout of the sequence $(\eps_t)_{t\leq{}n}$ to get the desired $\RadH$ bound. This tuning of $\eta$ via doubling is stated in the next result.

\begin{lemma}
\label{lem:algorithm_supervised_doubling}
Define 
\[
\Phi(x_{t_1:t_2},, \lp_{t_1:t_2}, \eps_{t_1:t_2}) = \beta^{p}\sup_{t_1\leq{}a\leq{}b\leq{}t_2}\nrm*{\sum_{t=a}^{b}\eps_t\lp_{t}x_t}^p.
\]
Consider the following strategy:
\begin{enumerate}
\item Choose $\eta_0=(\beta\cdot{}p)^{-p}$ for $p\geq{}2$ and $\eta_0=1$ for $p<2$. Update with $\eta_{i}=2^{-\frac{i}{p'-1}}\eta_{0}$.
\item In phase $i$, which consists of all $t\in\crl*{s_{i},\ldots,s_{i+1}-1}$, play \pref{alg:zigzag_supervised}, \textsc{ZigZag}, with  learning rate $\eta_{i}$.
\item Take $s_1=1$, $s_{N+1}=n+1$, and $s_{i+1}=\inf\crl{\tau\mid{}\eta_{i}\Phi(x_{s_i:\tau-1}, \lp_{s_i:\tau-1}, \eps_{s_i:\tau-1})>\eta_{i}^{-(p'-1)}}$, where $N$ is the index of the last phase (note that whether $t=s_{i+1}$ can be tested using only information available to the learner at time $t$).
\end{enumerate}
This strategy achieves {\small
\begin{align*}
&\En_{\eps}\brk*{\sum_{t=1}^{n}\ls(\yh_t, y_t) - \inf_{f\in\F}\sum_{t=1}^{n}\ls(f(x_t), y_t)}\leq{}
O\prn*{\beta^{2}\log^{2}n\En_{\eps}\nrm*{\sum_{t=1}^{n}\eps_t\lp_tx_t} + \min\crl*{\log{}n
 +  (p\cdot{}\beta)^{\frac{p}{p-1}}, \beta{}^p\log{}n}
 }.
\end{align*}}
\end{lemma}

\subsection{$\ls_p$ norms}
We now specialize our generic algorithm to the important special case of $\ls_{p}$ norms.
\begin{example}
\label{ex:lp_alg}
Fix $p\in(1,\infty)$. Let $\yh_t$ be the strategy produced by $\zigzag$ (\pref{alg:zigzag_supervised}) using the Burkholder function $\burk_{p}^{\ls_{p}}$ from \pref{ex:burkholder_pp} with the learning rate tuning strategy from \pref{lem:algorithm_supervised_doubling}. This strategy achieves
 {\small
\begin{align}
\En_{\eps}\brk*{\sum_{t=1}^{n}\ls(\yh_t, y_t) - \inf_{f\in\F}\sum_{t=1}^{n}\ls(f(x_t), y_t)}&\leq{}
O\prn*{\En_{\eps}\nrm*{\sum_{t=1}^{n}\eps_t\lp_tx_t}_{p}\cdot{}(p^{\star})^{2}\log^{2}n + (p^{\star})^{2}\log{}n
 }\label{eq:zigzag_lp}.
\end{align}}This algorithm serves as a generalization of AdaGrad to all powers of $p$.
If we take $p=2$, the result recovers the regret bound for full matrix AdaGrad \citep{duchi2011adaptive} up to logarithmic factors:
\begin{align}
 \En_{\eps}\brk*{\sum_{t=1}^{n}\ls(\yh_t, y_t) - \inf_{f\in\F}\sum_{t=1}^{n}\ls(f(x_t), y_t)}&\leq{}
\widetilde{O}\prn*{\sqrt{\sum_{t=1}^{n}\nrm*{x_t}_{2}^{2}}}.
\label{eq:zigzag_p2}
\intertext{We can also recover the regret bound for diagonal AdaGrad \citep{duchi2011adaptive} by taking $p=1+1/\log{}d$:}
\En_{\eps}\brk*{\sum_{t=1}^{n}\ls(\yh_t, y_t) - \inf_{f\in\F}\sum_{t=1}^{n}\ls(f(x_t), y_t)}  &\leq{}
\widetilde{O}\prn*{\sum_{i\in\brk{d}}\nrm{x_{1:n,i}}_{2}}\label{eq:zigzag_p1}.
\end{align}
Here $x_{1:n,i}$ denotes the $i$th row of the data matrix $(x_1, x_2, \ldots, x_n)\in\R^{d\times{}n}$
\end{example}
There is also a direct construction of a $\burk$ function for $\ls_{1}$ due to \cite{osekowski2016umd}, which is stated in the appendix as \pref{ex:zeta_l1}. Using this function we will achieve \pref{eq:zigzag_p1}, but without having to use the learning rate tuning strategy, and with only $O(\log{}d)$ terms in regret instead of $O(\log{}^2{}d)$.

\subsection{Online matrix prediction: Spectral norm}
We are not aware of an existing construction of an efficient Burkholder function for the spectral norm, trace norm, or more generally the Schatten $p$-norms. In spite if this difficulty we were able to design an algorithm that achieves the $\RadH$ rate for the setting of matrix prediction with rank $r$ trace norm-bounded matrices as the comparator class. This algorithm, \pref{alg:zigzag_spectral}, is described in the appendix.

In the online matrix prediction setting \citep{HazKalSha12} one takes $\X=\brk{d}\times{}\brk{d}$ and the hypothesis class $\F$ to be a set of $d\times{}d$ matrices. Writing $x_{t}=(i_t, j_t)$ for the $t$th input instance, we let $F(x_t) = F[i_t, j_t]$ denote the $(i_t, j_t)$'th entry of the matrix.

\pref{alg:zigzag_spectral} is a variant of \textsc{ZigZag} for matrix prediction where $\F$ is a set of low rank trace norm-bounded matrices:
\[
\F = \crl*{F\in\R^{d\times{}d}\mid{}\nrm*{F}_{\Sigma}\leq{}\tau, \mathrm{rank}(F)\leq{}r}.
\]
Suppose for concreteness that $\ls=\lsh$ is the hinge loss. Let $N_{\mathrm{row}}=\max_{i}\abs*{\crl*{t\mid{}i_t=i}}$ and $N_{\mathrm{col}}= \max_{j}\abs*{\crl*{t\mid{}j_t=j}}$; these are the maximum number of times an entry appears in a given row or column, respectively.
\begin{proposition}
\label{prop:spectral_final}
Let $\tau=\sqrt{r}d$, so that $\F$ contains all rank-$r$ matrices with entry magnitudes bounded by $1$. \pref{alg:zigzag_spectral} achieves the following regret bound:
\begin{equation}
\label{eq:spectral_final}
\sum_{t=1}^n \ell_{\mathrm{hinge}}(\hat{y}_t , y_t) - \inf_{F\in{}\F}\sum_{t=1}^{n} \ell_{\mathrm{hinge}}(F(x_t),y_t) \le \widetilde{O}\prn*{\sqrt{r}\cdot{}d\cdot{}\sqrt{\max\crl*{N_{\mathrm{row}}, N_{\mathrm{col}}}}}.
\end{equation}
\end{proposition}
\begin{remark}Consider the average regret $\Reg/n$, which appears as an upper bound on excess risk after online-to-batch conversion. 
\begin{itemize}
\item When entries are drawn from the uniform distribution, $N_{\mathrm{col}},N_{\mathrm{row}}\approx{}n/d$, which yields
\[
\frac{\Reg}{n}\approx{} \sqrt{\frac{rd}{n}}.
\]
This implies that the algorithm will begin to generalize after seeing a constant number of rows worth of entries, which is the best possible behavior in this setting.
\item Any entry pattern satisfying $N_{\mathrm{col}},N_{\mathrm{row}}\approx{}n/d$, is sufficient to obtain the optimistic $\Reg/n\approx{} \sqrt{rd/n}$ rate. Remarkably, this can happen even when the entries are chosen adaptively, so long as the condition is satisfied once the game ends.
\item In the worst case $\Reg/n\approx{}\sqrt{r}d/\sqrt{n}$, which is the standard worst-case Rademacher complexity bound for the trace norm, and is obtained when the entry distribution is too ``spiky''.
\end{itemize}
\end{remark}

The i.i.d./optimistic bound of $\sqrt{rd/n}$ matches that obtained by \cite[Theorem 4]{foygel2011concentration} for the statistical learning setting up to logarithmic factors, but the algorithm does not need to know in advance that the entries will be distributed i.i.d. 

The worst-case $\sqrt{r}d/\sqrt{n}$ bound is weaker than that of \cite{HazKalSha12}, which obtains worst-case regret of $\Reg/n\approx{}\sqrt{rd^{3/2}/n}$, because it does not fully exploit that well-behaved losses such as $\ls_{\mathrm{hinge}}$ are effectively bounded (see \cite{shamir2014matrix} for a discussion). One can achieve the best of both worlds by using the standard multiplicative weights strategy to combine the predictions of the two algorithms. One could also combine predictions with the transductive matrix prediction algorithm proposed in \cite{rakhlin2012relax}, which will obtain a tighter $\sqrt{r}d^{3/2}/n$ rate if there are no repetitions in the observed entries.

\pref{alg:zigzag_spectral} relies on an $\eps$-net and consequently runs in exponential time, but represents a substantial development in that the Burkholder's generic $\burk$ function construction is not clearly even computable. \pref{prop:spectral_final} is a corollary of \pref{thm:spectral_regret}, which is described in full in the appendix.

\section{Beyond linear classes: Necessary and sufficient conditions}
\label{sec:martingales2}

The aim of our paper is to analyze conditions for the existence of adaptive methods that enjoy per-sequence empirical Rademacher complexity as the regret bound. In this quest, we introduced the UMD property as a necessary and sufficient condition. In the present section, we consider arbitrary, possibly non-linear function classes $\F \subseteq [-1,1]^{\X}$ and show that a closely related one sided probabilistic UMD property is the analogous  necessary and sufficient condition.

For this section we restrict ourselves to absolute loss $\ell_{abs}(\hat{y},y) = |\hat{y} - y|$ and assume that $\Y = [-1,1]$. 

\begin{theorem}\label{thm:NScond}
Let $\ls_{abs}$ be the absolute loss and let $\F \subset [-1,1]^{\X}$ be any class of predictors. The following statements are equivalent:
\begin{enumerate}
\item  \label{thm:NScond:1} There exists a learning algorithm and constant $B$ such that the following regret bound against any adversary holds:
$$
 \sum_{t=1}^{n}\ls_{abs}(\hat{y}_t, y_t) - \inf_{f\in\F}\sum_{t=1}^{n}\ls_{abs}(f(x_t), y_t) \leq{} B\En_{\eps}\sup_{f \in \F} \sum_{t=1}^{n}\eps_tf(x_t) + b
$$
\item \label{thm:NScond:2}For any $\X$ valued tree $\x = (\x_1,\ldots,\x_n)$ where each $\x_t:\{\pm1\}^{t-1}\to \X$, there exists constant $C$ such that
\begin{align}\label{eq:UMDp}
\En_{\epsilon}\left[\sup_{f \in \F} \sum_{t=1}^n \epsilon_t f(\x_t(\epsilon_{1:t-1}))\right]  \le C \En_{\epsilon , \epsilon'}\left[\sup_{f \in \F} \sum_{t=1}^n \epsilon'_t f(\x_t(\epsilon_{1:t-1}))\right] + c,
\end{align}
where $\epsilon = (\epsilon_1,\ldots,\epsilon_n)$ and $\epsilon' = (\epsilon'_1,\ldots,\epsilon'_n)$ are independent Rademacher random variables.
\end{enumerate}
Moreover, $B=\Theta(C)$ and $b=\Theta(c)$. The same result holds if we replace the absolute loss with the hinge loss.
\end{theorem}

\subsection{Function classes with the generalized UMD property}
We now show that there are indeed \emph{nonlinear} function classes that satisfy the generalized UMD inequality \pref{eq:UMDp}.
\begin{example}[Kernel Classes]
\label{ex:umd_general_kernel}
Let $\H$ be a Reproducing Kernel Hilbert Space with kernel $K$ such that $\sup_{x\in\X}\sqrt{K(x,x)}\leq{}B$, and let $\F=\crl*{f\in\H\mid{}\nrm{f}_{\H}\leq{}1}$. Then there are constants $K_1, K_2$ such that the generalized UMD property \pref{eq:UMDp} holds with 
\[
\En_{\eps}\sup_{f\in\F}\sum_{t=1}^{n}\eps_{t}f(\x_{t}(\eps_{1:t-1}))\leq{} K_1\En_{\eps,\eps'}\sup_{f\in\F}\sum_{t=1}^{n}\eps'_{t}f(\x_{t}(\eps_{1:t-1})) + K_{2}B\log(n).
\]
\end{example}

The next example is that of homogenous polynomial classes under an injective tensor norm. The full description of this setting is deferred to \pref{app:proofs}.

\begin{example}[Homogeneous Polynomials]
\label{ex:polynomials}
Consider homogeneous polynomials of degree $2k$, with coefficients under the unit ball of the norm $(\nrm{\cdot}_{\crl{1,\ldots,k},\crl{k+1,\ldots,2k}})_{\star}$ in $(\R^{d})^{\tens{}2k}$. Then there exist constants $K_1, K_2$ such that the generalized UMD property \pref{eq:UMDp} holds with 
\[
\En_{\eps}\sup_{f\in\F}\sum_{t=1}^{n}\eps_{t}f(\x_{t}(\eps_{1:t-1}))\leq{} K_1k^{2}\log^{2}(d)\En_{\eps,\eps'}\sup_{f\in\F}\sum_{t=1}^{n}\eps'_{t}f(\x_{t}(\eps_{1:t-1})) + K_{2}k^{2}\log^{2}(d)\log(n).
\]
\end{example}

\subsection{Necessary versus sufficient conditions}

 When we take $\F$ to be the unit ball of the dual norm $\nrm*{\cdot}_{\star}$ as in previous sections, the inequality in \pref{eq:UMDp} becomes:
 \begin{equation}
 \label{eq:UMD_one_sided}
 \En_{\epsilon}\left\| \sum_{t=1}^n \epsilon_t \x_t(\epsilon_{1:t-1})\right\|  \le C \En_{\epsilon , \epsilon'}\left\| \sum_{t=1}^n \epsilon'_t \epsilon_t \x_t(\epsilon_{1:t-1})\right\|.
\end{equation}
This condition is sometimes referred to as a \emph{probabilistic one-sided UMD inequality} for Paley-Walsh martingales \citep{veraar2015analysis}.  Comparing the condition to the $\UMD_1$ inequality \pref{eq:umd1} one observes three differences: The Rademacher sequence $\eps'$ is drawn uniformly rather than being fixed, we only consider Paley-Walsh martingales (trees), and there is no supremum over end times. The supremum in \pref{eq:umd1} does not present a significant difference, as it can be removed from $\UMD_{1}$ at a multiplicative cost of $O(\log{}n)$. The randomization over $\eps'$ is more interesting. It turns out that if in addition to \pref{eq:UMD_one_sided} we require the opposite direction of this inequality to hold, i.e.  
\[
\En_{\epsilon , \epsilon'}\left\| \sum_{t=1}^n \epsilon'_t \epsilon_t \x_t(\epsilon_{1:t-1})\right\| \leq{} C'\En_{\epsilon}\left\| \sum_{t=1}^n \epsilon_t \x_t(\epsilon_{1:t-1})\right\|,
\] then this is equivalent to the full UMD property \pref{eq:umd1} up to the presence of the supremum \citep[Theorem 4.2.5]{veraar2015analysis}. Thus, \pref{eq:UMD_one_sided} can be thought of as a \emph{one-sided} version of the $\UMD$ inequality. 

There are indeed classes for which one-sided $\UMD$ inequality holds but the full $\UMD$ property does not. A result due to \cite{hitczenko1994domination} shows that there is a mild separation between these conditions even in the scalar setting:\footnote{See also \cite{hitczenko1993domination,cox2007some,cox2011vector}.}
\begin{theorem}[\cite{hitczenko1994domination}]
\label{thm:hitczenko_scalar}
 There exists a constant $K$ independent of $p$ such that for all $p\in[1,\infty)$,
  \begin{equation}
 \label{eq:hitczenko_scalar}
 \En_{\epsilon}\abs*{\sum_{t=1}^n \epsilon_t \x_t(\epsilon_{1:t-1})}^{p}  \leq K^{p} \En_{\epsilon , \epsilon'}\abs*{\sum_{t=1}^n \epsilon'_t \epsilon_t \x_t(\epsilon_{1:t-1})}^{p}.
\end{equation}
\end{theorem}
When $p=1$ this result is exactly the generalized UMD inequality \pref{eq:UMDp}, and for $p>1$ it gives a one-sided version of the $\UMD_{p}$ condition. This bound is quantitatively stronger than what one would obtain from the $\UMD_{p}$ property, since \citep{burkholder1984boundary} shows that the full two-sided $\UMD_{p}$ condition requires $K\geq{}p^{\star}-1$. 
In the next section we show that the stronger constants in the one-sided inequality \pref{eq:hitczenko_scalar} can be used to obtain improved rates for the low-rank experts setting of \cite{hazan2016online} The full $\UMD_{p}$ inequality would not be sufficient for this task due to its larger constant. However, we remark that the gap here is only in logarithmic factors, and that the separation between the one-sided and full UMD properties is very mild for all examples we are aware of.

\subsection{Application: Low-rank experts}
In this section we consider a supervised learning generalization of the problem of online learning with low-rank experts \citep{hazan2016online}. Within \pref{proto:linear_prediction}, we take $\X=\crl*{x\in\R^{d}\mid{}\nrm{x}_{\infty}\leq{}1}$ and take our set of predictors to be the simplex: $\F=\crl*{x\mapsto{}\tri*{w,x}\mid{}w\in\Delta_{d}}$. We let $\Y=\brk{-1, +1}$ and take $\ls$ to be any well-behaved loss. 

The challenge stated in \citep{hazan2016online} is to develop algorithms for this setting whose regret scales not with the dimension $d$ (as in the standard experts bound of $O(\sqrt{n\log{}d})$), but rather scales with the rank of the observed data matrix $X_{1:n}=(x_{1}\mid{}\ldots\mid{}x_{n})\in\R^{d\times{}n}$. \cite{hazan2016online} gave an algorithm obtaining regret $O(\sqrt{n}\cdot\rank(X_{1:n}))$ and showed a lower bound of $\Omega(\sqrt{n\cdot\rank(X_{1:n})})$. Note that these bounds differ by a factor of $\sqrt{\rank(X_{1:n})}$; improving this gap was stated in \citep{hazan2016online} as Open Problem (1). Using Hitczenko's decoupling inequality, this gap can be closed for the supervised setting.
\begin{theorem}
\label{thm:low_rank}
For the supervised experts setting, there exists a strategy $(\yh_{t})$ that attains
\begin{equation}
\label{eq:regret_low_rank}
\sum_{t=1}^{n}\ls(\yh_t, y_t) - \inf_{f\in\F}\sum_{t=1}^{n}\ls(f(x_t), y_t) \leq{} O\prn*{\sqrt{n\cdot\rank(X_{1:n})}} + O(\log{}n\log{}d).
\end{equation}
\end{theorem}
This bound matches the lower bound given in \citep{hazan2016online} up to a low-order additive $\log{}d$ term. The result has two main ingredients: First, using Hitczenko's inequality, we show that there exists an algorithm whose regret is bounded by a quantity that closely approximates the empirical Rademacher complexity $\RadH$ for the class $\F$. Then, following \cite{hazan2016online}, we show that the empirical Rademacher complexity of $\F$ on a sequence $\xr[n]$ can be bounded as $O(\sqrt{n\cdot\rank(X_{1:n})})$.

Our approach also yields improved rates in terms of \emph{approximate rank} of the matrix $X_{1:n}$, which was stated as Open Problem (3) in \citep{hazan2016online}. Define the $\gamma$-approximate rank of $X$ via $\rank_{\gamma}(X)=\min\crl*{\rank(X')\mid{}\nrm{X-X'}_{\infty}\leq{}\gamma, \nrm{X'}_{\infty}\leq{}1}$.

\begin{theorem}
\label{thm:low_rank_approx}
There exists a strategy $(\yh_{t})$ that for all $\gamma>0$ attains
\begin{equation}
\label{eq:regret_low_rank_approx}
\sum_{t=1}^{n}\ls(\yh_t, y_t) - \inf_{f\in\F}\sum_{t=1}^{n}\ls(f(x_t), y_t) \leq{} O\prn*{\sqrt{n\cdot\rank_{\gamma}(X_{1:n})} + \gamma\sqrt{n\log{}d}} + O(\log{}n\log{}d).
\end{equation}
Furthermore, the strategy is the same as that of \pref{thm:low_rank}.
\end{theorem}
A bound matching \pref{eq:regret_low_rank_approx} up to log factors was given in \citep{hazan2016online}, but only for the stochastic setting.

Lastly, we give improved rates for Open Problem (2) of \citep{hazan2016online}, which asks for experts bounds that only depend on the max norm of $X_{1:n}$. Recall that \[\nrm{X}_{\mathrm{max}}=\min_{U\in\R^{d\times{}d},V\in\R^{n\times{}d}, X=UV^{\trn}}\nrm{U}_{\infty,2}\nrm{V}_{\infty,2},\] where $\nrm{\cdot}_{\infty,2}$ denotes the group norm.

\begin{theorem}
\label{thm:experts_max_norm}
There exists a strategy $(\yh_{t})$ that attains
\begin{equation}
\label{eq:regret_low_rank}
\sum_{t=1}^{n}\ls(\yh_t, y_t) - \inf_{f\in\F}\sum_{t=1}^{n}\ls(f(x_t), y_t) \leq{} O\prn*{\sqrt{n}\cdot\nrm{X_{1:n}}_{\mathrm{max}}}  + O(\log{}n\log{}d).
\end{equation}
Furthermore, the strategy is the same as that of \pref{thm:low_rank} and \pref{thm:low_rank_approx}.
\end{theorem}

For \pref{thm:low_rank}, \pref{thm:low_rank_approx}, and \pref{thm:experts_max_norm}, the key idea is to (almost) achieve the empirical Rademacher complexity in the \emph{online setting}, then apply bounds that had previously been used in the \emph{statistical setting} to get tight data-dependent bounds. Since all of of these theorems are derived as upper bounds on the empirical Rademacher complexity, they are actually achieved simultaneously by a single algorithm, and this algorithm needs no knowledge of the rank, approximate rank parameter $\gamma$, or max norm a-priori.

While our bounds depend on the ambient dimension $d$, they do so only weakly, through an additive $\log{}d$ term that does not depend on, for example, $\sqrt{n}$. Therefore, they improve on \citep{hazan2016online} as long as the dimension $d$ is at most exponential in $\sqrt{n}$.

It is important to note that the new bounds we have stated do not immediately transfer to the online linear optimization setting considered in \citep{hazan2016online} due to the condition on the loss $\ls$. Rather, they act as supervised analogues to the results in that paper. We do not yet have an efficient algorithm that obtains \pref{eq:regret_low_rank} because we do not have an efficient $\burk$ function analogue for the one-sided $\UMD$ inequality.

\subsection{Empirical covering number bounds}

Having developed online learning algorithms for which regret is bounded by the empirical Rademacher complexity, we are in the appealing position of being able to apply empirical process tools designed for the \emph{statistical setting} to derive tight regret bounds for the \emph{adversarial setting}. One particularly powerful set of tools is those based on covering numbers and, in particular, chaining.

\begin{definition}[Empirical Cover]
\label{def:cover}
For a hypothesis class $\F:\X\to{}\R$, data sequence $\xr[n]$, and $\alpha>0$, a set $\mc{V}\subseteq{}\R^{n}$ is called an empirical covering with respect to $\ls_{p}$, $p\in(1,\infty)$, if
\begin{equation}
\forall{}f\in\F\;\exists{}v\in\mc{V}\mathrm{~~s.t.~~}\prn*{\frac{1}{n}\sum_{t=1}^{n}(f(x_t)-v_t)^{p}}^{1/p}\leq{}\alpha.
\end{equation}
The set $\mc{V}$ is a cover with respect to $\ls_{\infty}$ if $\forall{}f\in\F\;\exists{}v\in\mc{V}\mathrm{~~s.t.~~}\abs{f(x_t)-v_t}\leq{}\alpha\;\forall{}t\in\brk{n}$.
\end{definition}
We let the \emph{empirical covering number} $\mc{N}_{p}(\F, \alpha, \xr[n])$ denote the size of the smallest $\alpha$-empirical cover for $\F$ on $\xr[n]$ with respect to $\ls_p$.

Because our task is simply to obtain bounds on the empirical Rademacher complexity on a particular sequence $\xr[n]$, we can obtain regret bounds that depend on the data-dependent \emph{empirical} covering number defined above, instead of a \emph{worst-case} covering number. Such bounds have proved elusive in the adversarial setting, where most existing results are based on worst-case covering numbers (e.g. \cite{RakSriTew10}). In particular, we derive two regret bounds based on the classical covering number bound \citep{pollard90empproc} and Dudley Entropy Integral bound \citep{Dudley67} for Rademacher complexity.
\begin{theorem}[Empirical covering bound]
\label{thm:general_cover}
For any class $\F\subseteq{}[-1,+1]^{\mc{X}}$ satisfying the generalized $\UMD$ inequality \pref{eq:UMDp} with constant $C$, there exists a strategy $(\yh_{t})$ that attains
\begin{equation}
\label{eq:general_cover}
\sum_{t=1}^{n}\ls(\yh_t, y_t) - \inf_{f\in\F}\sum_{t=1}^{n}\ls(f(x_t), y_t) \leq{} O\prn*{C\cdot{}\inf_{\alpha>0}\crl*{\alpha{}n
 + \sqrt{\log{}\mc{N}_1(\F,\alpha,\xr[n])n}}}.
\end{equation}
\end{theorem}

\begin{theorem}[Empirical Dudley Entropy bound]
\label{thm:general_chaining}
For any class $\F\subseteq{}[-1,+1]^{\mc{X}}$ satisfying the generalized $\UMD$ inequality \pref{eq:UMDp} with constant $C$, there exists a strategy $(\yh_{t})$ that attains
\begin{equation}
\label{eq:general_chaining}
\sum_{t=1}^{n}\ls(\yh_t, y_t) - \inf_{f\in\F}\sum_{t=1}^{n}\ls(f(x_t), y_t) 
\leq{} 
O\prn*{
C\cdot{}\inf_{\alpha>0}\crl*{\alpha\cdot{}n + \int_{\alpha}^{1}\sqrt{\log{}\mc{N}_2(\F,\delta,\xr[n])n}d\delta}
}.
\end{equation}
\end{theorem}

More generally, since our upper bounds depend on the empirical Rademacher complexity conditioned on the data $\xr[n]$, more powerful techniques --- such as Talagrand's generic chaining --- may be applied to derive even tighter data-dependent covering bounds than those implied by \pref{eq:general_chaining}.

\cite{cohen2017online} recently obtained bounds in the online learning with expert advice setting that scale with the empirical covering number of the class $\F=\Delta_{\mathbb{N}}$ (the simplex on countably many experts) on the data sequence. They derive regret bounds that scale as 
\[
  \inf_{\alpha>0}\crl*{\alpha{}n + \mc{N}_{\infty}(\Delta_{\mathbb{N}}, \alpha, \xr[n]) + \sqrt{\mc{N}_{\infty}(\Delta_{\mathbb{N}}, \alpha, \xr[n])n}}.
  \]

This bound falls short of the Pollard-style covering bound \pref{eq:general_cover}, which enjoys \emph{logarithmic} scaling in the covering number $\mc{N}$. As a corollary of our empirical Rademacher complexity regret bound, we derive a rate with the correct dependence on $\mc{N}$ for the supervised learning generalization of the experts setting described in the previous section.

\begin{theorem}
\label{thm:cover_bound}
For the supervised experts setting, there exists a strategy $(\yh_{t})$ that attains
\begin{equation}
\label{eq:cover_pollard}
\sum_{t=1}^{n}\ls(\yh_t, y_t) - \inf_{f\in\Delta_d}\sum_{t=1}^{n}\ls(f(x_t), y_t) \leq{} O\prn*{\inf_{\alpha>0}\crl*{\alpha{}n
 + \sqrt{\log{}\mc{N}_1(\Delta_d,\alpha,\xr[n])n}}} + O(\log{}n\log{}d).
\end{equation}
\end{theorem}
This bound does not apply to the countable simplex $\Delta_{\mathbb{N}}$ due to the low-order additive $\log(d)$ term, but offers an improvement on two fronts: First, it has the correct logarithmic dependence on the empirical cover, and second, it scales with the $\ls_1$-cover instead of the $\ls_{\infty}$-cover. Note that one always has $\mc{N}_{1}\leq{}\mc{N}_{\infty}$. 

We remark that the extraneous $\log(d)$ can be replaced by the worst-case data-independent covering number (i.e. $\sup_{\xr[n]\in\mc{X}^{n}}\log{}\mc{N}_1(\Delta_{\mathbb{N}},\alpha,\xr[n])$), and so can apply to the countable simplex $\Delta_{\mathbb{N}}$ if $\mc{X}$ possesses additional structure a-priori. We leave replacing $\log(d)$ with an empirical covering number or removing it entirely as an open question.

We conclude this section by noting that one can further derive an improvement on \pref{eq:cover_pollard} based on the data-dependent Dudley chaining.

\begin{theorem}
\label{thm:cover_chaining}
For the supervised experts setting, there exists a strategy $(\yh_{t})$ that attains
\begin{equation}
\label{eq:cover_chaining}
\sum_{t=1}^{n}\ls(\yh_t, y_t) - \inf_{f\in\Delta_d}\sum_{t=1}^{n}\ls(f(x_t), y_t) 
\leq{} 
O\prn*{
\inf_{\alpha>0}\crl*{\alpha{}n + \int_{\alpha}^{1}\sqrt{\log{}\mc{N}_2(\Delta_d,\delta,\xr[n])n}d\delta}
} 
 + O(\log{}n\log{}d).
\end{equation}
\end{theorem}

\section{Discussion and further directions}

We considered the task of achieving regret bounded by the empirical Rademacher complexity $\RadH$ in the adversarial online learning setting. We showed that $\RadH$ satisfies a notion of \emph{sequence optimality}, and derived necessary and sufficient conditions under which this bound can be achieved based on a connection to decoupling inequalities for martingales, namely the \emph{UMD property}. We leveraged Burkholder's geometric characterization of $\UMD$ spaces to derive efficient algorithms based on Burkholder/Bellman functions. Most importantly, we showed that achieving tight data-dependent regret bounds such as $\RadH$ reduces to the crisp mathematical task of exhibiting a Burkholder function with the \emph{zig-zag concavity} property. We used this observation to give efficient algorithms for classes based on $\ls_p$ norms and group norms, and to derive improved rates for settings such as matrix prediction and learning with low-rank experts.

This work leaves open a plethora of new directions centered around applying the Burkholder function method in online learning and optimization.

\paragraph{Related work}
\citep{FosRakSri15} was the first work to explore data-dependent regret bounds via symmetrization techniques, but focused on non-constructive results instead of developing efficient algorithms. The present work extends the algorithmic directions proposed in that paper. 
\paragraph{General function classes}
Much of the existing work on adapting to data in online learning focuses on the experts setting, where of particular interest are small loss or $L^{\star}$-type bounds. Existing $\UMD$ results fall short in this setting because they have only been developed for the symmetric setting of the $\ls_{1}$ ball, a superset of the probability simplex, thus leading to looser bounds. Extending our algorithmic results to non-symmetric sets like the simplex and more generally abstract function classes as in \pref{eq:UMDp} is an interesting direction for future research. 

\paragraph{Designing $\burk$ functions}
The design of $\burk$ functions and related objects called Bellman functions has witnessed significant research activity in areas from harmonic analysis to optimal stopping and stochastic optimal control \citep{osekowski2012sharp,nazarov1996hunt,nazarov2001bellman}. The applicability to our setting has been limited so far by a focus on bounds that have sharp constants and are dimension- and horizon-independent. We anticipate that designing new $\burk$ functions from a computer science perspective --- for example, exploiting that we are tolerant to logarithmic factors in most settings --- will allow us to unlock the full power of these techniques for learning applications. 
\dfedit{One such example --- an elementary derivation of a scalar $\burk$ function with sub-optimal constants --- is given in the appendix as  \pref{thm:burkholder_elementary}.}

\paragraph{Beyond UMD}
$\UMD$ is far from the only martingale inequality that can be certified using Burkholder functions. For example, the textbook \citep{osekowski2012sharp} applies the Burkholder technique to inequalities all across probability, in both discrete and continuous time. We anticipate that this technique will find extensive application in and around online learning for a wide range of settings and performance measures.

\paragraph{Online linear and online convex optimization}
All of the algorithmic techniques proposed in this paper immediately extend to the online linear optimization and online convex optimization settings to yield analogous results, but their applicability is currently limited by the fact that the predictions made by these algorithms do not lie in a fixed range. The necessary and sufficient conditions extend as well, and we will flesh out these results in the full version of the paper.

\paragraph{Strongly convex losses}
The $\RadH$ bound is not tight for strongly convex losses such as the square loss. \emph{Offset rademacher complexity} techniques have been used to obtain tight worst-case rates in this case \citep{RakSri14a}. Developing $\UMD$-type inequalities for the offset Rademacher complexity will yield tighter distribution-dependent rates for regression tasks where strong convexity plays an important role.

\section*{Acknowledgements}
We thank Elad Hazan and Adam Os\k{e}kowski for helpful discussions. D.F. is supported in part by the NDSEG fellowship. Research is supported in part by the NSF under grants no. CDS\&E-MSS 1521529 and 1521544. Part of this work was performed while D.F. and K.S. were visiting the Simons Institute for the Theory of Computing and A.R. was visiting MIT.

\bibliography{paper}

\appendix

\section{Proofs}
\label{app:proofs}

\begin{proof}[Proof of \pref{lem:optimal_b}]
Recall that $\lsh(\yh, y) = \max\crl*{0, 1 - \yh\cdot{}y}$, $\lsa(\yh, y)=\abs*{\yh - y}$, $\lsl(\yh, y)=-\yh\cdot{}y$.
Fix a sequence $\xr[n]$, and let $y_t=\eps_t$ where $\eps\in\pmo^n$ is a Rademacher sequence. By our hypothesis, we have
\[
\mc{B}(\xr[n])\geq{}\En_{\eps}\brk*{\sum_{t=1}^{n}\ls(\yh_t, \eps_t) - \inf_{f\in\F}\sum_{t=1}^{n}\ls(f(x_t), \eps_t)} \geq{} \En_{\eps}\brk*{- \inf_{f\in\F}\sum_{t=1}^{n}\ls(f(x_t), \eps_t)},
\]
where the second inequality follows from convexity of each loss with respect to $\eps_t$, and that $\yh_t$ cannot adapt to $\eps_t$. Now, since $f(x_t)\in\pmu$ and $y_t\in\pmo$, for each loss we will have 
{\small
\[
\En_{\eps}\brk*{- \inf_{f\in\F}\sum_{t=1}^{n}\ls(f(x_t), \eps_t)} = \En_{\eps}\brk*{- \inf_{f\in\F}\sum_{t=1}^{n}-f(x_t)\cdot{}\eps_t}.
\]
}The RHS is equal to $\RadH(\xr[n])$. Thus, our hypothesis implies $\RadH(\xr[n])\leq\B(\xr[n])$.

\end{proof}

\begin{proof}[Proof of \pref{prop:burkholder}]
We stress that this proof is meant to serve as a warmup exercise. See the proof of \pref{thm:alg_supervised_tight} for the correctness proof for the full \textsc{ZigZag} algorithm (\pref{alg:zigzag_supervised}), which is more computationally efficient and attains a stronger performance guarantee.

Recall that the relaxation is given by
\[
\Rel(\xr[t], \lpr[t]) = \En_{\er[t]}\burk\prn*{\sum_{s=1}^{t}\lp_sx_s, \sum_{s=1}^{t}\eps_s{}\ls'_{s}x_s}.
\]
We first show that the initial condition property is satisfied.
\paragraph{Initial Condition}
\begin{align*}
\intertext{The initial value of the online learning game is:}
&\sum_{t=1}^{n}\ls(\yh_t, y_t) - \inf_{f\in\F}\sum_{t=1}^{n}\ls(f(x_t), y_t) - \Dconstt\cdot{}\RadH(\xr[n], \lpr[n]).
\intertext{Linearizing as in \pref{eq:linearize_and_dualize} and expanding out $\RadH$, we have}
&\leq \sum_{t=1}^{n}\yh_t\lp_t + \nrm*{\sum_{t=1}^{n}\lp_tx_t} - \Dconstt\cdot{}\En_{\eps}\nrm*{\sum_{t=1}^{n}\eps_t\lp_tx_t}
\intertext{Now use property \ref{def:burkholder:1} of the function $\burk$:}
&\leq \sum_{t=1}^{n}\yh_t\lp_t + \En_{\eps}\burk\prn*{\sum_{t=1}^{n}\lp_tx_t, \sum_{t=1}^{n}\eps_t\lp_tx_t}.\\
&= \sum_{t=1}^{n}\yh_t\lp_t + \Rel(\xr[n], \lpr[n]).
\end{align*}
This establishes the initial condition.
\paragraph{Admissibility Condition}
First, observe that we have
\begin{align*}
&\sup_{x_t}\inf_{\yh_t}\sup_{\lp_t}\En_{\eps_t}\brk*{ \yh_t\lp_t + \Rel(x_{1:t}, \lp_{1:t}, \eps_{1:t})}\\
&=\sup_{x_t}\inf_{\yh_t}\sup_{\lp_t}\brk*{ \yh_t\lp_t + \En_{\eps_{1:t}}\burk\prn*{\sum_{s=1}^{t}\lp_sx_s, 
\sum_{s=1}^{t}\eps_t\lp_sx_s}
}.
\end{align*}
Define a function $G_t:\R\to\R$:
\[
G_{t}(\alpha) = \En_{\eps_{1:t}}\burk\prn*{\sum_{s=1}^{t-1}\lp_sx_s + \alpha{}x_t, 
\sum_{s=1}^{t-1}\eps_t\lp_sx_s + \eps_{t}\alpha{}x_t}.
\]
Zig-zag concavity (property \ref{def:burkholder:2} of $\burk$) implies that $G_t(\alpha)$ is concave in $\alpha$. With this definition, the above is equal to 
\begin{align*}
&=\sup_{x_t}\inf_{\yh_t}\sup_{\lp_t}\brk*{ \yh_t\lp_t + G_{t}(\lp_t)
}.
\intertext{Observe that the strategy prescribed in \pref{eq:strategy_warmup} is equivalent to $\yh_t = -G'_t(0)$. Moving to an upper bound by replacing the infimum with this choice of $\yh_t$, we have:}
&=\sup_{x_t}\sup_{\lp_t}\brk*{ -G'_t(0)\cdot{}\lp_t + G_{t}(\lp_t)
}.
\intertext{By concavity of $G_t$, this is upper bounded by:}
&\leq{}\sup_{x_t}G_{t}(0)\\ 
&=\Rel(x_{1:t-1}, \lp_{1:t-1}, \eps_{1:t-1}).
\end{align*}
Hence, $\Rel$ is an admissible relaxation, and if we play the strategy $\yh_t$ in \pref{eq:strategy_warmup} we will have
\[
\sum_{t=1}^{n}\ls(\yh_t, y_t) - \inf_{f\in\F}\sum_{t=1}^{n}\ls(f(x_t), y_t) - \Dconstt\cdot{}\RadH(\xr[n], \lpr[n]) \leq{} \Rel(\xr[n],\lpr[n])\leq{}\Rel(\xr[n-1],\lpr[n-1])\leq{}\ldots\leq{}\Rel(\emptyset).
\]
Finally, by property \ref{def:burkholder:3} of $\burk$, $\Rel(\emptyset)=\burk(0,0)\leq{}0$, and so the final value of the game is at most zero. This implies that the regret bound of $\RadH(\xr[n], \lpr[n])$ is achieved.
\end{proof}

\subsection{Proofs from \pref{sec:martingales1}}

\begin{proof}[Proof of \pref{thm:umd_equiv}]
For the case $p,q\in(1,\infty)$, we appeal to \pref{thm:umd_pq}.

Now consider the case $q=1$, and suppose $\UMD_{p}$ holds for $p\in(1,\infty)$ with $\Cconstt_{p}$. Then by \pref{thm:umd_pq}, $\Cconstt_{2}\leq{} 200\Cconstt_{p}$. Finally, by \pref{thm:umd_L1}, $\Cconstt_{1}\leq{}108\Cconstt_{2}\leq{}108\cdot{}200\Cconstt_{p}$.

For the converse direction, we appeal to \cite{pisier2011martingales}, Remark 8.2.4.

\end{proof}

\begin{proof}[Proof of \pref{thm:umd_upper}]
Fix some $C>0$ to be chosen later.
Define the minimax value as
\begin{equation}
\V=\dtri*{\sup_{x_t}\inf_{q_{t}\in\Delta([-B,+B])}\sup_{y_t\in\brk{-1,+1}}\En_{\yh_t\sim{}q_t}}_{t=1}^{n}\brk*{\sum_{t=1}^{n}\ls(\yh_t, y_t) - \inf_{f\in\F}\sum_{t=1}^{n}\ls(f(x_t), y_t) - C\En_{\eps}\sup_{\tau\leq{}n}\nrm*{\sum_{t=1}^{\tau}\eps_t\ls'(\yh_t, y_t)x_t}
}
\end{equation}
where $\dtri*{\star}_{t=1}^{n}$ denotes repeated application of the operator $\star$. If $\V\leq{}A$, then there is some randomized strategy making predictions in $[-B, +B]$ whose regret is bounded by $C\En\En_{\eps}\sup_{\tau\leq{}n}\nrm*{\sum_{t=1}^{\tau}\eps_t\ls'(\yh_t, y_t)x_t} + A$ --- see \cite{FosRakSri15} for a more detailed discussion of this principle.

In view of the inequality \pref{eq:linearize_and_dualize}, 
\begin{align*}
\V&\leq
\dtri*{\sup_{x_t}\inf_{q_{t}\in\Delta([-B,+B])}\sup_{y_t\in\brk{-1,+1}}\En_{\yh_t\sim{}q_t}}_{t=1}^{n}\brk*{\sum_{t=1}^{n}\ls'(\yh_t, y_t)\yh_t + \nrm*{\sum_{t=1}^{n}\ls'(\yh_t, y_t)x_t} - C\En_{\eps}\sup_{\tau\leq{}n}\nrm*{\sum_{t=1}^{\tau}\eps_t\ls'(\yh_t, y_t)x_t}.
}
\end{align*}
Using the (now standard) minimax theorem swap technique --- see \cite{FosRakSri15}\footnote{A word of caution: we use the assumption on the loss that there exists a minimizer for every label within some bounded domain exactly for this reason that we can now use minimax theorem restricting $\hat{y}_t$'s to be in bounded domain.}
--- the last expression is equal to
\begin{align*}
\dtri*{\sup_{x_t}\sup_{p_t\in\Delta(\brk{-1,+1})}\inf_{\yh_t\in\brk{-B,+B}}\En_{y_t\sim{}p_t}}_{t=1}^{n}\brk*{\sum_{t=1}^{n}\ls'(\yh_t, y_t)\yh_t + \nrm*{\sum_{t=1}^{n}\ls'(\yh_t, y_t)x_t} - C\En_{\eps}\sup_{\tau\leq{}n}\nrm*{\sum_{t=1}^{\tau}\eps_t\ls'(\yh_t, y_t)x_t}.
}
\end{align*}
Choose $\yh_t^{\star}=\argmin_{f} \En_{y_t\sim{}p_t}\brk*{\ls(f,y_t)}$. By our assumption on the loss,  the minimizer is obtained in $[-B,B]$ and $\En_{y_t\sim{}p_t}\brk*{\ls'(\yh^{\star}_t,y_t)}=0$. With this (sub)optimal choice, we obtain an upper bound of 
\begin{align*}
\dtri*{\sup_{x_t}\sup_{p_t\in\Delta(\brk{-1,+1})}\En_{y_t\sim{}p_t}}_{t=1}^{n}\brk*{\sum_{t=1}^{n}\ls'(\yh^{\star}_t, y_t)\yh^{\star}_t + \nrm*{\sum_{t=1}^{n}\ls'(\yh^{\star}_t, y_t)x_t} - C\En_{\eps}\sup_{\tau\leq{}n}\nrm*{\sum_{t=1}^{\tau}\eps_t\ls'(\yh^{\star}_t, y_t)x_t}.
}
\end{align*}
Since $\yh^{\star}_t$ is the population minimizer, we have $\En_{y_t\sim{}p_t}\brk*{\ls'(\yh^{\star}_t, y_t)\yh^{\star}_t} = \En_{y_t\sim{}p_t}\brk*{\ls'(\yh^{\star}_t, y_t)}\yh^{\star}_t=0$.
The proceeding expression is then equal to
\begin{align*}
&\dtri*{\sup_{x_t}\sup_{p_t\in\Delta(\brk{-1,+1})}\En_{y_t\sim{}p_t}}_{t=1}^{n}\brk*{\nrm*{\sum_{t=1}^{n}\ls'(\yh^{\star}_t, y_t)x_t} - C\En_{\eps}\sup_{\tau\leq{}n}\nrm*{\sum_{t=1}^{\tau}\eps_t\ls'(\yh^{\star}_t, y_t)x_t}
}\\
&\leq{}
\dtri*{\sup_{x_t}\sup_{p_t\in\Delta(\brk{-1,+1})}\En_{y_t\sim{}p_t}}_{t=1}^{n}\brk*{\sup_{\tau\leq{}n}\nrm*{\sum_{t=1}^{\tau}\ls'(\yh^{\star}_t, y_t)x_t} - C\En_{\eps}\sup_{\tau\leq{}n}\nrm*{\sum_{t=1}^{\tau}\eps_t\ls'(\yh^{\star}_t, y_t)x_t}
}.
\end{align*}
Observe that we may rewrite the above expression as 
\[
\sup_{\x}\sup_{P}\En_{\yr[n]\sim{}P}\brk*{\sup_{\tau\leq{}n}\nrm*{\sum_{t=1}^{\tau}\ls'(\yh^{\star}_t(p_{1:t}), y_t)\x_t(\yr[t-1])} - C\En_{\eps}\sup_{\tau\leq{}n}\nrm*{\sum_{t=1}^{\tau}\eps_t\ls'(\yh^{\star}_t(p_{1:t}), y_t)\x_t(\yr[t-1])}
},
\]
where $P=(p_1,\ldots,p_n)$ is a sequence of conditional distributions over $\yr[n]$, $\x$ is a sequence of mappings $\x_{t}:\Y^{t-1}\to{}\X$, and $\yh_t^{\star}(p_{1:t})$ is the minimizer policy described above. For any fixed choice for $P$ and $\x$, we have that  $(\ls'(\yh^{\star}_t(p_{1:t}), y_t)\x_t(\yr[t-1]))_{t\leq{}n}$ is a martingale difference sequence, because the choice of $\yh_t^{\star}$ guarantees $\En\brk*{\ls'(\yh^{\star}_t(p_{1:t}), y_t)\x_t(\yr[t-1])\mid{}\yr[t-1]}=0$. 

Therefore, if $\UMD_{1}$ holds with constant $\Cconstt_{1}$, we have (by choosing a uniform random sign sequence in \pref{def:umd}) that for any fixed $P$, $\x$,
\[
\En\sup_{\tau\leq{}n}\nrm*{\sum_{t=1}^{\tau}\ls'(\yh^{\star}_t(p_{1:t}), y_t)\x_t(\yr[t-1])}\leq{}\Cconstt_{1}\En\En_{\eps}\sup_{\tau\leq{}n}\nrm*{\sum_{t=1}^{\tau}\eps_t\ls'(\yh^{\star}_t(p_{1:t}), y_t)\x_t(\yr[t-1])}.
\] 
This implies that the inequality holds for the supremum over $P$ and $\x$, so we have
\begin{align*}
\mc{V}&\leq{}
\dtri*{\sup_{x_t}\sup_{p_t\in\Delta(\brk{-1,+1})}\En_{y_t\sim{}p_t}}_{t=1}^{n}\brk*{\Cconstt_{1}\En_{\eps}\sup_{\tau\leq{}n}\nrm*{\sum_{t=1}^{\tau}\eps_t\ls'(\yh^{\star}_t, y_t)x_t} - C\En_{\eps}\sup_{\tau\leq{}n}\nrm*{\sum_{t=1}^{\tau}\eps_t\ls'(\yh^{\star}_t, y_t)x_t}
}.
\intertext{Thus, if we take $C\geq{}\Cconstt_{1}$:}
&\leq{} 0.
\end{align*}

We have established that there exists a strategy $(\yh_t)$ guaranteeing
\[
\En\brk*{\sum_{t=1}^{n}\ls(\yh_t, y_t) - \inf_{f\in\F}\sum_{t=1}^{n}\ls(f(x_t), y_t)} \leq{} \Cconstt_1\En\En_{\eps}\sup_{\tau\leq{}n}\nrm*{\sum_{t=1}^{\tau}\eps_t\ls'(\yh_t, y_t)x_t}
\]
Treating $(\ls'(\yh_t, y_t)x_t)_{t\leq{}n}$ as a fixed sequence, we may now apply \pref{corr:concentration_maximal_exp} to remove the supremum over end times:
\[
\leq{} 4\Cconstt_1\En_{\eps}\nrm*{\sum_{t=1}^{n}\eps_t\ls'(\yh_t, y_t)x_t} + 5\Cconstt_{1}\max_{t\in\brk{n}}\nrm{x_t}\log(n).
\]
By the standard contraction argument for Rademacher complexity, since $\abs*{\ls'}\leq{}1$, 
\[
 \leq{} 4\Cconstt_1\En_{\eps}\nrm*{\sum_{t=1}^{n}\eps_tx_t} + 5\Cconstt_{1}\max_{t\in\brk{n}}\nrm{x_t}\log(n).
\]
Finally, recall that by \pref{thm:umd_equiv}, $\Cconstt_{1}\leq{}O(\Cconstt_{p})$.

\end{proof}

\begin{proof}[Proof of \pref{thm:umd_constants}]
Most of the proofs in this theorem use the following fact: If $(X_t)_{t\leq{}n}$ is a martingale difference sequence, its restriction to a subset of coordinates is also a martingale difference sequence. This allows one to prove the deterministic $\UMD$ property \eqref{eq:umdp} for complex spaces by building up from simpler spaces.
\begin{itemize}
\item $(\R, \abs{\cdot})$: \cite{burkholder1984boundary} shows that for all $p\in(1,\infty)$, $\Cconstt_p=p^{\star}-1$.
\item $(\R^{d}, \nrm{\cdot}_{p})$, for $p\in(1,\infty)$:
\begin{equation}
\En_{X}\nrm*{\sum_{t=1}^{n}\eps_tX_t}_{p}^{p} = \sum_{i\in\brk{d}}\En_{X}\abs*{\sum_{t=1}^{n}\eps_tX_t[i]}^{p} \leq{}  (p^{\star}-1)\sum_{i\in\brk{d}}\En_{X}\abs*{\sum_{t=1}^{n}X_t[i]}^{p} = (p^{\star}-1)\En_{X}\nrm*{\sum_{t=1}^{n}X_t}_{p}^{p}.
\end{equation}
The middle inequality here uses the $\UMD_p$ constant for the scalar case.
\item $(\R^{d}, \nrm{\cdot}_{p})$, for $p\in\crl*{1, \infty}$: We will start with $\ls_{\infty}$. Set $p=\log{}d$, and observe that for $\ls_{p}$, by \pref{thm:umd_pq}, $\ls_p$ has $\Cconstt_{2}=O(\Cconstt_{p})=O(p^{\star})$ (the second bound is from the previous example). Then we have, for any sequence of signs,
\begin{align*}
\En\nrm*{\sum_{t=1}^{n}\eps_tX_t}_{\infty}^{2} &\leq{} \En\nrm*{\sum_{t=1}^{n}\eps_tX_t}_{p}^{2} \\
&\leq{} O(p^{\star})\En\nrm*{\sum_{t=1}^{n}X_t}_{p}^{2} \\
&\leq{} O(p^{\star})\En\prn*{d^{1/p}\nrm*{\sum_{t=1}^{n}X_t}_{\infty}}^{2}.
\end{align*}
Since $d^{1/\log{}d}=O(1)$, the last expression is at most
\begin{align*}
&O(p^{\star})\En\nrm*{\sum_{t=1}^{n}X_t}_{\infty}^{2}.
\end{align*}
Finally, note that $p^{\star}=O(\log{}d)$.\\
The same argument works for the $\ls_{1}$ norm using $p=1+1/\log{}d$. Alternatively, the constant can be deduced from duality using \pref{thm:umd_dual}. That these constants are optimal follows from \cite{veraar2015analysis}, Proposition 4.2.19.
 \item $(\R^{d}, \nrm{\cdot}_{\mc{A}}/\nrm{\cdot}_{\mc{A}^{\star}})$. Let us focus on $\nrm{\cdot}_{\mc{A}^{\star}}$. Assume $\mc{A}=\crl*{a_1,\ldots,a_{N}}$. Observe that 
 \begin{align*}
 \nrm{x}_{\mc{A}^{\star}}&=\max\crl*{\tri*{y,x}\mid{}y\in\textrm{conv}(\mc{A})}\\
 &=\max\crl*{\sum_{i\in\brk{N}}\theta_i\tri*{a_i,x_i}\mid{}\theta\in\Delta(N)}
 \intertext{Since we assumed $\mc{A}$ is symmetric:}
 &= \nrm*{(\tri*{a_i,x_i})_{i\in\brk{N}})}_{\infty}\\
 &= \nrm*{Ax}_{\infty}\text{, where $A\in\R^{N\times{}d}$ is the matrix of elements of $\mc{A}$ stacked as rows.}
 \end{align*}
For any martingale difference sequence $(X_t)_{t\leq{}n}$,  $(AX_t)_{t\leq{}n}$ is also a martingale difference. Therefore, we can deduce the $\UMD_2$ property for $\nrm{\cdot}_{\mc{A}^{\star}}$ from our result for $\nrm{\cdot}_{\infty}$. The $\UMD_{2}$ property for $\nrm{\cdot}_{\mc{A}}$ follows from \pref{thm:umd_dual}.

\item $(\R^{d\times{}d}, \nrm{\cdot}_{S_{p}})$, for $p\in(1,\infty)$: \cite{veraar2015analysis} Theorem 5.2.10 and Proposition 5.5.5.
\item $(\R^{d\times{}d}, \nrm{\cdot}_{\sigma})$: $\Cconstt_{2}=O(\log^{2}d)$. We will build up from the Schatten $p$-norms in the same fashion as for the $\ls_p$ spaces. Let $p=\log{}d$. For any sequence of signs,
\begin{align*}
\En\nrm*{\sum_{t=1}^{n}\eps_tX_t}_{\sigma}^{2} &\leq{} \En\nrm*{\sum_{t=1}^{n}\eps_tX_t}_{S_{p}}^{2}. \\
\intertext{Using \pref{thm:umd_pq} to get $C_{2}\leq{}O((p^{\star})^{2})$ for $S_p$:}
&\leq{} O((p^{\star})^{2})\En\nrm*{\sum_{t=1}^{n}X_t}_{S_p}^{2} \\
&\leq{} O((p^{\star})^{2})\En\prn*{d^{1/p}\nrm*{\sum_{t=1}^{n}X_t}_{\sigma}}^{2}.
\end{align*}
Since $d^{1/\log{}d}=O(1)$, the preceding expression is at most
\begin{align*}
&O((p^{\star})^{2})\En\nrm*{\sum_{t=1}^{n}X_t}_{\sigma}^{2}.
\end{align*}
Once again, $p^{\star}\leq{}\log{}d$. The constant for $\nrm{\cdot}_{\Sigma}$ follows from \pref{thm:umd_dual}, since the trace norm is dual to the spectral norm.
\item $(\R^{d\times{}d}, \nrm{\cdot}_{p,q})$, for $p,q\in(1,\infty)$: For any sequence of signs, we apply the UMD property for $\ls_p$ row-wise:
\begin{align*}
\En\nrm*{\sum_{t=1}^{n}\eps_tX_t}_{p,q}^{p} &= \sum_{i\in\brk{d}}\En\nrm*{\sum_{t=1}^{n}\eps_t(X_t)_{i\cdot{}}}_{q}^{p} .
\intertext{We know $\ls_{q}$ has $\Cconstt_{q}\leq{}q^{\star}$. By \pref{thm:umd_pq}, this implies that $\Cconstt_{p}$ for $\ls_{q}$ has $\Cconstt_{p}\leq{}O(p^{\star}\cdot{}q^{\star})$.}
&\leq{} O(p^{\star}\cdot{}q^{\star})\sum_{i\in\brk{d}}\En\nrm*{\sum_{t=1}^{n}(X_t)_{i\cdot{}}}_{q}^{p} \\
&= O(p^{\star}\cdot{}q^{\star})\En\nrm*{\sum_{t=1}^{n}X_t}_{p,q}^{p}.
\end{align*}
\item $(\H, \nrm{\cdot}_{\mc{H}})$ for any Hilbert space $\mc{H}$: See \pref{ex:burkholder_hilbert}.
\end{itemize}

\end{proof}

\subsection{Proofs from \pref{sec:algorithms}}

\subsubsection{Proofs for \pref{alg:zigzag_supervised}}

\begin{proof}[Proof of \pref{thm:alg_supervised_tight}]~
We will show that the strategy achieves the regret bound
\begin{equation}
\label{eq:zz_reg}
\En_{\eps}\brk*{\sum_{t=1}^{n}\ls(\yh_t^{\eps_{1:t-1}}, y_t) - \inf_{f\in\F}\sum_{t=1}^{n}\ls(f(x_t), y_t) - \Psi_{\eta,p}\prn*{\beta{}^{p}\nrm*{\sum_{t=1}^{n}\eps_t\ls'(\yh_t^{\eps_{1:t-1}}, y_t)x_t}^{p}}}\leq{}0.
\end{equation}
Our proof technique is to define a relaxation
\[
\Rel(x_{1:t}, \lp_{1:t}, \eps_{1:t}) = \frac{\eta}{p}\burk_{p}\prn*{\sum_{s=1}^{t}\lp_sx_s, \sum_{t=1}^{t}\eps_s\lp_sx_s}.
\]
and show that the relaxation is admissible for the following game:
\begin{equation}
\dtri*{\sup_{x_t}\inf_{\yh_t}\sup_{\lp_t}\En_{\eps_t}}_{t=1}^{n}\brk*{\sum_{t=1}^{n}\yh_t\lp_t -\inf_{f\in\F}\sum_{t=1}^{n}f(x_t)\lp_t -\Psi_{\eta,p}\prn*{\beta{}^{p}\nrm*{\sum_{t=1}^{n}\eps_t\lp_tx_t}^{p}}
}.
\end{equation}
This relaxation is slightly generalized compared to \pref{def:relaxation} in that Rademacher sequence $(\eps_t)_{t\leq{}n}$ also appears as an argument. This is essential to accomplish the coupling of the algorithm's randomness and the regret functional $\RadH$.

With the game defined we can proceed to showing that the relaxation satisfies the admissibility and initial conditions, with one extra step of linearization in the initial condition.
\paragraph{Initial Condition}
In view of \pref{eq:linearize_and_dualize},
\begin{align*}
&\sum_{t=1}^{n}\ls(\yh_t, y_t) - \inf_{f\in\F}\sum_{t=1}^{n}\ls(f(x_t), y_t) - \Psi_{\eta,p}\prn*{\beta{}^{p}\nrm*{\sum_{t=1}^{n}\eps_t\ls'(\yh_t, y_t)x_t}^{p}}\\
&\leq \sum_{t=1}^{n}\yh_t\lp_t + \nrm*{\sum_{t=1}^{n}\lp_tx_t} -\Psi_{\eta,p}\prn*{\beta{}^{p}\nrm*{\sum_{t=1}^{n}\eps_t\lp_tx_t}^{p}}\\
&\leq{}\sum_{t=1}^{n}\yh_t\ls'_t + \Psi_{\eta,p}\prn*{\nrm*{\sum_{t=1}^{n}\lp_tx_t}^{p}} -\Psi_{\eta,p}\prn*{\beta{}^{p}\nrm*{\sum_{t=1}^{n}\eps_t\lp_tx_t}^{p}}\\
&=\sum_{t=1}^{n}\yh_t\ls'_t + \frac{\eta}{p}\prn*{\nrm*{\sum_{t=1}^{n}\lp_tx_t}^{p} -\beta{}^{p}\nrm*{\sum_{t=1}^{n}\eps_t\lp_tx_t}^{p}}\\
&\leq{}\sum_{t=1}^{n}\yh_t\ls'_t + \frac{\eta}{p}\burk_{p}\prn*{\sum_{t=1}^{n}\lp_tx_t, \sum_{t=1}^{n}\eps_t\lp_tx_t}\\
&= \sum_{t=1}^{n}\yh_t\lp_{t}  + \Rel(x_{1:n}, \lp_{1:n}, \eps_{1:n}).
\end{align*}
\paragraph{Admissibility Condition}
\begin{align*}
&\sup_{x_t}\inf_{\yh_t}\sup_{\lp_t}\En_{\eps_t}\brk*{ \yh_t\lp_t + \Rel(x_{1:t}, \lp_{1:t}, \eps_{1:t})}\\
&=\sup_{x_t}\inf_{\yh_t}\sup_{\lp_t}\En_{\eps_t}\brk*{ \yh_t\lp_t + \frac{\eta}{p}\burk_{p}\prn*{\sum_{s=1}^{t}\lp_sx_s, \sum_{t=1}^{t}\eps_s\lp_sx_s}}\\ 
&=\sup_{x_t}\inf_{\yh_t}\sup_{\lp_t}\brk*{ \yh_t\lp_t + \En_{\eps_t}\frac{\eta}{p}\burk_{p}\prn*{\sum_{s=1}^{t}\lp_sx_s, \sum_{t=1}^{t}\eps_s\lp_sx_s}}\\ 
&=\sup_{x_t}\inf_{\yh_t}\sup_{\lp_t}\brk*{ \yh_t\lp_t + G_{t}(\lp_t)}
\end{align*}
Pluggin in the strategy specified by \pref{alg:zigzag_supervised}, the last expression is at most
\begin{align*}
&\sup_{x_t}\sup_{\lp_t}\brk*{ -G_{t}'(0)\cdot{}\lp_t + G_{t}(\lp_t)}\\ 
&\leq{}\sup_{x_t}G_{t}(0)\\ 
&=\Rel(x_{1:t-1}, \lp_{1:t-1}, \eps_{1:t-1}).\\ 
\end{align*}
Finally, since $\burk_{p}$ is Burkholder we have $\Rel(\cdot)\propto\burk_{p}(0,0)\leq{}0$, and so the final value of the game is at most zero. This implies that \pref{eq:zz_reg} is achieved.
\end{proof}

\begin{proof}[Proof of \pref{lem:algorithm_supervised_doubling}]
In what follows we will leave the dependence of $\yh_{t}, x_{t}, \ls'_{t}$ on $\eps_{1:t-1}$ implicit for notational convenience. We will handle this dependence at the end of the proof.

Assume $N>1$. Otherwise, the algorithm's regret is bounded as $2\eta_{1}^{-(p'-1)}=4\eta_{0}^{-(p'-1)}$.

\begin{align*}
\En_{\eps}\brk*{\sum_{t=1}^{n}\ls(\yh_t^{\eps_{1:t-1}}, y_t) - \inf_{f\in\F}\sum_{t=1}^{n}\ls(f(x_t), y_t)} &\leq{}
\En_{\eps}\brk*{\sum_{i=1}^{N}\brk*{\sum_{t=s_i}^{s_{i+1}-1}\ls(\yh_t^{\eps_{1:t-1}}, y_t) - \inf_{f\in\F}\sum_{t=s_i}^{s_{i+1}-1}\ls(f(x_t), y_t)}}
\intertext{Using the regret bound for \pref{alg:zigzag_supervised} (note that that algorithm has an anytime regret guarantee) given by \pref{thm:alg_supervised_tight}:}
&\leq{}
\En_{\eps}\brk*{\frac{1}{p}\sum_{i=1}^{N}\brk*{\eta{}_i\beta_{p}^{p}\nrm*{\sum_{t=s_i}^{s_{i+1}-1}\eps_{t}\lp_tx_{t}}^{p} + \frac{1}{p'-1}\eta_i^{-(p'-1)}}}.
\intertext{Introducing a new supremum:}
&\leq{}
\En_{\eps}\brk*{\frac{1}{p}\sum_{i=1}^{N}\brk*{\eta{}_i\Phi(x_{s_i:s_{i+1}-1}, \lp_{s_i:s_{i+1}-1}, \eps_{s_i:s_{i+1}-1}) + \frac{1}{p'-1}\eta_i^{-(p'-1)}}}.
\end{align*}
The doubling condition implies that $\eta_{i}\Phi(x_{s_i:s_{i+1}-2}, \lp_{s_i:s_{i+1}-2}, \eps_{s_i:s_{i+1}-2})\leq{}\eta_{i}^{-(p'-1)}$. To use this fact, observe that since $\nrm{x_t}\leq{}1$, we have that for any $C>0$,
\begin{align*}
&\eta_{i}\Phi(x_{s_i:s_{i+1}-1}, \lp_{s_i:s_{i+1}-1}, \eps_{s_i:s_{i+1}-1}) \\
&=\eta_{i}\beta_{p}^{p}\sup_{s_i\leq{}a\leq{}b\leq{}s_{i+1}-1}\nrm*{\sum_{t=a}^{b}\eps_t\lp_tx_t}^p\\
&\leq{} \eta_{i}(1+1/C)^p\beta_{p}^{p}\sup_{s_i\leq{}a\leq{}b\leq{}s_{i+1}-2}\nrm*{\sum_{t=a}^{b}\eps_t\lp_tx_t}^p + \eta_{i}C^p\beta_{p}^{p}.
\intertext{For $C=p$:}
&\leq{} \eta_{i}e\Phi(x_{s_i:s_{i+1}-2}, \eps_{s_i:s_{i+1}-2})+ \eta_{i}p^p\beta_{p}^{p}.\\
&= e\eta_{i}^{-(p'-1)}+ \eta_{i}p^p\beta_{p}^{p}.
\end{align*}
Returning to the regret bound, we have
\begin{align*}
&\leq{}
\En_{\eps}\brk*{\frac{1}{p}\sum_{i=1}^{N}\brk*{e\eta_{i}^{-(p'-1)}+ \eta_{i}p^p\beta_{p}^{p} + \frac{1}{p'-1}\eta_i^{-(p'-1)}}}\\
&\leq{}
\En_{\eps}\brk*{e\sum_{i=1}^{N}\eta_{i}^{-(p'-1)} + p^{p}\beta_{p}^{p}\eta_{i}}
\end{align*}
We will deal with the left-hand term first.
We now observe that $\eta_{N-1}\Phi(x_{s_{N-1}:s_{N}}, \lp_{s_{N-1}:s_{N}}, \eps_{s_{N-1}:s_{N}})>\eta_{N-1}^{-(p'-1)}$. Rearranging further implies $\eta_{N-1}^{-(p'-1)}\leq{} \Phi(x_{s_{N-1}:s_{N}}, \lp_{s_{N-1}:s_{N}}, \eps_{s_{N-1}:s_{N}})^{1/p} \leq \Phi(x_{1:n}, \lp_{1:n}, \eps_{1:n})^{1/p}$. Finally, since $\eta_i^{-(p'-1)}=2\eta_{i-1}^{-(p'-1)}$, 
\[
\sum_{i=1}^{N}\eta_i^{-(p'-1)}=\eta_{0}^{-(p'-1)}\sum_{i=1}^{N}2^i \leq{} 2\cdot{}2^{N}\eta_{0}^{-(p'-1)}\leq{}4\Phi(x_{1:n},\lp_{1:n}, \eps_{1:n})^{1/p}= 4\beta_{p}\sup_{1\leq{}a\leq{}b\leq{}n}\nrm*{\sum_{t=a}^{b}\eps_t\lp_tx_t}.
\]
For the second term, observe that $\eta_i\leq\eta_0$ for all $i$, so
\[
\sum_{i=1}^{N}p^{p}\beta_{p}^{p}\eta_{i}\leq{} p^{p}\beta_{p}^{p}\eta_0\cdot{}N.
\]
Finally, by the invariant $2^{N-1}\eta_{0}^{-(p'-1)}\leq{}\Phi(x_{1:n}, \eps_{1:n})^{1/p}$ we established earlier,
\[
N\leq{}\log\prn*{\Phi(x_{1:n}, \lp_{1:n}, \eps_{1:n})^{1/p} \eta_{0}^{(p'-1)}}+1
\]
Putting everything together, the regret is bounded as
\begin{align*}
&\En_{\eps}\max\crl*{2e\beta_{p}\sup_{1\leq{}a\leq{}b\leq{}n}\nrm*{\sum_{t=a}^{b}\eps_t\lp_tx_t} + p^{p}\beta_{p}^{p}\eta_{0}\prn*{\log\prn*{\sup_{1\leq{}a\leq{}b\leq{}n}\nrm*{\sum_{t=a}^{b}\eps_t\lp_tx_t}\eta_{0}^{(p'-1)}}+1}, 4\eta_{0}^{-(p'-1)}}\\
&\leq{}
\En_{\eps}\brk*{2e\beta_{p}\sup_{1\leq{}a\leq{}b\leq{}n}\nrm*{\sum_{t=a}^{b}\eps_t\lp_tx_t} + p^{p}\beta_{p}^{p}\eta_{0}\prn*{\log\prn*{\sup_{1\leq{}a\leq{}b\leq{}n}\nrm*{\sum_{t=a}^{b}\eps_t\lp_tx_t}\eta_{0}^{(p'-1)}}+1}
 +  4\eta_{0}^{-(p'-1)}}
 \intertext{Using that $\nrm{x_t}\leq{}1$:}
 &\leq{}
2e\beta_{p}\En_{\eps}\sup_{1\leq{}a\leq{}b\leq{}n}\nrm*{\sum_{t=a}^{b}\eps_t\lp_tx_t} + p^{p}\beta_{p}^{p}\eta_{0}\log{}\prn*{n\cdot{}\eta_{0}^{(p'-1)}}
 +  4\eta_{0}^{-(p'-1)}.
 \intertext{For the choice $\eta_0=(\beta_p\cdot{}p)^{-p}$:}
  &\leq{}
2e\beta_{p}\En_{\eps}\sup_{1\leq{}a\leq{}b\leq{}n}\nrm*{\sum_{t=a}^{b}\eps_t\lp_tx_t} + \log{}\prn*{n}
 +  (p\cdot{}\beta_p)^{\frac{p}{p-1}}.
  \intertext{For the choice $\eta_0=1$:}
 &\leq{}
2e\beta_{p}\En_{\eps}\sup_{1\leq{}a\leq{}b\leq{}n}\nrm*{\sum_{t=a}^{b}\eps_t\lp_tx_t} + p^{p}\beta_{p}^{p}\log{}\prn*{n}
 +  4.
\end{align*}

Writing $x_{t}(\eps_{1:t-1})$ and $\ls'_{t}(\eps_{1:t-1})$ to make the adversary's dependence on the sequence $\eps$ explicit, the main term of interest in the above quantity is
\[
\En_{\eps}\sup_{1\leq{}a\leq{}b\leq{}n}\nrm*{\sum_{t=a}^{b}\eps_t\lp_t(\eps_{1:t-1})x_t(\eps_{1:t-1})}.
\]

It remains to remove the supremum and decouple the data sequences $(x_{t})$ and $(\ls'_{t})$ from the Rademacher sequence $\eps$. Since $\ls'_{t}x_{t}$ can only react to $\eps_{1:t-1}$, the sequence $(\eps_{t}\ls'_{t}x_{t})_{t\leq{}n}$ is a martingale difference sequence. Since $\nrm*{\sum_{t=a}^{b}\eps_t\lp_tx_t}\leq{}n$, we may apply \pref{corr:doob_p1} to arrive at an upper bound of
\[
\leq{} O\prn*{
\log(n)\En_{\eps}\sup_{1\leq{}b\leq{}n}\nrm*{\sum_{t=1}^{b}\eps_t\lp_t(\eps_{1:t-1})x_t(\eps_{1:t-1})}
}.
\]
Now observe that since \pref{alg:zigzag_supervised} uses a Burkholder function $\burk_{p}$ for $(\nrm{\cdot},p,\beta_{p})$, \pref{thm:umd_burkholder} and \pref{thm:umd_equiv} together imply that the $\UMD_{1}$ inequality \pref{eq:umd1} holds with constant $O(\beta_{p})$, therefore, the above is bounded as
\[
\leq{} O\prn*{
\beta_{p}\log(n)\En_{\eps}\En_{\eps'}\sup_{1\leq{}b\leq{}n}\nrm*{\sum_{t=1}^{b}\eps'_t\lp_t(\eps_{1:t-1})x_t(\eps_{1:t-1})}
}.
\]
Note that the variables $(x_{t})$ and $(\ls'_{t})$ no longer depend on the Rademacher sequence appearing in the sum. Lastly, we apply \pref{corr:doob_p1} once more to remove the last supremum and arrive at the bound,
\[
\leq{} O\prn*{
\beta_{p}\log^{2}(n)\En_{\eps}\En_{\eps'}\nrm*{\sum_{t=1}^{b}\eps'_t\lp_t(\eps_{1:t-1})x_t(\eps_{1:t-1})}
}.
\]
\end{proof}

\begin{proof}[Proof of \pref{ex:lp_alg}]
\pref{eq:zigzag_lp} is obtained by plugging the optimal UMD constant $p^{\star}-1$ into the bound for \pref{lem:algorithm_supervised_doubling}. For \pref{eq:zigzag_p2}, observe that for any sequence $z_t$ we have $\En_{\eps}\nrm*{\sum_{t=1}^{n}\eps_tz_t}_{2}\leq \sqrt{\En_{\eps}\nrm*{\sum_{t=1}^{n}\eps_tz_t}_{2}^{2}}= \sqrt{\En_{\eps}\sum_{t=1}^{n}\nrm*{z_t}_{2}^{2}}$. Applying this fact with the algorithm's bound for $p=2$ gives the regret bound
\begin{equation*}
O\prn*{\sqrt{\sum_{t=1}^{n}\nrm*{\lp_tx_t}_{2}^{2}}\cdot{}\log^{2}n + \log{}n
 }.\\
 \end{equation*}
For \pref{eq:zigzag_p1}, observe that with $p=1/\log{}d$ we have the regret bound
\begin{equation*}
O\prn*{\En_{\eps}\nrm*{\sum_{t=1}^{n}\eps_t\lp_tx_t}_{p}\cdot{}\log{}d\log^{2}n + \log^{2}d\log{}n
 }.
 \end{equation*}
 However for any $X$, $\nrm{X}_{p}\leq{}d^{1-1/p}\nrm{X}_{1}$. For our choice of $p=1+1/\log{}d$ we have $d^{1-1/p}=O(1)$.
 \begin{align*}
&\leq{}O\prn*{\En_{\eps}\nrm*{\sum_{t=1}^{n}\eps_t\lp_tx_t}_{1}\cdot{}\log{}d\log^{2}n + \log^{2}d\log{}n
 }\\
 &\leq{}O\prn*{\En_{\eps}\nrm*{\sum_{t=1}^{n}\eps_tx_t}_{1}\cdot{}\log{}d\log^{2}n + \log^{2}d\log{}n
 }\\
 &=O\prn*{\sum_{i\in\brk{d}}\En_{\eps}\abs*{\sum_{t=1}^{n}\eps_tx_t[i]}\cdot{}\log{}d\log^{2}n + \log^{2}d\log{}n
 }\\
  &\leq{}O\prn*{\sum_{i\in\brk{d}}\sqrt{\sum_{t=1}^{n}(x_t[i])^2}\cdot{}\log{}d\log^{2}n + \log^{2}d\log{}n
 }\\
   &=O\prn*{\sum_{i\in\brk{d}}\nrm{x_{1:n,i}}_{2}\cdot{}\log{}d\log^{2}n + \log^{2}d\log{}n
 }.
 \end{align*}
\end{proof}

\subsubsection{Simplified doubling trick}
In this section we derive a variant of the doubling trick given in \pref{lem:algorithm_supervised_doubling} which achieves an upper bound on $\RadH$ rather than $\RadH$ itself, but does so with improved dependence on constants and low-order terms. This strategy will be used as a subroutine in subsequent algorithms.

\begin{lemma}
\label{lem:algorithm_lp1_doubling}
Suppose we have an anytime regret minimization algorithm $(\yh_{t})$ that guarantees a regret bound of the form
\[
\sum_{t=1}^{n}\ls(\yh_t, y_t) - \inf_{f\in\F}\sum_{t=1}^{n}\ls(f(x_t), y_t) \leq{} \frac{1}{p}\brk*{\eta{}K^{p}\En_{\eps}\nrm*{\sum_{t=1}^{n}\eps_tx_{t}}^{p}
 + \frac{1}{p'-1}\eta^{-(p'-1)}
},
\]
where $p>1$ is fixed and $\eta$ is a parameter of the algorithm. Define 
\[
\Phi(x_{t_1:t_2}) = \beta_{p}^{p}\En_{\eps}\sup_{t_1\leq{}a\leq{}b\leq{}t_2}\nrm*{\sum_{t=a}^{b}\eps_tx_t}^p.
\]

Consider the following strategy
\begin{enumerate}
\item Choose $\eta_0<1$ arbitrary. Update with $\eta_{i}=2^{-\frac{1}{p'-1}}\eta_{i-1}$.
\item In phase $i$, which consists of all $t\in\crl*{s_{i},\ldots,s_{i+1}-1}$, play strategy $(\yh_t)$ with learning rate $\eta_{i}$.
\item Take $s_1=1$, $s_{N+1}=n+1$, and $s_{i+1}=\inf\crl*{\tau\mid{}\eta_{i}\Phi(x_{s_i:\tau})>\eta_{i}^{-(p'-1)}}$, where $N$ is the index of the last phase.
\end{enumerate}
This strategy achieves
\begin{align*}
\sum_{t=1}^{n}\ls(\yh_t, y_t) - \inf_{f\in\F}\sum_{t=1}^{n}\ls(f(x_t), y_t)
&\leq{} K\prn*{\En_{\eps}\sup_{1\leq{}a\leq{}b\leq{}n}\nrm*{\sum_{t=a}^{b}\eps_tx_t}^p}^{1/p} + \eta_{0}^{-(p'-1)}\\
&\leq{} C\cdot{}(p')^{2}\cdot{}K\prn*{\En_{\eps}\nrm*{\sum_{t=1}^{n}\eps_tx_t}^p}^{1/p} + \eta_{0}^{-(p'-1)}.
\end{align*}
\end{lemma}

\begin{proof}[Proof of \pref{lem:algorithm_lp1_doubling}]
We assume $N>1$. Otherwise, the algorithm's regret is bounded as $2\eta_{1}^{-(p'-1)}=4\eta_{0}^{-(p'-1)}$.
\begin{align*}
\sum_{t=1}^{n}\ls(\yh_t, y_t) - \inf_{f\in\F}\sum_{t=1}^{n}\ls(f(x_t), y_t) &\leq{}
\sum_{i=1}^{N}\brk*{\sum_{t=s_i}^{s_{i+1}-1}\ls(\yh_t, y_t) - \inf_{f\in\F}\sum_{t=s_i}^{s_{i+1}-1}\ls(f(x_t), y_t)}
\intertext{Using the assumed regret bound (note that that algorithm has an anytime regret guarantee):}
&\leq{}
\frac{1}{p}\sum_{i=1}^{N}\brk*{\eta{}_iK^{p}\En_{\eps}\nrm*{\sum_{t=s_i}^{s_{i+1}-1}\eps_{t}x_{t}}^{p} + \frac{1}{p'-1}\eta_i^{-(p'-1)}}
\intertext{Introducing a new supremum:}
&\leq{}
\frac{1}{p}\sum_{i=1}^{N}\brk*{\eta{}_i\Phi(x_{s_i:s_{i+1}-1}) + \frac{1}{p'-1}\eta_i^{-(p'-1)}}
\intertext{Using the invariant $\eta_{i}\Phi(x_{s_i:s_{i+1}-1})\leq{}\eta_{i}^{-(p'-1)}$:}
&\leq{}
\frac{1}{p}\prn*{1 + \frac{1}{p'-1}}\sum_{i=1}^{N}\eta_i^{-(p'-1)}\\
&=\sum_{i=1}^{N}\eta_i^{-(p'-1)}
\end{align*}
We now observe that $\eta_{N-1}\Phi(x_{s_{N-1}:s_{N}})>\eta_{N-1}^{-(p'-1)}$. Rearranging further implies $\eta_{N-1}^{-(p'-1)}\leq{} \Phi(x_{s_{N-1}:s_{N}})^{1/p} \leq \Phi(x_{1:n})^{1/p}$. Finally, we can check that $\eta_i^{-(p'-1)}=2\eta_{i-1}^{-(p'-1)}$, so $2^{N}\eta_{0}^{-(p'-1)}\leq{}\Phi(x_{1:n})^{1/p}$.
\[
\sum_{i=1}^{N}\eta_i^{-(p'-1)}=\eta_{0}^{-(p'-1)}\sum_{i=1}^{N}2^i \leq{} 2\cdot{}2^{N}\eta_{0}^{-(p'-1)}\leq{}\Phi(x_{1:n})^{1/p}= K\prn*{\En_{\eps}\sup_{1\leq{}a\leq{}b\leq{}n}\nrm*{\sum_{t=a}^{b}\eps_tx_t}^p}^{1/p}.
\]
This gives the first inequality. For the second we just apply Doob's maximal inequality.
In particular, let $Z_{b}=\sup_{1\leq{}a\leq{}b}\nrm*{\sum_{t=a}^{b}\eps_tx_t}$. Then $Z_{b}$ is a sub-martingale, so Doob's maximal inequality implies $\En_{\eps}\sup_{b\leq{}n}Z_b^{p}\leq{} (p')^{p}\En_{\eps}Z_{n}^{p}$. Applying Doob's inequality once more shows that $\En_{\eps}Z_{n}^{p}\leq{} (p')^p\En_{\eps}\nrm*{\sum_{t=1}^{n}\eps_tx_t}$, which gives the result.
\end{proof}


\subsubsection{Proofs for \pref{alg:zigzag_spectral}}

We do not know of an explicit $\burk$ function for the spectral norm. The approach we employ (\pref{alg:zigzag_spectral}) is to run many sub-algorithms for classes for which we \emph{do} have an efficient $\burk$ function (weighted Euclidean norms), then aggregate the predictions of these sub-algorithms with the multiplicative weights strategy.

Let $X_{t}=e_{i_t}\tens{}e_{j_t}$ be the incidence matrix for the entry $(i_t, j_t)$. Then we may write $F(x_t) = \tri*{F, X_t}$.

\begin{theorem}
\label{thm:spectral_regret}
Suppose $y_t\in\brk{-1, 1}$. The predictions $(\yh_t)$ produced by \pref{alg:zigzag_spectral}, for any well-behaved loss with $\ls(\yh, y)\leq{}1$ for $\abs{\yh}\leq{}1$, satisfy the regret bound,
\[
\En\brk*{\sum_{t=1}^{n}\ls(\yh_{t}, y_t) - \inf_{F\in\F}\sum_{t=1}^{n}\ls\prn*{F(x_t), y_t}} \leq{} O\prn*{\frac{\eta}{2}\tau^{2}\En_{\eps}\nrm*{\sum_{t=1}^{n}\eps_{t}X_t}_{\sigma}^{2} + \frac{\eta^{-1}}{2}  + \sqrt{nrd\log\prn{\tau{}n}}}.
\]
\end{theorem}
\begin{proposition}
\label{prop:spectral_doubling}
Using the doubling trick as in \pref{lem:algorithm_lp1_doubling}, the regret of \pref{alg:zigzag_spectral} is bounded as 
\begin{equation}
\label{eq:spectral_doubling}
\widetilde{O}\prn*{\tau\sqrt{\En_{\eps}\nrm*{\sum_{t=1}^{n}\eps_{t}X_t}_{\sigma}^{2}}  + \sqrt{nrd\log\prn{\tau{}n}}}.
\end{equation}
\end{proposition}

\begin{proof}[Proof of \pref{prop:spectral_final}]
We begin with the bound from \pref{prop:spectral_doubling} and bound $\En_{\eps}\nrm*{\sum_{t=1}^{n}\eps_{t}X_t}_{\sigma}^{2}$ to get the result.

The first step is to apply concentration to remove the expectation over $\eps$. Observe that the spectral norm of each $X_t$ is bounded by $1$ (since each $X_t$ is an the indicator matrix).  Hence, by Theorem 6.1 of \cite{tropp2012user} we have that the probability of spectral norm $\nrm*{\sum_{t=1}^n  \epsilon_t X_t}_{\sigma}$ is larger than $t$ --- for any $t > \nrm*{\sum_t X_t X_t^\trn}$ --- has a sub-exponential tail. In particular, letting $\sigma^{2}=\max\left\{\nrm*{\sum_t X_t X_t^\trn}_\sigma, \nrm*{\sum_t X_t^\trn X_t}_\sigma\right\}$, we have that with probability at least $1-\delta$ over the draw of $\eps$,
$$
\nrm*{\sum_{t=1}^n  \epsilon_t X_t}_{\sigma}^2 \le O\prn*{\sigma^{2}\log^{2}(d/\delta)}.
$$ 
Since each $X_t$ is bounded this implies that 
\[\En_{\eps}\nrm*{\sum_{t=1}^n  \epsilon_t X_t}_{\sigma}^2 \le O\prn*{\sigma^{2}\log^{2}(nd)}.
\]
Returning to \pref{eq:spectral_doubling} and recalling the value of $\sigma^{2}$, this implies a regret bound
\begin{align*}
 \sum_{t=1}^n \ell_{\mathrm{hinge}}(\hat{y}_t , y_t) -& \inf_{\substack{F : \|F\|_{\mathrm{trace}}  \le \tau,\\ \mathrm{rank}(F) \le r}} \sum_{t=1}^{n}\ell_{\mathrm{hinge}}(\left<F,X_t\right>,y_t)\\
  & \le \widetilde{O}\prn*{\tau \sqrt{\max\left\{\nrm*{\sum_t X_t X_t^\trn}_\sigma, \nrm*{\sum_t X_t^\trn X_t}_\sigma\right\}}  + \sqrt{r d n}}.\\
  \intertext{Using that $X_{t}$ are incidence matrices and so $\sum_{t}X_{t}X_{t}^{\trn}$ and $\sum_{t}X_{t}^{\trn}X_{t}$ are diagonal, a straightforward calculation reveals:}
& \le \widetilde{O}\prn*{\tau \sqrt{\max\crl*{\max_{i}\abs*{\crl*{t\mid{}i_t=i}}, \max_{j}\abs*{\crl*{t\mid{}j_t=j}}}} + \sqrt{r d n }}\\
& = \widetilde{O}\prn*{\tau \sqrt{\max\crl*{N_{\mathrm{row}}, N_{\mathrm{col}}}} + \sqrt{r d n }}.
\intertext{Now, using that $\tau=\sqrt{r}d$ and that $N_{\mathrm{row}},N_{\mathrm{col}}\geq{}n/d$ by the pigeonhole principle,}
& \le \widetilde{O}\prn*{\sqrt{r}\cdot{}d\cdot{}\sqrt{\max\crl*{N_{\mathrm{row}}, N_{\mathrm{col}}}}}.
\end{align*}

\end{proof}

\begin{algorithm}
\caption{\textsc{SpectralZigZag}}\label{alg:zigzag_spectral}
\begin{algorithmic}[1]
\Procedure{SpectralZigZag}{$\eta$, rank $r$, trace norm bound $\tau{}$}\\
Let $\mc{V}$ be an $\alpha$-net for $\crl*{V\in\R^{d\times{}r}\mid{}\nrm{V}_{F}=\sqrt{\tau}}$ with respect to $\ls_2$, with $\alpha=1/(T\cdot\tau)$.\\
Let $\gamma=\sqrt{\log\abs{\V}/T}$.\\
Let $q_{1}=\mathrm{Uniform}(\V)$.
\For{each time $t$:}
\For{each $v\in\V$:}
\State Let $G_{t}^{v}(p) = \En_{\sigma_{t}}\frac{\eta{}\tau^{2}}{2}(1-\alpha)^{-1}\burk_{\ls_{2},2}\prn*{(\sum_{s=1}^{t-1}\lp^{v}_sX_s + pX_t)\cdot{}V_{v}, (\sum_{t=1}^{t-1}\eps_s\lp^v_sX_s + \sigma_tpX_t)\cdot{}V_{v}}$.
\State $f_{t}^{v}=-(G_{t}^{v})'(0)$.
\EndFor
\State Sample $v\sim{}q_{t}$ and play $\yh_t=\textrm{Clip}_{[-1,+1]}(f_t^{v})$.
\State Let $\lsb_t=(\ls(f_{t}^{v}, y_t))_{v\in\V}$.
\State Let $q_{t+1}[v] = \exp\prn*{-\gamma\sum_{s=1}^{t}\lsb_s[v]}/\mathcal{Z}$ for each $v\in\V$.\Comment{$\mc{Z}$ is the normalizing constant.}
\State Draw $\eps_{t}\in\pmo$.
\EndFor
\EndProcedure
\end{algorithmic}
\end{algorithm}

\begin{proof}[Proof of \pref{thm:spectral_regret}]

\begin{align*}
&\En\brk*{\sum_{t=1}^{n}\ls(\yh_{t}, y_t) - \inf_{f\in\F}\sum_{t=1}^{n}\ls\prn*{f(x_t), y_t}}\\
&=
\En\brk*{\sum_{t=1}^{n}\ls(\yh_{t}, y_t) - \inf_{U,V\in\R^{d\times{}r}:\nrm{U}_{F},\nrm{V}_{F}\leq{}\sqrt{\tau}}\sum_{t=1}^{n}\ls\prn*{\tri*{X_t,UV^{\trn}}, y_t}}\\
&=
\sum_{t=1}^{n}\En_{v\sim{}q_{t}}\ls(\mathbf{Clip}_{[1,+1]}(f_t^v), y_t) - \inf_{U,V\in\R^{d\times{}r}:\nrm{U}_{F},\nrm{V}_{F}\leq{}\sqrt{\tau}}\sum_{t=1}^{n}\ls\prn*{\tri*{X_t,UV^{\trn}}, y_t}
\intertext{Since $\ls$ is well-behaved, playing the clipping $f_{t}^{v}$ only reduces the learner's loss.}
&\leq
\sum_{t=1}^{n}\En_{v\sim{}q_{t}}\ls(f_t^v, y_t) - \inf_{U,V\in\R^{d\times{}r}:\nrm{U}_{F},\nrm{V}_{F}\leq{}\sqrt{\tau}}\sum_{t=1}^{n}\ls\prn*{\tri*{X_t,UV^{\trn}}, y_t}
\intertext{Let $\Reg$ denote the meta-algorithm $q_t$'s regret bound.}
&\leq
\min_{v\in\V}\sum_{t=1}^{n}\ls(f_t^v, y_t) - \inf_{U,V\in\R^{d\times{}r}:\nrm{U}_{F},\nrm{V}_{F}\leq{}\sqrt{\tau}}\sum_{t=1}^{n}\ls\prn*{\tri*{X_t,UV^{\trn}}, y_t} + \Reg\\
&\leq
\min_{v\in\V}\brk*{\sum_{t=1}^{n}\ls'(f_t^v, y_t)f_t^v - \inf_{U,V\in\R^{d\times{}r}:\nrm{U}_{F},\nrm{V}_{F}\leq{}\sqrt{\tau}}\sum_{t=1}^{n}\ls'(f_t^v, y_t)\tri*{X_t,UV^{\trn}}
} + \Reg\\
\intertext{Using the $\alpha$-net property of $\mc{V}$ and that the loss is $1$-Lipschitz:}
&\leq
\min_{v\in\V}\brk*{\sum_{t=1}^{n}\ls'(f_t^v, y_t)f_t^v - \inf_{v\in\V}\inf_{U\in\R^{d\times{}r}:\nrm{U}_{F}\leq{}\sqrt{\tau}}\sum_{t=1}^{n}\ls'(f_t^v, y_t)\tri*{X_t,UV_{v}^{\trn}}
} + \Reg + \tau{}T\alpha{}
\intertext{Since $\alpha=O(1/\tau{}T)$:}
&\leq
\min_{v\in\V}\brk*{\sum_{t=1}^{n}\ls'(f_t^v, y_t)f_t^v - \inf_{v\in\V}\inf_{U\in\R^{d\times{}r}:\nrm{U}_{F}\leq{}\sqrt{\tau}}\sum_{t=1}^{n}\ls'(f_t^v, y_t)\tri*{X_t,UV_{v}^{\trn}}
} + \Reg + 1\\
&\leq
\sup_{v\in\V}\brk*{\sum_{t=1}^{n}\ls'(f_t^v, y_t)f_t^v - \inf_{U\in\R^{d\times{}r}:\nrm{U}_{F}\leq{}\sqrt{\tau}}\sum_{t=1}^{n}\ls'(f_t^v, y_t)\tri*{X_t,UV^{\trn}_{v}}
} + \Reg + 1
\intertext{Using the sub-algorithm's regret-type bound (\pref{lem:spectral_subalg}):}
&\leq
\sup_{v\in\V}\brk*{\frac{\eta}{2}(1-\alpha)^{-1}\tau\En_{\eps}\nrm*{\sum_{t=1}^{n}\eps_{t}\ls'(f_t^v, y_t)X_tV}_{F}^{2} - \frac{\eta^{-1}}{2}
} + \Reg + 1
\intertext{By contraction:}
&\leq
\sup_{v\in\V}\brk*{\frac{\eta}{2}(1-\alpha)^{-1}\tau\En_{\eps}\nrm*{\sum_{t=1}^{n}\eps_{t}X_tV}_{F}^{2} - \frac{\eta^{-1}}{2}
} + \Reg + 1\\
\intertext{Using the definition of $\V$:}
&\leq
\frac{\eta}{2}(1-\alpha)^{-1}\tau^2\En_{\eps}\nrm*{\sum_{t=1}^{n}\eps_{t}X_t}_{\sigma}^{2} - \frac{\eta^{-1}}{2} + \Reg + 1.\\
&\leq
\frac{\eta}{2}\tau^2\En_{\eps}\nrm*{\sum_{t=1}^{n}\eps_{t}X_t}_{\sigma}^{2} - \frac{\eta^{-1}}{2} + \Reg + 1.
\end{align*}
Finally, observe that $q_{t}$ is generated with the standard multiplicative weights update strategy (e.g. \cite{hazan2016introduction}. Since each $f_t^{v}$ is clipped, the range of the losses seen by the algorithm are bounded by $1$. This implies 
\[
\Reg \leq{} O(\sqrt{n\log\abs{\V}}).
\]
We can find an $\alpha$-net for $\crl*{V\in\R^{d\times{}r}\mid{}\nrm{V}_{F}=\sqrt{\tau}}$ of size $O(\prn*{C\tau/\alpha}^{rd}) = O(\prn*{\tau^{2}n}^{rd})$, so we have
\[
\Reg \leq{} O(\sqrt{nrd\log\prn{\tau{}n}}).
\]

\end{proof}

\begin{lemma}
\label{lem:spectral_subalg}
Let $f_t^{v}$ be defined as in \pref{alg:zigzag_spectral} for some $v\in\V$. Then $f_t^v$ enjoys the regret-like bound
\begin{equation}
\sum_{t=1}^{n}\ls'(f_t^v, y_t)f_t^v + (1-\alpha)^{-1}\sqrt{\tau}\nrm*{\sum_{t=1}^{n}\ls'(f_t^v, y_t)X_tV}_{F}
\leq{} \frac{\eta}{2}(1-\alpha)^{-1}\tau\En_{\eps}\nrm*{\sum_{t=1}^{n}\eps_{t}\ls'(f_t^v, y_t)X_tV}_{F}^{2} + \frac{\eta^{-1}}{2}.
\end{equation}

\end{lemma}
\begin{proof}[Proof of \pref{lem:spectral_subalg}]
\begin{align*}
&\sum_{t=1}^{n}\ls'(f_t^v, y_t)f_t^v + (1-\alpha)^{-1}\sqrt{\tau}\nrm*{\sum_{t=1}^{n}\ls'(f_t^v, y_t)X_tV}_{F}
- \frac{\eta}{2}(1-\alpha)^{-1}\tau\En_{\eps}\nrm*{\sum_{t=1}^{n}\eps_{t}\ls'(f_t^v, y_t)X_tV}_{F}^{2} - \frac{\eta^{-1}}{2}
\intertext{Using the AM-GM inequality:}
&\leq{}
\sum_{t=1}^{n}\ls'(f_t^v, y_t)f_t^v + \frac{\eta}{2}(1-\alpha)^{-1}\tau\nrm*{\sum_{t=1}^{n}\ls'(f_t^v, y_t)X_tV}_{F}^{2}
- \frac{\eta}{2}(1-\alpha)^{-1}\tau\En_{\eps}\nrm*{\sum_{t=1}^{n}\eps_{t}\ls'(f_t^v, y_t)X_tV}_{F}^{2}
\intertext{Using that $\burk_{2}^{\ls_2}$ is Burkholder:}
&\leq{}
\sum_{t=1}^{n}\ls'(f_t^v, y_t)f_t^v + \frac{\eta}{2}(1-\alpha)^{-1}\tau\En_{\eps}\burk_{2}^{\ls_{2}}\prn*{\sum_{t=1}^{n}\ls'(f_t^v, y_t)X_tV, \sum_{t=1}^{n}\eps_{t}\ls'(f_t^v, y_t)X_tV}
\intertext{Repeating the same step-by-step admissibility proof as in \pref{alg:zigzag_supervised}:}
&\leq{} 0.
\end{align*}

\end{proof}

\subsection{Proofs from \pref{sec:martingales2}}

\begin{proof}[Proof of \pref{thm:NScond}]
We shall first show that $\ref{thm:NScond:2}$ implies $\ref{thm:NScond:1}$, specifically for constant $B = 2C$. We can write down the minimax value for the proposed regret bound and check if it indeed is achievable. To this end, note that 
\begin{align*}
\V&=\dtri*{\sup_{x_t}\inf_{\yh_t}\sup_{y_t\in\brk{-1,+1}}}_{t=1}^{n}\brk*{\sum_{t=1}^{n}\ls(\yh_t, y_t) - \inf_{f\in\F}\sum_{t=1}^{n}\ls(f(x_t), y_t) - 2 C\En_{\eps} \sup_{f \in \F}\sum_{t=1}^{n}\eps_t f(x_t)}\\
&=\dtri*{\sup_{x_t}\sup_{p_t\in\Delta\brk{-1,+1}} \inf_{\yh_t} \En_{y_t \sim p_t}}_{t=1}^{n}\sup_{f \in\F}\brk*{ \sum_{t=1}^{n}(\ls(\yh_t, y_t) - \ls(f(x_t), y_t)) - 2C\En_{\eps} \sup_{f \in \F}\sum_{t=1}^{n}\eps_t f(x_t)}\\
&\le \dtri*{\sup_{x_t} \sup_{p_t\in\Delta\brk{-1,+1}} \inf_{\yh_t}\En_{y_t \sim p_t}}_{t=1}^{n}\sup_{f \in\F}\brk*{ \sum_{t=1}^{n}\ls'(\yh_t, y_t)(\yh_t - f(x_t)) - 2C\En_{\eps} \sup_{f \in \F}\sum_{t=1}^{n}\eps_t f(x_t)}\\
\intertext{setting $\yh^*_t$ to be minimizer of $\En \ell(\hat{y}_t,y_t)$, we have}
&\le \dtri*{\sup_{x_t} \sup_{p_t\in\Delta\brk{-1,+1}} \inf_{\yh_t}\En_{y_t \sim p_t}}_{t=1}^{n}\sup_{f \in\F}\brk*{ \sum_{t=1}^{n}\ls'(\yh^*_t, y_t)(\yh^*_t - f(x_t)) - 2C\En_{\eps} \sup_{f \in \F}\sum_{t=1}^{n}\eps_t f(x_t)}\\
&= \dtri*{\sup_{x_t} \sup_{p_t\in\Delta\brk{-1,+1}} \inf_{\yh_t}\En_{y_t \sim p_t}}_{t=1}^{n}\sup_{f \in\F}\brk*{ \sum_{t=1}^{n}- \ls'(\yh^*_t, y_t) f(x_t) - 2C\En_{\eps} \sup_{f \in \F}\sum_{t=1}^{n}\eps_t f(x_t)}\\
&= \dtri*{\sup_{x_t} \sup_{p_t\in\Delta\brk{-1,+1}} \En_{y_t \sim p_t}}_{t=1}^{n}\sup_{f \in\F}\brk*{ \sum_{t=1}^{n} (\En_{y'_t \sim p_t}\ls'(\yh^*_t, y'_t) - \ls'(\yh^*_t, y_t)) f(x_t) - 2C\En_{\eps} \sup_{f \in \F}\sum_{t=1}^{n}\eps_t f(x_t)}\\
&\le \dtri*{\sup_{x_t} \sup_{p_t\in\Delta\brk{-1,+1}} \En_{y_t, y'_t \sim p_t}}_{t=1}^{n}\sup_{f \in\F}\brk*{ \sum_{t=1}^{n} (\ls'(\yh^*_t, y'_t) - \ls'(\yh^*_t, y_t)) f(x_t) - 2C\En_{\eps} \sup_{f \in \F}\sum_{t=1}^{n}\eps_t f(x_t)}\\
&= \dtri*{\sup_{x_t} \sup_{p_t\in\Delta\brk{-1,+1}} \En_{y_t, y'_t \sim p_t} \En_{\epsilon'_t}}_{t=1}^{n}\sup_{f \in\F}\brk*{ \sum_{t=1}^{n} \epsilon'_t (\ls'(\yh^*_t, y'_t) - \ls'(\yh^*_t, y_t)) f(x_t) - 2C\En_{\eps} \sup_{f \in \F}\sum_{t=1}^{n}\eps_t f(x_t)}\\
&\le \dtri*{\sup_{x_t}  \En_{\epsilon'_t}}_{t=1}^{n}\sup_{f \in\F}\brk*{ \sum_{t=1}^{n} 2 \epsilon'_t f(x_t) - 2C\En_{\eps} \sup_{f \in \F}\sum_{t=1}^{n}\eps_t f(x_t)}\\
&= \sup_{\x} \En_{\epsilon'} \sup_{f \in\F}\brk*{ \sum_{t=1}^{n} 2 \epsilon'_t f(\x_t(\epsilon'_{1:t-1})) - 2 C\En_{\eps} \sup_{f \in \F}\sum_{t=1}^{n}\eps_t \x_t(\epsilon'_{1:t-1})}.\\
\end{align*}
However by $\ref{thm:NScond:2}$, we have that the above is bounded by $0$ and so we can conclude that the minimax strategy does attain the regret bound proposed in \ref{thm:NScond:1}. 

Now to prove that \ref{thm:NScond:1} implies \ref{thm:NScond:2} (with constant $B$), notice that  we have an algorithm that guarantees regret bound:
\[
\sum_{t=1}^{n}\ls(\hat{y}_t, y_t) - \inf_{f\in\F}\sum_{t=1}^{n}\ls(f(x_t), y_t) \leq{} B\En_{\eps}\sup_{f \in \F} \sum_{t=1}^{n}\eps_tf(x_t)
\]
Assume now that the adversary at time first provides input instance $\x_t(\epsilon_{1:t-1})$ where $\x$ is any arbitrary $\X$ valued binary tree. Also assume that $y_t$ is picked to be $\epsilon_t$ a draw of a coin flip. In this case, we have from the regret bound that
\[
\sum_{t=1}^{n}\ls(\hat{y}_t, \epsilon_t) - \inf_{f\in\F}\sum_{t=1}^{n}\ls(f(\x_t(\eps_{1:t-1})), \epsilon_t) \leq{} B\En_{\eps'}\sup_{f \in \F} \sum_{t=1}^{n}\eps'_tf(\x_t(\eps_{1:t-1}))
\]
Taking expectation we find that,
\[
\En_{\epsilon}\left[\sum_{t=1}^{n}\ls(\hat{y}_t, \epsilon_t) - \inf_{f\in\F}\sum_{t=1}^{n}\ls(f(\x_t(\eps_{1:t-1})), \epsilon_t)\right] \leq{} B \En_{\eps, \eps'}\sup_{f \in \F} \sum_{t=1}^{n}\eps'_tf(\x_t(\eps_{1:t-1}))
\]
Now notice that irrespective of what $\hat{y}_t$ the algorithm picks, $\En_{\epsilon_t}\ls(\hat{y}_t, \epsilon_t) = 1$. Hence,
\[
\En_{\epsilon}\left[\sup_{f\in\F}\sum_{t=1}^{n}(1 -\ls(f(\x_t(\eps_{1:t-1})), \epsilon_t))\right] \leq{} B \En_{\eps, \eps'}\sup_{f \in \F} \sum_{t=1}^{n}\eps'_tf(\x_t(\eps_{1:t-1}))
\]
However note that when $y \in \{\pm1\}$ and $a \in  [-1,1]$, we have that $\ell(a,y) = |a - y| = 1 - ay$. Hence from above we conclude that,
\[
\En_{\epsilon}\left[\sup_{f\in\F}\sum_{t=1}^{n}\epsilon_t f(\x_t(\eps_{1:t-1}))\right] \leq{} B \En_{\eps, \eps'}\sup_{f \in \F} \sum_{t=1}^{n}\eps'_tf(\x_t(\eps_{1:t-1}))
\]
Since the above is true for any choice of $\x$ by adversary, we have shown that \ref{thm:NScond:1} implies \ref{thm:NScond:2} with constant $B$.
\end{proof}

\begin{proof}[Proof of \pref{ex:umd_general_kernel}]
Let $\x$ be some $\X$-valued tree. Observe that by the reproducing property,
\[
\En_{\sigma}\sup_{f\in\F}\sum_{t=1}^{n}\sigma_{t}f(\x_t(\sigma)) = \En_{\sigma}\nrm*{\sum_{t=1}^{n}\sigma_{t}K(\cdot,\x_t(\sigma))}_{\H},
\]
and likewise $\En_{\sigma,\eps}\sup_{f\in\F}\sum_{t=1}^{n}\eps_{t}f(\x_t(\sigma)) = \En_{\sigma,\eps}\nrm*{\sum_{t=1}^{n}\eps_{t}K(\cdot,\x_t(\sigma))}_{\H}$.

Since $\H$ is a Hilbert space the deterministic UMD property for power $2$ is trivial. For any fixed sequence $\eps\in\pmo^n$,
\[
\En_{\sigma}\nrm*{\sum_{t=1}^{n}\sigma_{t}K(\cdot,\x_t(\sigma))}_{\H}^{2} = \En_{\sigma}\nrm*{\sum_{t=1}^{n}\eps_{t}\sigma_tK(\cdot,\x_t(\sigma))}_{\H}^2.
\]
By Corollary \ref{corr:umd_L1}, this implies there is some $C$ such that
\[
\En_{\sigma}\sup_{\tau\leq{}n}\nrm*{\sum_{t=1}^{\tau}\sigma_{t}K(\cdot,\x_t(\sigma))}_{\H} = C\En_{\sigma}\sup_{\tau\leq{}n}\nrm*{\sum_{t=1}^{\tau}\eps_{t}\sigma_tK(\cdot,\x_t(\sigma))}_{\H}.
\]
Now suppose $\eps$ is drawn uniformly at random. For a fixed draw of $\sigma$, \pref{corr:concentration_maximal_exp} implies that the RHS enjoys the bound
\begin{align*}
\En_{\eps}\sup_{\tau\leq{}n}\nrm*{\sum_{t=1}^{\tau}\eps_{t}K(\cdot,\x_t(\sigma))}_{\H} 
&\leq{} 2\En_{\eps}\nrm*{\sum_{t=1}^{n}\eps_tK(\cdot,\x_t(\sigma))}_{\H} + 5\max_{t\in\brk{n}}\nrm{K(\cdot,\x_t(\sigma))}_{\H}\log(n)\\
&\leq{} 2\En_{\eps}\nrm*{\sum_{t=1}^{n}\eps_tK(\cdot, \x_t(\sigma))}_{\H} + 5B\log(n).
\end{align*}
\end{proof}
\subsubsection{Polynomials}
Suppose we receive data $x_1,\ldots,x_n\in\R^d$ and want to compete with a class $\F$ of homogeneous polynomials of degree $k$. Any homogeneous degree $k$ polynomial $f$ may be represented via a coefficient tensor $M$ in $(\R^d)^{\tens{}k}$ via
\[
f(x) = \tri*{M, x^{\tens{}k}}.
\]
We may take $M$ to be symmetric, so that $M_{1,\ldots,k}=M_{\pi(1),\ldots,\pi(k)}$ for any permutation.
We may thus work with a class $\M\subseteq{}(\R^d)^{\tens{}k}$ of symmetric tensors, then take $\F=\crl*{x\mapsto\tri*{M,x^{\tens{}k}}\mid{}M\in\M}$. Our task is then to decide which norm to place on $\M$. Following, e.g., \cite{adamczak2015concentration,wang2016operator}, we define a class of general tensor norms. Let $\mc{J}=\crl*{J_1,\ldots,J_{N}}$ be a partition of $\brk{k}$. For some $\alpha\in\brk{d}^{k}$ and $J\subseteq{}\brk{k}$, let $\alpha_{J}=(\alpha_{i})_{i\in{}J}$. We then define
\begin{equation}
\nrm*{M}_{\mc{J}}=\sup\crl*{
\sum_{\alpha\in\brk{d}^{k}}M_{\alpha}\prod_{l=1}^{N}x_{\alpha_{J_{l}}}^{l}\mid{}\nrm*{x^{l}}_{2}\leq{}1\;\forall{}l\in\brk{N}
},
\end{equation}
where $x^{l}\in(\R^{d})^{\tens{}\abs{J_l}}$. Under this notation we have $\nrm{M}_{\crl{1},\crl{2}}$ as the spectral norm and $\nrm{M}_{\crl{1,2}}$ as the Frobenius norm when $k=2$ and $M$ is a matrix. In general, $\nrm{M}_{\crl{1},\crl{2},\ldots,\crl{k}}$ is called the \emph{injective tensor norm}.
\begin{example}[Homogeneous Polynomials]
\label{ex:polynomials}
Consider homogeneous polynomials of degree $2k$, and let $\M$ be the unit ball of the norm $(\nrm{\cdot}_{\crl{1,\ldots,k},\crl{k+1,\ldots,2k}})_{\star}$ in $(\R^{d})^{\tens{}2k}$. Then there exist $K_1, K_2$ such that
\[
\En_{\sigma}\sup_{f\in\F}\sum_{t=1}^{n}\sigma_{t}f(\phi_{t}(\sigma_{1:t-1}))\leq{} K_1k^{2}\log^{2}(d)\En_{\sigma,\eps}\sup_{f\in\F}\sum_{t=1}^{n}\eps_{t}f(\phi_{t}(\sigma_{1:t-1})) + K_{2}k^{2}\log^{2}(d)\log(n).
\]

\end{example}
\begin{proof}[Proof of of Example \ref{ex:polynomials}]
Fix an $\X$-valued tree $\x$. Then we have
\[
\En_{\sigma}\sup_{f\in\F}\sum_{t=1}^{n}\sigma_{t}f(\x_t(\sigma)) = \En_{\sigma}\sup_{M\in\M}\sum_{t=1}^{n}\sigma_{t}\tri*{M,\x_t(\sigma)^{\tens{}2k}} = \En_{\sigma}\nrm*{\sum_{t=1}^{n}\sigma_{t}\x_t(\sigma)^{\tens{}2k}}_{\crl{1,\ldots,k},\crl{k+1,\ldots,2k}}
\]
For some tensor $T\in(\R^{d})^{\tens{}2k}$, we can define its flattening $\overline{T}$ into a $\R^{d^k\times{}d^k}$ matrix and verify that in fact
\[
\nrm{T}_{\crl{1,\ldots,k},\crl{k+1,\ldots,2k}} = \max_{u,v\in\R^{d^k}\mid{}\nrm{u}_2,\nrm{v}_2\leq{}1}\sum_{\alpha\in\brk{d}^k,\beta\in\brk{d}^k}T_{\alpha,\beta}u_{\alpha}v_{\beta} = \tri*{u,\overline{T}v}=\nrm{\overline{T}}_{\sigma},
\]
so in fact this is the spectral norm of the flattened matrix. Let $\mathbf{X}_{t}\in\R^{d^k\times{}d^k}$ be the flattening of $(\x_t)^{\tens{}2k}$.
Then 
\[
\En_{\sigma}\sup_{f\in\F}\sum_{t=1}^{n}\sigma_{t}f(\x_t(\sigma)) = \En_{\sigma}\nrm*{\sum_{t=1}^{n}\sigma_t\mathbf{X}_{t}(\sigma)}_{\sigma},
\]
so we can prove the desired inequality by applying the UMD inequality for the spectral norm. Recall from \pref{thm:umd_constants} that the UMD inequality for the spectral norm has a constant of order $\log^{2}(\mathrm{dim})$, which for this application translates into a constant of order $O(k^{2}\log^{2}(d))$. We finally apply \pref{corr:concentration_maximal_exp} as in \pref{ex:umd_general_kernel} to get the result.
\end{proof}

\subsubsection{Low-rank experts}
In this section we prove \pref{thm:low_rank}. The proof relies on the following key lemma, which is proven using the one-sided UMD property for scalars.
\begin{lemma}
\label{lem:low_rank_regret}
There exists a strategy $(\yh_t)$ for the experts setting that guarantees
\begin{equation}
\label{eq:low_rank_rad}
\sum_{t=1}^{n}\ls(\yh_t, y_t) - \inf_{f\in\F}\sum_{t=1}^{n}\ls(f(x_t), y_t) \leq{} O\prn*{\En_{\eps}\nrm*{\sum_{t=1}^{n}\eps_{t}x_{t}}_{\infty}^{\log{}d}}^{1/\log{}d}.
\end{equation}
\end{lemma}

With this lemma, we need one more fact to prove \pref{thm:low_rank}, which is a corollary of John's theorem about the volume of a minimum-volume enclosing ellipsoid.
\begin{lemma}[\cite{hazan2016online}, Lemma 12]
\label{lem:john_ellipsoid}
 Let $K$ be a symmetric convex set in $\R^{d}$. There exists a positive semidefinite matrix $\Xi$ such that for all $x\in{}K$,
\begin{equation}
\tri*{x, \Xi{}x} \leq{} \sup_{f\in{}K^{\star}}\abs{\tri{f,x}}^{2}\leq{}d\cdot{}\tri*{x, \Xi{}x}.
\end{equation}
\end{lemma}
Applying \pref{lem:john_ellipsoid} to the intersection of the $\ls_{\infty}$ ball and $\mathrm{span}(\xr[n])$ gives a Euclidean approximation to the $\ls_{\infty}$ norm in terms of the rank of $X_{1:n}$.
\begin{corollary}
\label{corr:infinity_norm_low_rank}

There exists some positive semidefinite $\Xi\in\R^{d\times{}d}$ such that for all $S\in\mathrm{span}(\xr[n])$,
\begin{equation}
\tri*{S, \Xi{}S} \leq{} \nrm*{S}_{\infty}^{2} \leq{}\rank(X_{1:n})\cdot{}\tri*{S, \Xi{}S}.
\end{equation}
\end{corollary}
We can now proceed to the proof of the main theorem.
\begin{proof}[Proof of \pref{thm:low_rank}]
By \pref{lem:low_rank_regret}, there exists a strategy whose regret is bounded by
\[
O\prn*{\En_{\eps}\nrm*{\sum_{t=1}^{n}\eps_{t}x_{t}}_{\infty}^{\log{}d}}^{1/\log{}d}.
\]

We now complete the upper bound using concentration. Let $Z=\nrm*{\sum_{t=1}^{n}\eps_{t}x_{t}}_{\infty}$. Then we can write $\prn*{\En_{\eps}\nrm*{\sum_{t=1}^{n}\eps_{t}x_{t}}_{\infty}^{\log{}d}}^{1/\log{}d}$ as $\prn*{\En{}Z^{\log{}d}}^{1/\log{}d}$, where the expectation is over the sequence $\eps$. We will upper bound this quantity in terms of the rank. First observe that by \pref{corr:infinity_norm_low_rank}, there exists a PSD matrix $\Xi$ such that 
\[
\En_{\eps}\nrm*{\sum_{t=1}^{n}\eps_{t}x_{t}}_{\infty} \leq {} \sqrt{\rank(X_{1:n})}\En_{\eps}\nrm*{\sum_{t=1}^{n}\eps_{t}x_{t}}_{\Xi},
\]
where $\nrm{x}_{\Xi}=\tri*{x,\Xi{}x}$.

Observe that since $\nrm{\cdot}_{\Xi}$ is Euclidean,
\[
\En_{\eps}\nrm*{\sum_{t=1}^{n}\eps_{t}x_{t}}_{\Xi} \sqrt{\En_{\eps}\nrm*{\sum_{t=1}^{n}\eps_{t}x_{t}}_{\Xi}^{2}} = \sqrt{\sum_{t=1}^{n}\nrm*{x_t}^{2}_{\Xi}}\leq{}\sqrt{\sum_{t=1}^{n}\nrm*{x_t}^{2}_{\infty}}\leq{}\sqrt{n},
\]
where the second-to-last inequality uses \pref{corr:infinity_norm_low_rank}. This establishes that
\[
\En{}Z\leq{} \sqrt{\rank(X_{1:n})n}.
\]
Now, since $\nrm{x_t}_{\infty}\leq{}1$, \pref{lem:concentration_rademacher} implies that with probability at least $1-\delta$ over the draw of $\eps$,
\[
Z \leq{} O\prn*{ \En{}Z + \log(1/\delta)}.
\]
By the law of total expectation, this establishes that for all $\delta>0$,
\[
\prn*{\En_{\eps}\nrm*{\sum_{t=1}^{n}\eps_{t}x_{t}}_{\infty}^{\log{}d}}^{1/\log{}d} 
\leq{} O\prn*{\prn*{(\sqrt{\rank(X_{1:n})n} + \log(1/\delta))^{\log{}d} + n^{\log{}d}\delta}^{1/\log{}d}
}.
\]
Taking $\delta=n^{-\log{}d}$, the above quantity is bounded by
\[
O\prn*{\prn*{(\sqrt{\rank(X_{1:n})n} + \log(n)\log(d))^{\log{}d}}^{1/\log{}d}},
\]
which is further bounded as 
\[
O\prn*{\sqrt{\rank(X_{1:n})n} + \log(n)\log(d)}.
\]

\end{proof}

\begin{proof}[Proof of \pref{thm:low_rank_approx}]
This result is proven from the same starting point as in \pref{thm:low_rank}. Recall from \pref{lem:low_rank_regret} that there is a strategy whose regret is bounded by
\[
O\prn*{\En_{\eps}\nrm*{\sum_{t=1}^{n}\eps_{t}x_{t}}_{\infty}^{\log{}d}}^{1/\log{}d}.
\]
Suppose $\rank_{\gamma}(X_{1:n})=r$. Then there exist matrices $X'_{1:n}\in\R^{d\times{}n}$ and $Z_{1:n}\in\R^{d\times{}n}$ such that
\[
X_{1:n} = X'_{1:n} + Z_{1:n},
\]
with $\rank(X'_{1:n})=r$ and $\nrm{Z}_{\infty}\leq{}\gamma$. Using $x'_{t}$ to denote the $t$th column of $X'_{1:n}$ and $z_{t}$ to denote the $t$th column of $Z_{1:n}$, triangle inequality implies
\begin{align*}
\prn*{\En_{\eps}\nrm*{\sum_{t=1}^{n}\eps_{t}x_{t}}_{\infty}^{\log{}d}}^{1/\log{}d} = 
\prn*{\En_{\eps}\nrm*{X_{1:n}\eps}_{\infty}^{\log{}d}}^{1/\log{}d} &\leq{} O\prn*{\prn*{\En_{\eps}\nrm*{X'_{1:n}\eps}_{\infty}^{\log{}d}}^{1/\log{}d} + \prn*{\En_{\eps}\nrm*{Z_{1:n}\eps}_{\infty}^{\log{}d}}^{1/\log{}d}}\\
&= O\prn*{\prn*{\En_{\eps}\nrm*{\sum_{t=1}^{n}\eps_tx'_t}_{\infty}^{\log{}d}}^{1/\log{}d}} + O\prn*{\prn*{\En_{\eps}\nrm*{\sum_{t=1}^{n}\eps_{t}z_t}_{\infty}^{\log{}d}}^{1/\log{}d}}
\end{align*}
Since the loss matrix in the first term has rank $r$, this term can be bounded exactly as in \pref{thm:low_rank}. We now show how to bound the second term. First, observe that since $\nrm{Z_{1:n}}_{\infty}\leq{}\gamma$, the standard estimate on the maximum of $d$ subgaussian random variables gives
\[
\En_{\eps}\nrm*{\sum_{t=1}^{n}\eps_{t}z_t}_{\infty}\leq{} O(\gamma\sqrt{n\log{}d}).
\]
\pref{lem:concentration_rademacher} implies that with probability at least $1-\delta$ over the draw of $\eps$
\[
\nrm*{\sum_{t=1}^{n}\eps_{t}z_t}_{\infty}\leq{} O(\gamma\sqrt{n\log{}d} + \gamma\log(1/\delta)).
\]
Applying the law of total expectation (and recalling that $\gamma\leq{}1$), this implies that for all $\delta>0$
\[
\prn*{\En_{\eps}\nrm*{\sum_{t=1}^{n}\eps_{t}z_t}_{\infty}^{\log{}d}}^{1/\log{}d} \leq{} O\prn*{
\prn*{
(\gamma\sqrt{n\log{}d} + \gamma\log(1/\delta))^{\log{}d} + n^{\log{}d}\delta
}^{1/\log{}d}
}
\]
Taking $\delta=n^{-\log{}d}$, the above is finally bounded as
\[
O(\gamma\sqrt{n\log{}d} + \gamma\log{}n\log{}d).
\]

\end{proof}

\begin{proof}[Proof of \pref{thm:experts_max_norm}]
This proof follows the same structure as \pref{thm:low_rank} and \pref{thm:low_rank_approx}. Starting from \pref{lem:low_rank_regret}, we have that there is a strategy whose regret is bounded by
\[
O\prn*{\En_{\eps}\nrm*{\sum_{t=1}^{n}\eps_{t}x_{t}}_{\infty}^{\log{}d}}^{1/\log{}d}.
\]
Observe that $\En_{\eps}\nrm*{\sum_{t=1}^{n}\eps_{t}x_{t}}_{\infty} = \En_{\eps}\nrm*{X_{1:n}\eps}_{\infty}$. From the definition of the max norm, there exist $U\in\R^{d\times{}d}$, $V\in\R^{n\times{}d}$ such that $X_{1:n}=UV^{\trn}$ and $\nrm{U}_{\infty,2}\nrm{V}_{\infty,2}=\nrm{X_{1:n}}_{\mathrm{max}}$. With this observation, we have
\[
\En_{\eps}\nrm*{X_{1:n}\eps}_{\infty} = \En_{\eps}\nrm*{UV^{\trn}\eps}_{\infty} = \En_{\eps}\nrm*{U\sum_{t=1}^{n}v_t\eps_t}_{\infty},
\]
where $v_{t}$ denotes the $t$th row of $V$. Now, observe that
\[
\nrm{U}_{\infty,2}= \max_{i\in\brk{d}}\nrm{u_{i}}_{2} = \max_{i\in\brk{d}}\max_{x:\nrm{x}_2\leq{}1}\tri*{u_{i},x} = \max_{x:\nrm{x}_2\leq{}1}\nrm*{Ux}_{\infty} = \nrm{U}_{2\to\infty},
\]
so $\nrm{\cdot}_{\infty,2}$ is actually the $2\to\infty$ operator norm. This implies that
\[
\En_{\eps}\nrm*{U\sum_{t=1}^{n}v_t\eps_t}_{\infty} \leq{} \nrm{U}_{\infty,2}\cdot{}\En_{\eps}\nrm*{\sum_{t=1}^{n}v_t\eps_t}_{2}.
\]
Proceeding with the standard Euclidean calculation for Rademacher complexity (e.g. \cite{kakade09complexity}), and using that $\nrm{v_t}_{2}\leq{}\nrm{V}_{\infty,2}\;\forall{}t$, the above implies that 
\[
\En_{\eps}\nrm*{\sum_{t=1}^{n}\eps_{t}x_{t}}_{\infty}\leq{}\nrm{U}_{\infty,2}\nrm{V}_{\infty,2}\sqrt{n}=\nrm{X_{1:n}}_{\mathrm{max}}\sqrt{n}.
\]

Once again, we appeal to \pref{lem:concentration_rademacher}, which implies that with probability at least $1-\delta$ over the draw of $\eps$,
\[
\nrm*{\sum_{t=1}^{n}\eps_{t}x_t}_{\infty}\leq{} O(\nrm{X}_{\mathrm{max}}\cdot\sqrt{n} + \log(1/\delta)).
\]
Again using the law of total expectation, this implies that for all $\delta>0$
\[
\prn*{\En_{\eps}\nrm*{\sum_{t=1}^{n}\eps_{t}x_t}_{\infty}^{\log{}d}}^{1/\log{}d} \leq{} O\prn*{
\prn*{
(\nrm{X}_{\mathrm{max}}\cdot\sqrt{n} + \log(1/\delta))^{\log{}d} + n^{\log{}d}\delta.
}^{1/\log{}d}
}
\]
Taking $\delta=n^{-\log{}d}$, we have
\[
O(\nrm{X}_{\mathrm{max}}\cdot\sqrt{n} + \log{}n\log{}d).
\]

\end{proof}

We now focus on proving \pref{lem:low_rank_regret}. The structure of this proof will follow that of \pref{thm:NScond}, which gives an upper bound on regret in terms of $\RadH$ whenever the one-sided UMD inequality holds. To achieve the desired bound in this framework, we will need the following corollary of Hitczenko's decoupling inequality \pref{thm:hitczenko_scalar}.

\begin{corollary}[One-sided UMD inequality for $\ls_{p}$ norms]
\label{corr:lp_one_sided}
There exists some constant $K$ such that for all $p\geq{}1$,
\begin{equation}
\En_{\eps}\nrm*{\sum_{t=1}^{n}\eps_{t}\x_{t}(\eps)}_{p}^{p} \leq{} K^{p}\En_{\eps,\eps'}\nrm*{\sum_{t=1}^{n}\eps'_t\eps_{t}\x_{t}(\eps)}_{p}^{p},
\end{equation}
where $\x$ is any $\X$-valued tree.
\end{corollary}
\begin{proof}[Proof of \pref{corr:lp_one_sided}]
Simply apply \pref{thm:hitczenko_scalar} coordinate-wise.
\end{proof}
	With this inequality, we proceed to prove \pref{lem:low_rank_regret}.
\begin{proof}[Proof of \pref{lem:low_rank_regret}]
Let $p=\log{}d$. Recall that we have defined
\[
\Psi_{\eta,p}(x) = \frac{1}{p}\prn*{\eta{}x + \frac{1}{p'-1}\eta^{1-p'}}.
\]
We first will prove that there is a strategy $(\yh_{t})$ that achieves 
\[
\sum_{t=1}^{n}\ls(\yh_t, y_t) - \inf_{f\in\F}\sum_{t=1}^{n}\ls(f(x_t), y_t) \leq{} \Psi_{\eta,p}\prn*{C\En_{\eps}\nrm*{\sum_{t=1}^{n}\eps_tx_{t}}_{\infty}^{p}}
\]
for some $C>0$. This portion of the proof will closely follow \pref{thm:NScond}.
Fix $C$ to be decided later and define
\begin{align*}
\V&=\dtri*{\sup_{x_t}\inf_{\yh_t}\sup_{y_t\in\brk{-1,+1}}}_{t=1}^{n}\brk*{\sum_{t=1}^{n}\ls(\yh_t, y_t) - \inf_{f\in\F}\sum_{t=1}^{n}\ls(f(x_t), y_t) -  \Psi_{\eta,p}\prn*{C\En_{\eps}\nrm*{\sum_{t=1}^{n}\eps_tx_{t}}_{\infty}^{p}}
}.
\intertext{Observe that the regret bound we desired is achievable if there is a value for $C$ such that $\mc{V}\leq{}0$.}
\mc{V}&=\dtri*{\sup_{x_t}\sup_{p_t\in\Delta\brk{-1,+1}} \inf_{\yh_t} \En_{y_t \sim p_t}}_{t=1}^{n}\sup_{f \in\F}\brk*{ \sum_{t=1}^{n}(\ls(\yh_t, y_t) - \ls(f(x_t), y_t)) - \Psi_{\eta,p}\prn*{C\En_{\eps}\nrm*{\sum_{t=1}^{n}\eps_tx_{t}}_{\infty}^{p}}
}\\
&\le \dtri*{\sup_{x_t} \sup_{p_t\in\Delta\brk{-1,+1}} \inf_{\yh_t}\En_{y_t \sim p_t}}_{t=1}^{n}\sup_{f \in\F}\brk*{ \sum_{t=1}^{n}\ls'(\yh_t, y_t)(\yh_t - f(x_t)) - \Psi_{\eta,p}\prn*{C\En_{\eps}\nrm*{\sum_{t=1}^{n}\eps_tx_{t}}_{\infty}^{p}}
}\\
\intertext{Setting $\yh^*_t$ to be minimizer of $\En \ell(\hat{y}_t,y_t)$, we have}
&\le \dtri*{\sup_{x_t} \sup_{p_t\in\Delta\brk{-1,+1}} \inf_{\yh_t}\En_{y_t \sim p_t}}_{t=1}^{n}\sup_{f \in\F}\brk*{ \sum_{t=1}^{n}\ls'(\yh^*_t, y_t)(\yh^*_t - f(x_t)) - \Psi_{\eta,p}\prn*{C\En_{\eps}\nrm*{\sum_{t=1}^{n}\eps_tx_{t}}_{\infty}^{p}}
}\\
&= \dtri*{\sup_{x_t} \sup_{p_t\in\Delta\brk{-1,+1}} \inf_{\yh_t}\En_{y_t \sim p_t}}_{t=1}^{n}\sup_{f \in\F}\brk*{ \sum_{t=1}^{n}- \ls'(\yh^*_t, y_t) f(x_t) - \Psi_{\eta,p}\prn*{C\En_{\eps}\nrm*{\sum_{t=1}^{n}\eps_tx_{t}}_{\infty}^{p}}
}\\
&= \dtri*{\sup_{x_t} \sup_{p_t\in\Delta\brk{-1,+1}} \En_{y_t \sim p_t}}_{t=1}^{n}\sup_{f \in\F}\brk*{ \sum_{t=1}^{n} (\En_{y'_t \sim p_t}\ls'(\yh^*_t, y'_t) - \ls'(\yh^*_t, y_t)) f(x_t) - \Psi_{\eta,p}\prn*{C\En_{\eps}\nrm*{\sum_{t=1}^{n}\eps_tx_{t}}_{\infty}^{p}}
}\\
&\le \dtri*{\sup_{x_t} \sup_{p_t\in\Delta\brk{-1,+1}} \En_{y_t, y'_t \sim p_t}}_{t=1}^{n}\sup_{f \in\F}\brk*{ \sum_{t=1}^{n} (\ls'(\yh^*_t, y'_t) - \ls'(\yh^*_t, y_t)) f(x_t) - \Psi_{\eta,p}\prn*{C\En_{\eps}\nrm*{\sum_{t=1}^{n}\eps_tx_{t}}_{\infty}^{p}}
}\\
&= \dtri*{\sup_{x_t} \sup_{p_t\in\Delta\brk{-1,+1}} \En_{y_t, y'_t \sim p_t} \En_{\epsilon'_t}}_{t=1}^{n}\sup_{f \in\F}\brk*{ \sum_{t=1}^{n} \epsilon'_t (\ls'(\yh^*_t, y'_t) - \ls'(\yh^*_t, y_t)) f(x_t) - \Psi_{\eta,p}\prn*{C\En_{\eps}\nrm*{\sum_{t=1}^{n}\eps_tx_{t}}_{\infty}^{p}}
}\\
&\le \dtri*{\sup_{x_t}  \En_{\epsilon'_t}}_{t=1}^{n}\sup_{f \in\F}\brk*{ \sum_{t=1}^{n} 2 \epsilon'_t f(x_t) - \Psi_{\eta,p}\prn*{C\En_{\eps}\nrm*{\sum_{t=1}^{n}\eps_tx_{t}}_{\infty}^{p}}
}\\
&= \sup_{\x} \En_{\epsilon'}\sup_{f\in\F}\brk*{\sum_{t=1}2\eps'_{t}f(\x_{t}(\eps'))
 - \Psi_{\eta,p}\prn*{C\En_{\eps}\nrm*{\sum_{t=1}^{n}\eps_t\x_{t}(\eps')}_{\infty}^{p}}
}.
\intertext{Using that the simplex $\Delta_{d}$ is a subset of the $\ls_{1}$ ball:} 
&\leq{} \sup_{\x} \En_{\epsilon'}\brk*{2\nrm*{\sum_{t=1}\eps'_{t}\x_{t}(\eps')}_{\infty}
 - \Psi_{\eta,p}\prn*{C\En_{\eps}\nrm*{\sum_{t=1}^{n}\eps_t\x_{t}(\eps')}_{\infty}^{p}}
}.
\intertext{Using \pref{eq:var}, this is upper bounded by}
&= \sup_{\x} \En_{\epsilon'}\frac{\eta}{p}\brk*{2\nrm*{\sum_{t=1}\eps'_{t}\x_{t}(\eps')}_{\infty}^{p}
 - C\En_{\eps}\nrm*{\sum_{t=1}^{n}\eps_t\x_{t}(\eps')}_{\infty}^{p}
}.
\intertext{We can replace the left $\ls_{\infty}$ norm with the $\ls_{p}$ norm as an upper bound: }
&\leq{} \sup_{\x} \En_{\epsilon'}\frac{\eta}{p}\brk*{2\nrm*{\sum_{t=1}\eps'_{t}\x_{t}(\eps')}_{p}^{p}
 - C\En_{\eps}\nrm*{\sum_{t=1}^{n}\eps_t\x_{t}(\eps')}_{\infty}^{p}
}.
\intertext{We now apply the one-sided UMD property for the $\ls_{p}$ norm \pref{corr:lp_one_sided}:}
&\leq{} \sup_{\x} \En_{\eps,\epsilon'}\frac{\eta}{p}\brk*{2K^{p}\nrm*{\sum_{t=1}\eps_{t}\x_{t}(\eps')}_{p}^{p}
 - C\nrm*{\sum_{t=1}^{n}\eps_t\x_{t}(\eps')}_{\infty}^{p}
}.
\end{align*}
Finally, since $p=\log{}d$, there is some constant $A$ such that $\nrm{x}_{p}\leq{}A\nrm{x}_{\infty}$ pointwise. Therefore, if we take $C=O(K)^{p}$, the expression is bounded by zero.

Now, to achieve the final theorem's bound, simply using the doubling trick given in \pref{lem:algorithm_lp1_doubling} on top of the strategy described above. Since $p'=O(1)$, the doubling strategy will guarantee a regret bound of
\[
\sum_{t=1}^{n}\ls(\yh_t, y_t) - \inf_{f\in\F}\sum_{t=1}^{n}\ls(f(x_t), y_t) \leq{} O\prn*{K\prn*{\En_{\eps}\nrm*{\sum_{t=1}^{n}\eps_tx_{t}}_{\infty}^{p}}^{1/p}}.
\]
\end{proof}

\begin{proof}[Proof of \pref{thm:hitczenko_scalar}]
This theorem is an immediate corollary of \citep{hitczenko1994domination}, Theorem 1.1. We will spend a moment to explain this in detail, as that theorem is stated in terms of \emph{tangent sequences}, which are a concept that otherwise does not appear in the present paper.

Given an adapted sequence $(Z_t)_{t\leq{}n}$, we define its \emph{decoupled tangent sequence} $(Z'_{t})_{t\leq{}n}$ as follows: At time $t$, conditioned on $Z_{1:t-1}$, sample $Z'_{t}$ as an i.i.d. copy of $Z_{t}$ under the conditional distribution $\Pr(Z_{t}\mid{}Z_{1},\ldots,Z_{t-1})$. Then $(Z'_{t})_{t\leq{}n}$ satisfies
\begin{enumerate}
\item Identical conditional distribution: $\Pr(Z'_{t}\mid{}Z_{1},\ldots,Z_{t-1}) = \Pr(Z_{t}\mid{}Z_{1},\ldots,Z_{t-1})$
\item Conditional independence: $\Pr(Z'_{1},\ldots,Z'_{n}\mid{}Z_{1},\ldots,Z_{n}) = \prod_{t=1}^{n}\Pr(Z'_{t}\mid{}Z_{1},\ldots,Z_{n})$
\end{enumerate}

With this definition, \citep{hitczenko1994domination}, Theorem 1.1 is stated as follows:

\noindent \emph{There is some universal constant $K$ such that for any adapted sequence $(Z_t)$ and its decoupled tangent sequence $(Z'_t)$, for any $1\leq{}p<\infty$, 
\begin{equation}
\label{eq:scalar_decoupling}
\En\abs*{\sum_{t=1}^{n}Z_t}^{p}\leq{} K^{p} \abs*{\sum_{t=1}^{n}Z'_t}^{p}.
\end{equation}
}

We now show how to conclude \pref{thm:hitczenko_scalar} from this result. Observe that for a Paley-Walsh martingale $(\eps_{t}\x_{t}(\eps_{t:t-1}))_{t=1}^{n}$, its decoupled tangent sequence is given by $(\eps'_{t}\x_{t}(\eps_{t:t-1}))_{t=1}^{n}$, where $\eps'$ is an independent sequence of Rademacher random variables. Furthermore, this sequence is distributed identically to $(\eps'_{t}\eps_{t}\x_{t}(\eps_{t:t-1}))_{t=1}^{n}$. Therefore \pref{thm:hitczenko_scalar} follows from specializing \pref{eq:scalar_decoupling} to Paley-Walsh martingales.
\end{proof}

\subsubsection{Empirical covering number bounds}

\begin{proof}[Proof of \pref{thm:general_cover} and \pref{thm:general_chaining}]
\pref{thm:NScond} proves that when the one-sided UMD-property \pref{eq:UMDp} holds, there exists a strategy whose regret is bounded as
\[
C\En_{\eps}\sup_{f \in \F} \sum_{t=1}^{n}\eps_tf(x_t).
\]
Since this quantity is the statistical Rademacher complexity, we may apply the classical covering number bound \cite[Proposition 12.3]{StatNotes2012}:
\[
\En_{\eps}\sup_{f \in \F} \sum_{t=1}^{n}\eps_tf(x_t)\leq{} O\prn*{\inf_{\alpha>0}\crl*{\alpha{}n
 + \sqrt{\log{}\mc{N}_1(\Delta_d,\alpha,\xr[n])n}}}.	
\]

Likewise, the classical Dudley entropy integral bound \cite[Theorem 12.4]{StatNotes2012} yields:
\[
\En_{\eps}\sup_{f \in \F} \sum_{t=1}^{n}\eps_tf(x_t) \leq{}O\prn*{
\inf_{\alpha>0}\crl*{\alpha\cdot{}n + \int_{\alpha}^{1}\sqrt{\log{}\mc{N}_2(\F,\delta,\xr[n])n}d\delta}}.
\]

\end{proof}

\begin{proof}[Proof of \pref{thm:cover_bound} and \pref{thm:cover_chaining}]
By \pref{lem:low_rank_regret}, there exists a strategy whose regret is bounded by
\[
O\prn*{\En_{\eps}\nrm*{\sum_{t=1}^{n}\eps_{t}x_{t}}_{\infty}^{\log{}d}}^{1/\log{}d}.
\]
Observe that 
\[
\En_{\eps}\nrm*{\sum_{t=1}^{n}\eps_{t}x_{t}}_{\infty} = \En_{\eps}\sup_{f\in\Delta_d}\sum_{t=1}^{n}\eps_{t}f(x_{t}).
\]
We prove the theorem by appealing to the following classical empirical process bounds \cite[Proposition 12.3, Theorem 12.4]{StatNotes2012}. For \pref{thm:cover_bound}:
\[
\En_{\eps}\sup_{f\in\Delta_d}\sum_{t=1}^{n}\eps_{t}f(x_{t}) \leq{} O\prn*{\inf_{\alpha>0}\crl*{\alpha{}n
 + \sqrt{\log{}\mc{N}_1(\Delta_d,\alpha,\xr[n])n}}}.
\]
For \pref{thm:cover_chaining}:
\[
\En_{\eps}\sup_{f\in\Delta_d}\sum_{t=1}^{n}\eps_{t}f(x_{t}) \leq{} O\prn*{\inf_{\alpha>0}\crl*{\alpha{}n + \int_{\alpha}^{1}\sqrt{\log{}\mc{N}_2(\Delta_d,\delta,\xr[n])n}d\delta}}.
\]

To show the final bound, proceed with the concentration argument used in the proof of \pref{thm:low_rank}.

\end{proof}

\section{UMD spaces and martingale inequalities}
\label{app:martingale}

\subsection{Stopping inequalities}

Let $(Z_t)$ be a martingale. For two stopping times $\tau_1,\tau_2$, we define its stopped version as $Z_{t}^{\tau_1:\tau_2}$ via 
\[
dZ_{t}^{\tau_1:\tau_2}=dZ_{t}\ind\crl*{t>\tau_1}\ind\crl*{t\leq{}\tau_{2}}.
\]
\begin{proposition}[\cite{veraar2015analysis}, Proposition 3.1.14]
\label{prop:stopping_time}
For any $p\in[1,\infty)$,
\[
\En\nrm*{Z_{n}^{\tau_1:\tau_2}}^{p}\leq{}2^{p}\En\nrm*{Z_{n}}^{p}.
\]
\end{proposition}

\begin{theorem}[Doob's Maximal Inequality]
For any martingale $(Z_t)_{t\geq{}1}$ taking values in $(\Bspace, \nrm{\cdot})$ and any $p\in(1,\infty]$,
\begin{equation}
\label{eq:doob1}
\En\sup_{\tau\leq{}n}\nrm*{\sum_{t=1}^{\tau}dZ_t}^{p}\leq{} (p')^{p}\En\nrm*{\sum_{t=1}^{n}dZ_t}^{p}.
\end{equation}
Furthermore
\begin{equation}
\label{eq:doob2}
\Pr\prn*{\sup_{\tau\leq{}n}\nrm*{\sum_{t=1}^{\tau}dZ_t}>\lambda}\leq{}\frac{1}{\lambda}\En\nrm*{\sum_{t=1}^{n}dZ_t}\quad\forall{}\lambda>0.
\end{equation}
More generally, \pref{eq:doob1} and \pref{eq:doob2} hold when the sequence $(\nrm*{\sum_{t=1}^{\tau}Z_t})_{\tau\geq{}1}$ is replaced by any non-negative submartingale $(F_{\tau})_{\tau\geq{}1}$.
\end{theorem}
\begin{corollary}
\label{corr:doob_p1}
If $(F_n)$ is a non-negative submartingale and $F_n\leq{}A$ almost surely then for all $\eta>0$,
\[
\En\brk*{\max_{\tau\leq{}n}F_{\tau}}\leq{}(\log{}A+
\log{}\eta)\cdot{}\En\brk*{F_n} + \frac{1}{\eta}.
\]
\end{corollary}
\begin{proof}[Proof of \pref{corr:doob_p1}]
\begin{align*}
\En\brk*{\max_{\tau\leq{}n}F_{\tau}}&=\int_{0}^{\infty}\Pr\prn*{\max_{\tau\leq{}n}F_{\tau}>\lambda}d\lambda\\
&=\int_{0}^{A}\Pr\prn*{\max_{\tau\leq{}n}F_{\tau}>\lambda}d\lambda\\
&\leq{}\int_{1/\eta}^{A}\Pr\prn*{\max_{\tau\leq{}n}F_{\tau}>\lambda}d\lambda + \frac{1}{\eta}\\
&\leq{}\En\brk*{F_n}\int_{\frac{1}{\eta}}^{A}\frac{1}{\lambda}d\lambda + \frac{1}{\eta}\\
&=\prn*{\log{}A+\log\eta}\cdot{}\En\brk*{F_n} + \frac{1}{\eta}.
\end{align*}
\end{proof}

\subsection{UMD inequalities}

\begin{theorem}[\cite{veraar2015analysis}, Theorem 4.2.7]
\label{thm:umd_pq}
Suppose $(\Bspace, \nrm{\cdot})$ is such that the deterministic UMD inequality 
\[
\En\nrm*{\sum_{t=1}^{n}\eps_tdZ_t}^{p}\leq{} \Cconstt_{p}^p\En\nrm*{\sum_{t=1}^{n}dZ_t}^p
\]
holds for $p\in(1,\infty)$. Then the determinstic UMD inequality
\[
\En\nrm*{\sum_{t=1}^{n}\eps_tdZ_t}^{q}\leq{} \Cconstt_{q}^q\En\nrm*{\sum_{t=1}^{n}dZ_t}^q
\]
holds for any $q\in(1,\infty)$, with
\[
\Cconstt_{q}\leq{} 100\prn*{\frac{q}{p} + \frac{q'}{p'}}\Cconstt_{p}.
\]

\end{theorem}

\begin{theorem}[\cite{pisier2011martingales}, Theorem 8.23]
\label{thm:umd_L1}
Suppose that the deterministic UMD inequality
\[
\sup_{n}\En\nrm*{\sum_{t=1}^{n}\eps_tdZ_t}^2\leq{} \Cconstt_{2}^{2}\sup_{n}\En\nrm*{\sum_{t=1}^{n}dZ_t}^2
\]
holds for any sign sequence. Then the $L_{1}$ UMD inequality
\[
\En\sup_{n}\nrm*{\sum_{t=1}^{n}\eps_tdZ_t}\leq{} 54\Cconstt_2\En\sup_{n}\nrm*{\sum_{t=1}^{n}dZ_t}
\]
holds as well.
\end{theorem}
\begin{corollary}
\label{corr:umd_L1}
If deterministic UMD inequality
\[
\En\nrm*{\sum_{t=1}^{n}\eps_tdZ_t}^{2}\leq{} \Cconstt_2^2\En\nrm*{\sum_{t=1}^{n}dZ_t}^2
\]
holds for any sign sequence, then the $L_{1}$ UMD inequality
\[
\En\sup_{n}\nrm*{\sum_{t=1}^{n}\eps_tdZ_t}\leq{} 108\Cconstt_2\En\sup_{n}\nrm*{\sum_{t=1}^{n}dZ_t}
\]
holds as well.
\end{corollary}

\begin{theorem}[\cite{veraar2015analysis}, Proposition 4.2.17]
\label{thm:umd_dual}
If $(\Bspace, \nrm{\cdot})$ is $\UMD_p$ with constant $\Cconstt_p$, then $(\Bspace^{\star}, \nrm{\cdot}_{\star})$ is $\UMD_{p'}$ with constant $\Cconstt_{p'}=\Cconstt_{p}$.
\end{theorem}

\subsection{Concentration for Rademacher complexity}
\begin{lemma}[\cite{bartlett2005local}, Theorem A.2]
\label{lem:concentration_rademacher}
With probability at least $1-\delta$ over the draw of $\eps$,
\begin{align*}
\nrm*{\sum_{t=a}^{b}\eps_ty_t} &\leq{} \En_{\eps}\nrm*{\sum_{t=a}^{b}\eps_ty_t} + \sqrt{\En_{\eps}\nrm*{\sum_{t=a}^{b}\eps_ty_t}\cdot{}2\max_{t\in\brk{n}}\nrm{y_t}\log(1/\delta)} + \frac{\max_{t\in\brk{n}}\nrm{y_t}\log(1/\delta)}{3}\\
&\leq{} 2\En_{\eps}\nrm*{\sum_{t=a}^{b}\eps_ty_t} + \max_{t\in\brk{n}}\nrm{y_t}\log(1/\delta).
\end{align*}
\end{lemma}

\begin{lemma}
\label{lem:concentration_maximal}
For any fixed sequence $y_1,\ldots,y_n$, with probability at least $1-\delta$ over the draw of $\eps$,
\[
\sup_{1\leq{}a\leq{}b\leq{}n}\nrm*{\sum_{t=a}^{b}\eps_ty_t} \leq{} 4\En_{\eps}\nrm*{\sum_{t=1}^{n}\eps_ty_t} + 2\max_{t\in\brk{n}}\nrm{y_t}\log(n/\delta).
\]
\end{lemma}
\begin{corollary}
\label{corr:concentration_maximal_exp}
\[
\En_{\eps}\sup_{1\leq{}a\leq{}b\leq{}n}\nrm*{\sum_{t=a}^{b}\eps_ty_t} \leq{} 4\En_{\eps}\nrm*{\sum_{t=1}^{n}\eps_ty_t} + 5\max_{t\in\brk{n}}\nrm{y_t}\log(n).
\]
\end{corollary}

\begin{proof}[Proof of \pref{lem:concentration_maximal}]
Consider $Z=\nrm*{\sum_{t=a}^{b}\eps_{t}y_t}$ for fixed $a,b$ and a fixed sequence $y_1,\ldots,y_n$. 
Applying \pref{lem:concentration_rademacher}  and taking a union bound over all possible pairs $(a,b)$, of which there are strictly less than $n^{2}$, we have that with probability at least $1-\delta$,
\begin{align*}
\sup_{1\leq{}a\leq{}b\leq{}n}\nrm*{\sum_{t=a}^{b}\eps_ty_t}
&\leq{} 2\sup_{1\leq{}a\leq{}b\leq{}n}\En_{\eps}\nrm*{\sum_{t=a}^{b}\eps_ty_t} + 2\max_{t\in\brk{n}}\nrm{y_t}\log(n/\delta).
\intertext{By \pref{prop:stopping_time}:}
&\leq{} 4\En_{\eps}\nrm*{\sum_{t=1}^{n}\eps_ty_t} + 2\max_{t\in\brk{n}}\nrm{y_t}\log(n/\delta).
\end{align*}
\end{proof}

\section{Burkholder/Bellman functions}
\label{app:burkholder}

\subsection{Elementary design of $\burk$ functions}
The following construction for the scalar case does not obtain optimal constants, but should give the reader a taste of how one can construct a $\burk$ function from first principles.
\begin{theorem}[Elementary Scalar $\burk$ Function]
\label{thm:burkholder_elementary}
Let $k\geq{}4$ be an even integer. Then the function 
\[
\burk(x,y) = \frac{k}{2}\prn*{x^{k} - 2\binom{k}{2}x^{k-2}y^{k} - \frac{1}{k-2}\binom{k}{2}^{-1}\prn*{4\binom{k}{2}\binom{k-2}{2}}^{k-2}y^{k}}.
\]
is Burkholder for $\abs{\cdot}^{k}$, with UMD constant
\[
\Cconstt_{k}\leq{} \alpha{}k^{4}
\]
for some constant $\alpha$.
\end{theorem}
\begin{proof}
Let $\wt{\burk}(x,y) = x^{k} - Cx^{k-2}y^{2} - By^{k}$. We will show that $\wt{\burk}$ is Burkholder for an appropriate choice of constants $B$ and $C$.

Fix $h\in\R$ and let $G(t) = \wt{\burk}(x+ht, y+\eps{}ht)$ for $\eps\in\pmo$. By direct calculation we have
\begin{align*}
G''(0) &= 2h^{2}\brk*{ 
\binom{k}{2}x^{k-2}
- C\prn*{\binom{k-2}{2}x^{k-4}y^{2} + 2\binom{k-2}{2}\eps{}x^{k-3}y + x^{k-2}}
-B\binom{k}{2}y^{k-2}
}
\intertext{Since $k$ is even, $x^{k-4}y^{2}$ is a square; we will simply drop this term.}
&\leq{}
2h^{2}\brk*{
\binom{k}{2}x^{k-2}
- C\prn*{2\binom{k-2}{2}\eps{}x^{k-3}y + x^{k-2}}
-B\binom{k}{2}y^{k-2}
}\\
&\leq{}
2h^{2}\brk*{
\binom{k}{2}x^{k-2}
+2C\binom{k-2}{2}\abs{x}^{k-3}\abs{y} - Cx^{k-2}
-B\binom{k}{2}y^{k-2}
}
\end{align*}
By Young's inequality, we have \[
2C\binom{k-2}{2}\abs{x}^{k-3}\abs{y}=\underbrace{\prn*{2C\binom{k-2}{2}\abs{y}}}_{a}\cdot{}\underbrace{\abs{x}^{k-3}}_{b}\leq{} \frac{1}{k-2}\prn*{(2C\binom{k-2}{2})^{k-2}y^{k-2} + (k-3)x^{k-2} },
\] 
where we have applied $a\cdot{}b\leq{}\frac{1}{k-2}a^{k-2}+\frac{k-3}{k-2}b^{\frac{k-2}{k-3}}$.

Returning to $G''(0)$, we now have
\begin{align*}
G''(0)&\leq{}
2h^{2}\brk*{
\prn*{\binom{k}{2} + \frac{k-3}{k-2} -C}x^{k-2}
+\prn*{\frac{1}{k-2}\prn*{2C\binom{k-2}{2}}^{k-2}- B\binom{k}{2}}y^{k-2}
}.
\intertext{In particular, we can take $C\geq{}2\binom{k}{2}$ and $B\geq{}\frac{1}{k-2}\prn*{2C\binom{k-2}{2}}^{k-2}\binom{k}{2}^{-1}$.}
&\leq{} 0.
\end{align*}
This certifies that $G$ is zig-zag concave. To see the upper bound property, observe by that Young's inequality,
\[
x^{k} - Cx^{k-2}y^{2} - By^{k} \geq{} \frac{2}{k}x^{k} - \prn*{\frac{2}{k}C^{\frac{k}{2}} + B}y^{k}.
\]
Hence, if we take $\burk(x,y) = \frac{k}{2}\wt{\burk}(x,y)$, we have
\[
\burk(x,y) \geq{} x^{k} - \prn*{C^{\frac{k}{2}} + \frac{k}{2}B}y^{k}
.\]

\end{proof}

\subsection{$\burk$ functions for $p=1$}

\begin{definition}[$(1,1)$ Weak Type Burkholder Function]
\label{def:burkholder_weak}
A function $\burk:\Bspace\times{}\Bspace\to\R$ is $(\nrm{\cdot}, \beta)$ Burkholder for weak type if
\begin{enumerate}
\item\label{def:burkholder_weak:1} $\burk{}(x, x') \geq{} \ind\crl*{\nrm*{x}\geq{}1} - \beta\nrm*{x'}$.
\item\label{def:burkholder_weak:2} $\burk{}$ is \emph{zig-zag concave}: $z\mapsto{}\burk{}(x+\eps{}z, x'+z)$ is concave for all $x,x'\in\X$ and $\eps\in\pmo$.
\item\label{def:burkholder_weak:3} $\burk{}(0,0)\leq{}0$.
\end{enumerate}
\end{definition}
\begin{lemma}
\label{lem:burkholder_weak_p1}
Suppose we are given a weak type Burkholder function $\burk{}_{\nrm{\cdot}, \mathrm{weak}}$ for $(\nrm{\cdot}, \beta)$. Then for all arguments $x,y$ with $\nrm{x},\nrm{y}\leq{}B$, the following function is Burkholder for $(\nrm{\cdot}, 1, C\beta\log(B/\eps))$ up to additive slack $\eps$:
\begin{equation}
\burk{}_{\nrm{\cdot}, 1}(x,y) \defeq{} \eps\sum_{k=1}^{N}\burk{}_{\nrm{\cdot},\mathrm{weak}}(x/\lambda_k, y/\lambda_k),
\end{equation}
where $N=\ceil{B/\eps}$ and $\lambda_k=k\eps$.
\end{lemma}
\begin{proof}[Proof of \pref{lem:burkholder_weak_p1}]
Let $V(x,y) = \nrm{x} - C'\beta\log(B/\eps)\nrm{y} - \eps$. We will show that $\burk{}(x,y)\geq{}V(x,y)$ when $\nrm{x},\nrm{y}\leq{}B$.
\begin{align*}
V(x,y) &= \nrm{x} - C'\beta\log(B/\eps)\nrm{y} - \eps\\
&\leq{} \eps + \eps\sum_{k=1}^{N}\ind\crl*{\nrm{x}\geq{}\lambda_k} - C'\beta\log(B/\eps)\nrm{y} - \eps\\
&\leq{} \eps\sum_{k=1}^{N}\brk*{\burk{}_{\nrm{\cdot},\mathrm{weak}}(x/\lambda_k, y/\lambda_k) + \frac{\beta}{\lambda_k}\nrm{y}} - C'\beta\log(B/\eps)\nrm{y}\\
&= \burk{}_{\nrm{\cdot}, 1}(x,y) + \eps\sum_{k=1}^{N}\frac{\beta}{\lambda_k}\nrm{y} - C'\beta\log(B/\eps)\nrm{y}\\
&= \burk{}_{\nrm{\cdot}, 1}(x,y) + \beta\nrm{y}\sum_{k=1}^{N}\frac{1}{k} - C'\beta\log(B/\eps)\nrm{y}\\
&\leq \burk{}_{\nrm{\cdot}, 1}(x,y) + C\beta\nrm{y}\log(N) - C'\beta\log(B/\eps)\nrm{y}
\intertext{For sufficiently large $C'$:}
&\leq{} \burk{}_{\nrm{\cdot}, 1}(x,y).
\end{align*}
It can be seen immediately that $\burk{}_{\nrm{\cdot}, 1}(x,y)$ is zig-zag concave and has $\burk{}_{\nrm{\cdot}, 1}(0,0)\leq{}0$.
\end{proof}
\subsubsection{$\zeta$-Convexity}
\begin{definition}
Say $(\Bspace, \nrm{\cdot})$ is $\zeta$-convex if there exists $\zeta:\Bspace\times{}\Bspace\to\R$ such that
\begin{enumerate}
\item $\zeta$ is biconvex.
\item
$
\zeta(x,y) \leq{} \nrm{x+y}\quad\text{if } \nrm{x}=\nrm{y}=1,
$
\end{enumerate}
\end{definition}

Given a such a function $\zeta$, we can construct a ``canonical'' function $u$ which satisfies some additional properties
\begin{definition}
\label{def:zeta_canonical}
\[
u(x,y) \defeq \left\{
\begin{array}{ll}
\max\crl*{\zeta(x,y), \nrm{x+y}},& \max\crl*{\nrm{x}, \nrm{y}}<1\\
\nrm{x+y},& \max\crl*{\nrm{x}, \nrm{y}}\geq1.
\end{array}
\right.
.
\]
Then $u$ is biconvex, has $\zeta(0,0)\leq{}u(0,0)$, and satisfies
\[
u(x,y) \leq{} \nrm{x+y}\quad\text{if } \max\crl*{\nrm{x}, \nrm{y}}\geq{}1.
\]
Also, $u(x,y) = u(-x, -y)$.
\end{definition}
\begin{assumption}
\label{ass:regularity}
$u(x, -x)\leq{}0$.
\end{assumption}
The $\zeta$ function given in \pref{ex:zeta_l1} satisfies this condition. More generally, most $\zeta$ functions can be made to satisfy this property with a slight blowup in the UMD constant they imply (c.f. \cite[Lemma 8.5]{burkholder1986martingales}).

 By \cite[8.6]{burkholder1986martingales} \pref{ass:regularity} implies $u(x,y) \leq{} u(0,0) + \nrm{x+y}$.
The following argument due to \citep{burkholder1986martingales} shows how to create a $\burk$ function from the function $u$.
\begin{theorem}
\label{thm:zeta_burkholder}
Suppose $\nrm{\cdot}$ is $\zeta$-convex and $u$ satisfies \pref{ass:regularity}. Then this space is UMD with weak type estimate
\[
\Pr\prn*{\nrm*{\sum_{t=1}^{n}dZ_t}\geq{}1}\leq{} \frac{2}{u(0,0)}\En\nrm*{\sum_{t=1}^{n}\eps_tdZ_t}
\]
for any martingale difference sequence $(dZ_t)$.
Furthermore, the function
\[
\burk(x,y) = 1 - \frac{u(x+y, y-x)}{u(0,0)}
\]
is weak-type Burkholder for $(\nrm{\cdot}, \frac{2}{\zeta(0,0)})$, in the sense of \pref{def:burkholder_weak}.

\end{theorem}
\begin{proof}[Proof of \pref{thm:zeta_burkholder}]
For the weak type estimate, we will start with the base function
\[
V(x,y) = \ind\crl*{\nrm{x}\geq{}1} - 
\frac{2}{u(0,0)}\nrm{y}.
\]
We will now show that $V(x,y)\leq{}\burk(x,y)$. First, observe that
\[	
\ind\crl*{\nrm{x}\geq{}1}=\ind\crl*{\nrm{(x+y) + (x-y)}\geq{}2}\leq{}\ind\crl*{\max\crl{\nrm{x+y}, \nrm{y-x}}\geq{}1}\leq{} \ind\crl*{2\nrm{y}\geq{}u(x+y, y-x)},
\]
where the last inequality follows from the additional property of $u$ from \pref{def:zeta_canonical}. We have now established
\begin{align*}
V(x,y)
&\leq{}\ind\crl*{2\nrm{y}\geq{}u(x+y, y-x)} - \frac{2}{u(0,0)}\nrm{y}\\
&=\ind\crl*{2\nrm{y}-u(x+y, y-x) + u(0,0) \geq{} u(0,0)} - \frac{2}{u(0,0)}\nrm{y}
\intertext{By the second additional property of $u$ from \pref{def:zeta_canonical}, $2\nrm{y}-u(x+y, y-x) + u(0,0)\geq{}0$, and so we may apply Markov's inequality}
&\leq{}\frac{2\nrm{y}-u(x+y, y-x) + u(0,0)}{u(0,0)} - \frac{2}{u(0,0)}\nrm{y}\\
&= \burk(x,y).
\end{align*}
Observe that $\burk(0,0) =0$ and, since $u$ is biconvex, $-u(x+y, y-x)$ is zig-zag concave, and so $\burk$ is itself zig-zag concave. We can now prove that the UMD property holds with constant $\frac{2}{u(0,0)}\leq{}\frac{2}{\zeta(0,0)}$ using the standard step-by-step peeling argument with $\burk$.
\end{proof}

\begin{example}[$\ls_1^{d}$ \cite{osekowski2016umd}]
\label{ex:zeta_l1}
Define
\[
z(x,y) = \left\{
\begin{array}{ll}
\frac{a\tri*{x,y}}{2} - \frac{1}{2a}, & \nrm{x+y} + \nrm{x-y} \leq{} 2/a\\
\frac{\nrm{x+y}}{2}\log\prn*{\frac{a}{2}\prn*{\nrm{x+y} + \nrm{x-y}}} - \frac{\nrm{x-y}}{2}, & \nrm{x+y} + \nrm{x-y} > 2/a
\end{array}
\right.
.
\]
Then define
\[
\zeta(x,y) = \frac{2}{\log(3a)}\prn*{1 + \sum_{i=1}^{d}z(x_i, y_i)}.
\]
For $a\geq{}d\log{}d$ the $\zeta$-convexity properties are satisfied and the bound $\zeta(0,0)\leq{}\frac{2}{\log{}d + \log(2\log{}d)}\prn*{1 - \frac{1}{2\log{}d}}$ is achieved.

\end{example}

\end{document}